\documentclass[11pt]{article}

\usepackage[margin=1in]{geometry}
\usepackage{amsmath,amsthm,amssymb,amsfonts,enumerate,tikz}
\usepackage{algpseudocode,algorithm,algorithmicx}
\usepackage{hyperref}
\usepackage[mathscr]{eucal}
\usepackage{setspace}
\usepackage[capitalise]{cleveref} 
\usepackage{caption}
\usepackage{subcaption}

\newtheorem{theorem}{Theorem}
\newtheorem*{theorem*}{Theorem}
\newtheorem{proposition}[theorem]{Proposition}
\newtheorem{lemma}[theorem]{Lemma}

\theoremstyle{definition}

\newtheorem{definition}[theorem]{Definition}

\newtheorem{thmx}{Theorem}
 
\numberwithin{theorem}{section} % important bit

\renewcommand{\subset}{\subseteq}

\renewcommand{\tilde}{\widetilde}
\renewcommand{\epsilon}{\varepsilon}

\newcommand{\commentout}[1]{}
\def\sign{\text{sign}}

\def\quadand{\quad\text{and}\quad} 
\def\calP{\mathcal{P}}

\def\<{\langle}
\def\>{\rangle}
\def\({\Big(}
\def\){\Big)}
\def\Lip{{\rm Lip}}
\def\bfb{\boldsymbol{b}}
\def\bfe{\boldsymbol{e}}
\def\bfj{\boldsymbol{j}}
\def\bfk{\boldsymbol{k}}
\def\bff{\boldsymbol{f}}
\def\bfr{\boldsymbol{r}}
\def\bfs{\boldsymbol{s}}
\def\bfx{\boldsymbol{x}}
\def\bfy{\boldsymbol{y}}
\def\bfu{\boldsymbol{u}}

\def\bfalpha{\boldsymbol{\alpha}}
\def\bfbeta{\boldsymbol{\beta}}
\def\bfgamma{\boldsymbol{\gamma}}
\def\bfl{\boldsymbol{\ell}}
\def\A{\mathcal{A}}
\def\calA{\mathcal{A}}
\def\calB{\mathcal{B}}

\def\calE{\mathcal{E}}
\def\calF{\mathcal{F}}

\def\N{\mathbb{N}}
\def\calN{\mathcal{N}}
\def\calN{\mathcal{N}} 

\def\R{\mathbb{R}}

\def\calY{\mathcal{Y}}
\def\Z{\mathbb{Z}}

\def\andspace{\quad\text{and}\quad}
\def\forallspace{\quad\text{for all}\quad}

\numberwithin{equation}{section}
\title{Approximation of functions with one-bit neural networks}
\author{C. Sinan G\"unt\"urk\footnote{Courant Institute, New York University. Email: gunturk@cims.nyu.edu} \and Weilin Li\footnote{City University of New York, City College. Email: wli6@ccny.cuny.edu}}
\begin{document}
	\maketitle
	
	\begin{abstract}
	The celebrated universal approximation theorems for neural networks roughly state that any reasonable function can be arbitrarily well-approximated by a network whose parameters are appropriately chosen real numbers. This paper examines the approximation capabilities of one-bit neural networks -- those whose nonzero parameters are $\pm a$ for some fixed $a\not=0$. One of our main theorems shows that for any $f\in C^s([0,1]^d)$ with $\|f\|_\infty<1$ and error $\epsilon$, there is a $f_{NN}$ such that $|f(\bfx)-f_{NN}(\bfx)|\leq \epsilon$ for all $\bfx$ away from the boundary of $[0,1]^d$, and $f_{NN}$ is either implementable by a $\{\pm 1\}$ quadratic network with $O(\epsilon^{-2d/s})$ parameters or a $\{\pm \frac 1 2 \}$ ReLU network with $O(\epsilon^{-2d/s}\log (1/\epsilon))$ parameters, as $\epsilon\to0$. We establish new approximation results for iterated multivariate Bernstein operators, error estimates for noise-shaping quantization on the Bernstein basis, and novel implementation of the Bernstein polynomials by one-bit quadratic and ReLU neural networks. 
	\end{abstract}



% % % % % % % % % % % % % % % % % % % %

\section{Introduction}

\subsection{Motivation}

In this paper, we address the following question regarding the approximation capabilities of quantized neural networks: what classes of functions can be approximated by neural networks of a given size such that their weights and biases are constrained to be in a small set of allowable values, especially with regards to the extreme one-bit case? While this is an interesting mathematical question by itself and warrants special attention given the growing importance of machine learning across numerous scientific disciplines, here we illuminate two particularly important broader questions that motivate this theoretical study. 

Our first motivation comes from a practical problem. State-of-the-art neural networks oftentimes contain a massive number of parameters and are trained on enormous computational machines. It appears that in regards to performance of neural networks, the ``bigger is better" philosophy largely holds true  \cite{canziani2016analysis}. As higher resolution audio, image, and video data become increasingly more common, the size of high performance networks will only continue to grow and require more computational resources to utilize. This conflicts with the desire to use them on portable and low power devices such as smartphones. Quantization is the process of replacing high resolution floating point numbers with coarser ones. It is a natural solution to this issue, since simpler binary operations can help alleviate the costly computational burden that comes with using expensive floating point operations. 

Our second motivation is related to the over-parameterization phenomena. Since state-of-the-art neural networks often contain many more parameters than both the number of training samples and data dimensionality \cite{belkin2019reconciling}, it is widely believed that there is not a unique set of parameters that can be used to represent a given function. If we imagine that the set of all parameters generating a prescribed function is a manifold embedded in a high dimensional parameter space, then our main question is closely related to whether this manifold is sufficiently close to some lattice point. Since certain parameter choices may be more desirable than others, this is a central question in not only quantization, but also in neural network compression and model reduction.

\subsection{Quantized neural networks}

In this paper, we exclusively examine strict neural networks. We use the adjective ``strict" to emphasize that such networks do not have any skip connections and apply the same activation function to each node except for the final affine layer. More specifically, fix a function $\beta\colon\R\to\R$, and slightly abusing notation, for each $m\geq 1$, we extend it to a map $\beta\colon \R^m\to\R^m$ defined as $\beta(\bfx):=(\beta(x_1),\dots,\beta(x_m))$. 

\begin{definition}
	\label{def:strictNN}
	A strict neural network with activation $\beta$ is any function $f\colon \R^d\to\R^m$ of the form,
	\begin{equation*}
		f(\bfx):= W_L\beta(W_{L-1} \cdots \beta(W_1(\bfx))), \quad W_\ell(\bfu):=A_\ell \bfu + \bfb_\ell \quad \text{for} \quad  \ell=1,\dots,L. 
	\end{equation*}	
\end{definition}

In this definition, each $A_\ell\in \R^{N_\ell \times N_{\ell-1}}$ is referred to as a weight matrix, $\bfb_\ell\in \R^{N_\ell}$ is called a bias vector, $L$ is the number of layers, and $N_0=d$ and $N_L=m$. For each $1\leq \ell\leq L-1$, we refer to $\bfu_\ell:=\beta(W_\ell\beta(\cdots \beta(W_1(\bfx)))$ as the $\ell$-th layer's output. This network has $L$ layers, and applies the same activation function to each node except for the final linear layer. It has $N_\ell$ nodes in layer $\ell$, hence has $\sum_{\ell=1}^L N_\ell$ nodes in total, and we define the number of parameters as however many nonzero entries in $\{A_\ell\}_{\ell=1}^L$ and $\{\bfb_\ell\}_{\ell=1}^L$ there are. Since we place no restrictions on the weight matrices' structures, our framework allows for fully connected and convolutional networks. We say a function $g$ can be implemented by neural network provided it can be written in the above form with appropriate weights. 

Other common definitions of a neural network allow for additional operations and flexibility. For instance, some conventions allow for the use of skip connections whereby $\bfu_{\ell+1}$ is allowed to depend on any previous layer's outputs $\bfu_{\ell},\dots, \bfu_1$. Another example is the use of different activation functions per node and possibly in each layer, including the identity function which effectively bypasses an activation. Following the terminology from \cite{gribonval2022approximation}, to differentiate between the strict notion we use versus less stringent notions, we refer to the latter as generalized neural networks.

We are interested in strict quantized neural networks. While it is convenient to treat the entries of $\{A_\ell\}_{\ell=1}^L$ and $\{\bfb_\ell\}_{\ell=1}^L$ as real numbers for theoretical analysis, when used for computations, each entry is traditionally stored as a 32-bit float in memory. To reduce the number of bits, we consider a fixed finite $\calA\subset \R$ called the alphabet. 

\begin{definition}
	\label{def:quanNN}
	A strict $\calA$-quantized neural network is a strict neural network where all nonzero entries of the weight matrices $\{A_\ell\}_{\ell=1}^L$ and bias vectors $\{\bfb_\ell\}_{\ell=1}^L$ belong to $\calA$. 
\end{definition}

An alphabet $\calA$ that only consists of a small set of allowable values, such as the extreme one-bit case, is theoretically interesting as it poses stringent constraints and is computationally relevant as one can take advantage of special hardware and software for one-bit floating point operations. Since the weights and biases are selected from the same alphabet $\calA$, this definition imposes a particular scaling on the associated network, which can be altered by dilating the domain.

\subsection{The challenge: approximation by coarsely quantized networks}

\label{sec:challenge}

The problem of neural network quantization can be studied from an approximation theory perspective. Let $\calN:=\calN(\calA,\beta,L,N,P)$ be the set of functions that can be expressed as a strict $\calA$-quantized neural network with activation $\beta$ and has at most $L$ layers, $N$ nodes, and $P$ parameters. For a prescribed function class $\calF$, and distortion measure $\calE\colon \calF\times \calN\to\R$, we study the
$$
\text{approximation error} \quad := \quad \sup_{f\in \calF} \ \ \inf_{g\in \calN}  \ \ \calE(f,g).
$$ 

The approximation error achieved by generalized neural networks has been extensively studied in the traditional case where there is no quantization. Well known classical universal approximation theorems \cite{cybenko1989approximation,barron1994approximation} are qualitative statements for shallow networks. Modern versions \cite{yarotsky2017error,shaham2018provable,lu2021deep,daubechies2021nonlinear,gribonval2022approximation} provide quantitative approximation rates in terms of the function class, number of layers, parameters, etc. This is only a partial list of references, and additional ones can be found in the bibliography of a recent comprehensive survey \cite{devore2021neural}. Perhaps this is a suitable place to mention that the set of strict neural networks is a subset of their generalized counterparts, so any function class that can be approximated by strict quantized neural networks can also be well-approximated by generalized unquantized ones.

It is natural to wonder if these results or their proof strategies can be adapted to strict quantized networks. If $\calA$ is of sufficiently high resolution and covers a wide range of numbers, such as $\calA=\delta\Z \cap [-M,M]$, for sufficiently large $M$ and small $\delta>0$, then the aforementioned approximation results for unquantized generalized networks can be suitably adapted, such as in \cite[Lemma 3.7]{bolcskei2019optimal}. However, this approach requires using an increasingly higher resolution alphabet (reducing $\delta$) to achieve smaller errors. Major difficulties come into play once we fix an alphabet $\calA$ and a function class $\calF$, and ask to approximate any $f\in\calF$ up to any error $\epsilon$ by a $\calA$-quantized strict neural network. 

Many of the aforementioned papers employ the following ubiquitous strategy. For a prescribed $f\in \calF$, we approximate $f$ by a particular linear combination, $\sum_{k=1}^N a_k \phi_k$,
where the coefficients $\{a_k\}_{k=1}^N$ depend on $f$ and the span of $\{\phi_k\}_{k=1}^\infty$ is dense in $\calF$. Examples of $\{\phi_k\}_{k=1}^\infty$ include local polynomials, ridge functions, and wavelets. While it is possible that $\phi_k$ is not implementable as a generalized neural network, it is enough to find a $\psi_k$ that is a close approximation of $\phi_k$ which is  implementable. After $\{\psi_k\}_{k=1}^N$ are implemented, the summation  $\sum_{k=1}^N a_k \psi_k$ is carried out by a linear last layer whose weights are $\{a_k \}_{k=1}^N$.

However, this approximation and implementation strategy becomes problematic when $\calA$ is a small set, especially for the one-bit case. For example, if we were to closely follow the same strategy, we need each $\psi_k$ to be implementable by a $\calA$-quantized strict neural network, and require $a_k\in \calA$ to perform the linear combination $\sum_{k=1}^N a_k \psi_k$. From this point of view, it is natural to desire the following three properties: 
\begin{enumerate}\itemsep-.25em
 	\item[($P_1$)] 
 	Approximation. Given a large function class $\calF$, finite linear combinations of $\{\phi_k\}_{k=1}^\infty$ with real coefficients can efficiently approximate any $f\in\calF$. 
 	\item[($P_2$)] 
 	Quantization. For any $f\in\calF$ and its approximation $\sum_{k=1}^N a_k\phi_k$, the real coefficients $\{a_k\}_{k=1}^N$ can be replaced with suitable ones from just $\calA$ without incurring to much additional error.
	\item[($P_3$)] 
	Implementation. For each $\phi_k$, there is a good approximant $\psi_k$ that can be implemented by a strict $\calA$-quantized neural network. 
\end{enumerate}

While there many satisfactory choices for ($P_1$), it is not immediate if any of those can be made compatible with ($P_2$) and ($P_3$). On the other hand, there are a few known choices of $\{\phi_k\}_{k=1}^\infty$ that satisfy ($P_2$). Some are for very restrictive function classes: when $f$ is bandlimited and $\{\phi_k\}_{k=1}^\infty$ contain shifts of a sinc-kernel \cite{DD,exp_decay} or when $f$ is a power series of a single complex variable and $\{\phi_k\}_{k=1}^\infty$ is the standard polynomial power basis \cite{gunturk2005approximation}. 

The approach taken in this paper builds upon our recent publication \cite{onebitBernstein}, and there, we showed that any continuous function on $[0,1]$ can be approximated by a $\pm 1$ linear combination of Bernstein polynomials. A Bernstein polynomial of order $n$ and index $\bfk=(k_1,\dots,k_d)$ with $0\leq k_\ell\leq n$ is the function $p_{n,\bfk}\colon\R^d\to[0,1]$ defined as 
$$
p_{n,\bfk}(\bfx)
=\binom{n}{\bfk} \bfx^{\bfk} (1-\bfx)^{n-\bfk}
=\prod_{\ell=1}^d \binom{n}{k_\ell} x_\ell^{k_\ell} (1-x_\ell)^{n-k_\ell}.
$$
Continuing with this line of research, we investigate the multivariate Bernstein system as a potential candidate that satisfies all three of our desired properties. 

\subsection{Main contributions}

The main theorems are proved by decomposing the total approximation error into three Bernstein related terms, 
\begin{equation*}
	f-f_{NN}
	\quad =\underbrace{f-f_B}_{\text{Bern. approx. error}} + \underbrace{f_B-f_Q}_{\text{Bern. quan. error}} + \underbrace{f_Q-f_{NN}}_{\text{Bern. implementation error}}. 
\end{equation*}
Here, $f_B$ is a linear combination of multivariate Bernstein polynomials whose coefficients are real and appropriately bounded, $f_Q$ is a $\pm a$ linear combination of multivariate Bernstein polynomials for appropriate $a$, and $f_{NN}$ is a function implementable by a strict $\{\pm a\}$-quantized neural network. 

Our first main theorem concerns $f-f_Q=(f-f_B)+(f_B-f_Q)$. It shows that any smooth multivariate $f$ can be approximated by a $\pm 1$ linear combination of Bernstein polynomials with a quantitative rate that exploits smoothness of $f$. In the following, $\|\cdot\|_{C^s}$ is a norm on the space of $s$-times continuously differentiable functions and $\|\cdot\|_{C^1\Lip}$ is a norm on the space of continuously differentiable functions whose first order partial derivatives are Lipschitz. 

\begin{thmx}
	\label{thm:mainbernstein}
	Let $s,d,n\geq 1$, $\mu\in (0,1)$, and $f\in C^{s}([0,1]^d)$ with $\|f\|_\infty\leq \mu$. If $s\geq 3$, also assume that
	$
	n\geq \displaystyle \frac{\sqrt{2^{s+1}}d}{4(1-\mu)}\|f\|_{C^1\Lip}. 
	$
	Then for any $1\leq \ell\leq d$, there exists a sequence $\{\sigma_{\bfk}\}_{0\leq \bfk\leq n}$ such that $\sigma_{\bfk}\in \{\pm 1\}$ for each $\bfk$ and for all $\bfx\in [0,1]^d$,
	$$
	\Big|f(\bfx)-\sum_{0\leq \bfk\leq n} \sigma_{\bfk} p_{n,\bfk}(\bfx)\Big|
	\lesssim_{s,d,\mu} \|f\|_{C^s}  \min\big(1,n^{-s/2}x_\ell^{-s} (1-x_\ell)^{-s} \big).
	$$
\end{thmx}

\cref{thm:mainbernstein} is proved constructively by first approximating $f$ with a suitable $f_B$ of the form $\sum_{0\leq \bfk\leq n} a_{\bfk} p_{n,\bfk}$. Then secondly, the coefficients $\{a_{\bfk}\}_{0\leq \bfk\leq n}$ are fed into an algorithm called $\Sigma\Delta$ quantization to produce the desired one-bit sequence $\{\sigma_{\bfk}\}_{0\leq \bfk\leq n}$ from which we get $f_Q$ of the form $\sum_{0\leq \bfk\leq n} \sigma_{\bfk} p_{n,\bfk}$. The approximant $f_Q$ is a $\pm 1$ linear combination of Bernstein polynomials and can be numerically computed, provided that we have point samples of $f$ on a sufficiently dense grid, as summarized in \cref{alg:binarybernalg}. The  $\Sigma\Delta$ algorithm falls under a general class of ``noise-shaping" methods, where the main idea is to compute the signs $\{\sigma_{\bfk}\}_{0\leq \bfk\leq n}$ so that the ``noise" $\{a_{\bfk}-\sigma_{\bfk}\}_{0\leq \bfk\leq n}$, when fed into the synthesis operator $c\mapsto \sum_{0\leq \bfk\leq n} c_{\bfk} p_{n,\bfk}$, is small in a suitable sense. 

Our second main result deals with the implementation error $f_Q-f_{NN}$. It shows that any one-bit linear combination of Bernstein polynomials is implementable by strict one-bit neural networks, with either the quadratic $\rho(t)=\frac{1}{2} t^2$ or ReLU $\sigma(t)=\max(t,0)$ activation. The proof is constructive and schematic diagrams for the constructed networks are shown in Figures \ref{fig:Bnetwork} and \ref{fig:Bpascal}. 

\begin{thmx}
	\label{thm:mainimplementation} 
	For any integers $d,n\geq 1$ and function $f=\sum_{0\leq\bfk\leq n}\sigma_{\bfk} p_{n,\bfk}$ where $\sigma_{\bfk}\in \{\pm 1\}$ for each $\bfk$, the following hold. 
	\begin{itemize}
		\item 
		There is a $\{\pm 1\}$-quantized quadratic neural network $f_{NN,\rho}$ that has $O(n)$ layers and $O(n^d)$ nodes and parameters, as $n\to\infty$, such that $f_{NN,\rho}=f$.
		\item 
		For each $\epsilon$, there exists a $\{\pm \frac{1}{2}\}$-quantized ReLU neural network $f_{NN,\sigma}$ with $O(n\log (n/\epsilon))$ layers and $O( (n^2+n^d)\log (n/\epsilon))$ nodes and parameters, as $n\to\infty$ and $\epsilon\to0$, such that $\|f-f_{NN,\sigma}\|_\infty\leq \epsilon$.
	\end{itemize}
\end{thmx}

Combining the previous theorems, we obtain the final main theorem of this paper. 

\begin{thmx}
	\label{thm:main}
	Let $s,d,n\geq 1$, $\mu\in (0,1)$, and $f\in C^{s}([0,1]^d)$ with $\|f\|_\infty\leq \mu$. If $s\geq 3$, also assume that
	$
	n\geq \displaystyle \frac{\sqrt{2^{s+1}} d}{4(1-\mu)}\|f\|_{C^1\Lip}. 
	$ 
	For any $1\leq \ell\leq d$, there exist functions $f_{NN,\rho}$ and $f_{NN,\sigma}$ such that:
	\begin{itemize}
		\item 
		for $f_{NN}\in \{f_{NN,\rho}, f_{NN,\sigma}\}$ and any $\bfx\in [0,1]^d$, 
		\begin{equation*}
			|f(\bfx)-f_{NN}(\bfx)|
			\lesssim_{s,d,\mu} \|f\|_{C^s}  \min\big(1,n^{-s/2}x_\ell^{-s} (1-x_\ell)^{-s} \big).
		\end{equation*}
		\vspace{-2em}
		
		\item 
		$f_{NN,\rho}$ is implementable by a strict $\{\pm 1\}$-quantized quadratic neural network that has $O(n)$ layers and $O(n^d)$ nodes and parameters, as $n\to\infty$.
		\item 
		$f_{NN,\sigma}$ is implementable by a strict $\{\pm \frac{1}{2}\}$-quantized ReLU neural network that has $O(n\log n)$ layers and $O( (n^2+n^d)\log n)$ nodes and parameters, as $n\to\infty$. 
	\end{itemize}
\end{thmx}

\begin{figure}
	\centering
	\begin{subfigure}{0.45\textwidth}
		\includegraphics[width=\textwidth]{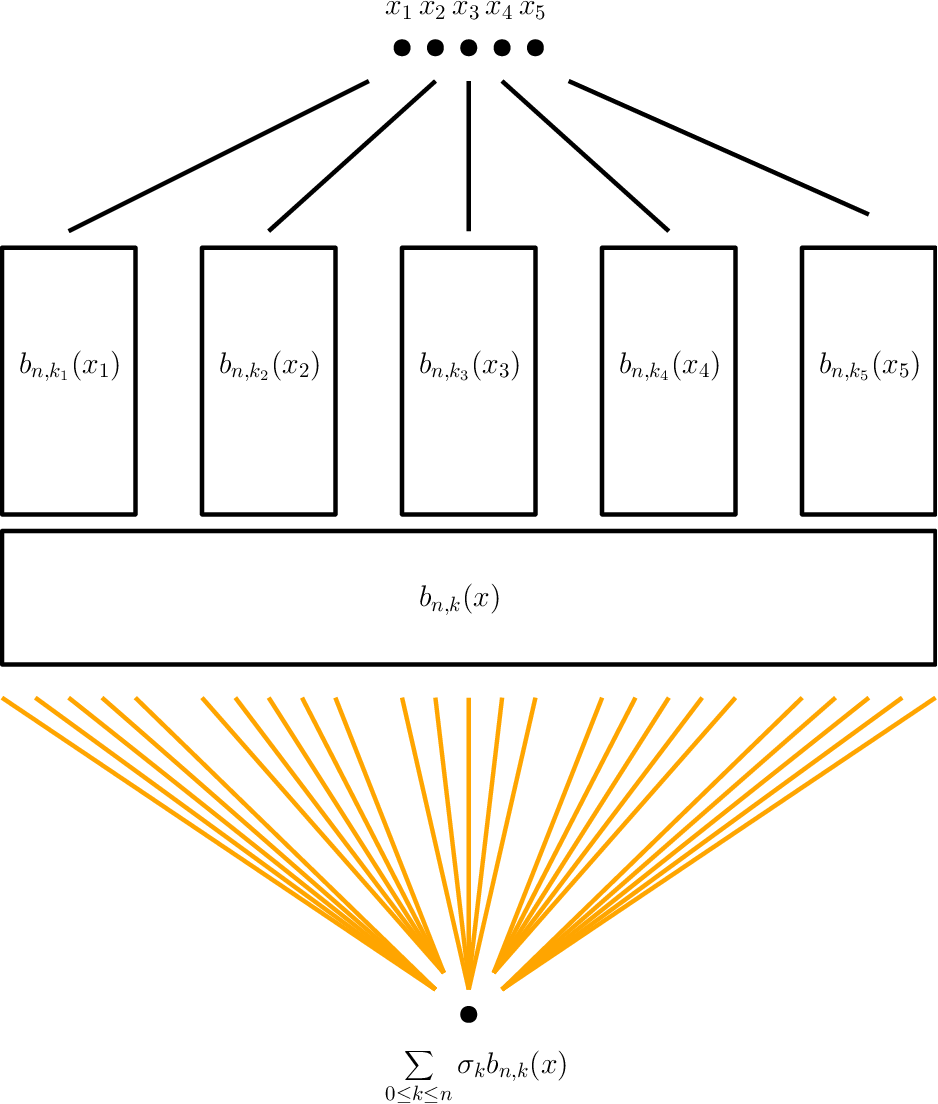}
	\caption{One-bit Bernstein neural network}
	\label{fig:Bnetwork}
	\end{subfigure}
	\quad 
	\begin{subfigure}{0.45\textwidth}
		\centering
		\vspace{1em}
		\includegraphics[width=\textwidth]{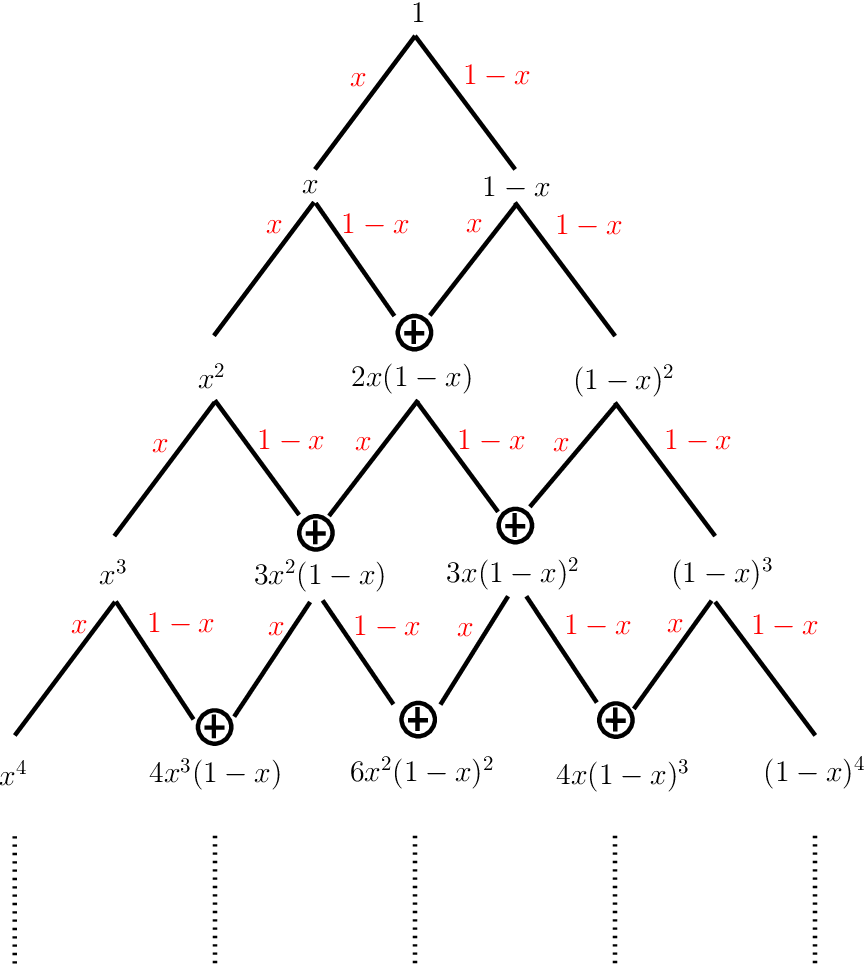}
		\caption{Pascal-Bernstein triangle}	
		\label{fig:Bpascal}
	\end{subfigure}
	
	\caption{Schematic diagrams of the constructed one-bit neural networks.}
\end{figure}

\subsection{Why Bernstein? Other contributions}

In this paper, we show that the multivariate Bernstein polynomials $\{p_{n,\bfk}\}_{0\leq \bfk\leq n}$ of order $n$ satisfy the three important properties ($P_1$), ($P_2$), and ($P_3$), which are then used to prove the main theorems.

\begin{enumerate}
	\item[($P_1$)] 
	Approximation. In \cref{thm:approxsmooth}, we use iterated Bernstein operators to approximate any $f\in C^s([0,1]^d)$, with a rate of approximation that exploits smoothness of $f$, which generalizes the one-dimensional results of Micchelli \cite{micchelli1973} and Felbecker \cite{felbecker1979}. This allows us to avoid a known saturation result \cite[Chapter 10, Theorem 3.1]{DL}, which says that the usual Bernstein operator is unable to exploit higher order derivatives of the target function. In \cref{thm:bapprox}, we convert the resulting iterated approximant into a linear combination of Bernstein polynomials without much amplification in the resulting coefficients.
	
	\item[($P_2$)] 
	Quantization. In \cref{thm:rsigmadelta}, we show that any linear combination of Bernstein polynomials with real coefficients can be replaced with a suitable sequence of $\pm 1$ linear combination, without significant error. This step is done constructively through a directional $\Sigma\Delta$ algorithm. This is perhaps surprising because previous applications of noise-shaping quantization utilize some notion of redundancy in the generating system, whereas the Bernstein polynomials are linearly independent. One explanation is that the Bernstein system of order $n$ and of a single variable span a subspace whose numerical rank is approximately $\sqrt n$.

	\item[($P_3$)] 
	Implementation. For both the ReLU and quadratic activation functions, our implementation strategy exploits a natural Pascal triangle interpretation of the univariate Bernstein polynomials, as shown in Figure \ref{fig:Bpascal}. This connection is vital in being able to implement the Bernstein polynomials in a stable and efficient way. Since we exclusively employ strict neural networks, our constructions are more constrained than those found in papers that employ generalized neural networks. The constructions are found in \ref{sec:NNappendix}. 
\end{enumerate}

\subsection{Additional related work} 

This paper addresses the universal approximation capabilities of coarsely quantized neural networks: whether they can arbitrarily well-approximate large function classes. Related work on approximation by unquantized neural networks or with variable high resolution alphabets were discussed in \cref{sec:challenge}. There we also explain why the problem of universal approximation with coarsely quantized networks is significantly different and requires novel technical developments.

On the other hand, several algorithms for neural network quantization have been developed and their performances have been evaluated empirically, see the survey article \cite{guo2018survey}. In essence, existing algorithms take a pre-trained network and replace each layer with a quantized approximant, and/or directly train the network by quantizing the back-propagation vectors. While it has been empirically observed that they can compress networks without sacrificing substantial accuracy, there is little theory explaining why. Several recent works \cite{ashbrock2021stochastic,lybrand2021greedy} address this gap by proposing gradient-based quantization methods with provable guarantees. 

This paper's material builds upon our previous work \cite{gunturk2022approximation}. There we analyzed the approximation properties of one-bit linear combinations of univariate Bernstein polynomials, and Theorem \ref{thm:main} is a multivariate generalization of \cite[Theorem 7]{gunturk2022approximation}. While \cite[Appendix B]{gunturk2022approximation} briefly touches upon implementation by quadratic neural networks, it only pertained to a single dimension and did not provide quantitative bounds for the size of such networks. On the other hand, the results in this paper deal with arbitrary dimensions, and for both the quadratic and ReLU activation functions. 

This version improves upon the first draft of this paper in two ways: we implement the approximation strategy given in \cref{thm:mainbernstein} with strict neural networks, and the ReLU case only necessitates a one-bit alphabet. 

\subsection{Organization}

The organization of subsequent sections follow the ordering ($P_1$), ($P_2$), and ($P_3$). The Bernstein approximation error $f-f_B$ is treated in  \cref{sec:Bapprox}. The quantization error in the Bernstein basis $f_B-f_Q$ is dealt with in \cref{sec:Bquan}. The implementation error of these one-bit approximations $f_Q-f_{NN}$ for both the quadratic and ReLU activation functions is studied in \cref{sec:Bimplement}, while the
constructions of $f_{NN}$ are carried out in \cref{sec:NNappendix}. Proofs of Theorems \ref{thm:mainbernstein}, \ref{thm:mainimplementation}, and \ref{thm:main} are provided in Sections \ref{sec:proofthmmain1}, \ref{sec:proofthmmain2}, and \ref{sec:proofthmmain3}, respectively. Final remarks and other aspects of this paper, including the computational method and entropy considerations, are contained in \cref{sec:finalremarks}. 

\subsection{Basic notation}

We let $\R$ be the reals and $\N$ be the natural numbers including zero. For reasons that will become evident, we let $\log$ denote the base 2 logarithm. The ceiling and floor functions are denoted $\lceil \cdot \rceil$ and $\lfloor\cdot \rfloor$, respectively. We use the notation $A \lesssim B$ to mean that there is a universal constant $C$ such that $A\leq CB$. When we write $A \lesssim_{a,b,c} B$, we mean that there is a $C>0$ that possibly depends on $a,b,c$ for which $A\leq CB$. 

Let $\bfe_1,\dots,\bfe_d$ denote the canonical orthonormal basis for $\R^d$. For any $p\in [1,\infty]$, let $\|\cdot\|_p$ be the usual $p$-th norm on vectors in $\R^d$. We also use the same notation for the $\ell^p$-norm of a function defined on a countable set, and the $L^p$ norm of a function. 

We denote multi-indices and tuples with boldface letters. For $\bfk,\bfl\in\N^d$ and $m\in\N$, we define the following. We write $\bfk\leq \bfl$ as shorthand for $k_j\leq \ell_j$ for each $j$. Likewise, we let $\bfk\leq m$ (or $m\leq \bfk$) to mean that $k_j\leq m$ (or $m\leq k_j$) for each $j$. For any $\bfx\in\R^d$, let $\bfx^{\bfk}:=x_1^{k_1}\cdots x_d^{k_d}$, and $|\bfk|=k_1+\cdots+k_d$. The degree of $\bfx^{\bfk}$ is defined to be $|\bfk|$. For any $t\in\R$, we let $t\bfk=(tk_1,\dots,tk_d)$. We follow standard convention for dealing with combinatorial factors: 
$$
\binom{\bfk}{\bfl}:= \prod_{j=1}^d \binom{k_j}{\ell_j}, 
\andspace
\binom{m}{\bfk}:= \prod_{j=1}^d \binom{m}{k_j}.
$$

\section{Approximation error}
\label{sec:Bapprox} 

In this section, we concentrate on the error incurred by approximating a smooth $f$ with $f_B=\sum_{0\leq \bfk\leq n}a_{\bfk} p_{n,\bfk}$, a linear combination of Bernstein polynomials of order $n$, where each coefficient $a_{\bfk}$ can be suitably controlled in terms of $f$. For reasons that will become apparent later on when we discuss the quantization error, we will construct a $f_B$ whose coefficients in the Bernstein basis are not much larger than $\|f\|_\infty$. More specifically, the results of this section will provide us with upper bounds on the approximation error by Bernstein polynomials, 
$$
\calE_{n,c}^{\text{approx}}(f)
:=\inf_{\{a_{\bfk}\}_{0\leq \bfk \leq n}, \, |a_{\bfk}|\leq c} \Big\| f-\sum_{0\leq \bfk\leq n} a_{\bfk} p_{n,\bfk}\Big\|_\infty. 
$$

A famous theorem of Bernstein showed that $B_n(f)$ converges uniformly to $f$ on $[0,1]$ provided that $f$ is continuous, which generalizes to higher dimensions, see \cite{hildebrandt1933linear} for the two-dimensional case. While $B_n(f)$ may appear to be a natural candidate for $f_B$, there is a well known saturation result which says that, even for infinitely differentiable $f$, the fastest rate of decay is $1/n$, see \cite[Chapter 10, Theorem 3.1]{DL}. 

To circumvent this saturation, we follow an approach of Micchelli \cite{micchelli1973}, and show in \cref{thm:approxsmooth}, that appropriate iterates of multivariate Bernstein operators achieve improved convergence rates that exploit smoothness of $f$. \cref{thm:iteratetolinear} allows us to convert these iterated approximations to a linear combination of Bernstein polynomials and \cref{thm:bapprox} provides us with the existence of a suitable $f_B$. 

\subsection{Background on Bernstein polynomials}

For a fixed integer $n\geq 1$, we denote the set of univariate Bernstein polynomials by $\calB_n:=\{p_{n,k}\}_{k=0}^n$. Each $p_{n,k}\colon [0,1]\to [0,1]$ is a polynomial of degree $n$ and 
$$
p_{n,k}(x):= \binom{n}{k}x^k (1-x)^{n-k}. 
$$
The Bernstein polynomials are nonnegative, form a partition of unity for $[0,1]$, and form a basis for the vector space of $n$ degree algebraic polynomials. They can be used to give a constructive proof of the classical Weierstrass theorem, by showing that any continuous $f$ on $[0,1]$ can be uniformly approximated by the Bernstein polynomial of $f$, 
$$
B_n(f)(x)
:=\sum_{k=0}^n f\(\frac{k}{n}\) p_{n,k}(x). 
$$
For reasons that will be apparent later, it will be convenient for us to extend the index $k$ beyond $n$ in the definition of $p_{n,k}$. If $k<0$ or $k > n$, then we define $p_{n,k}:= 0$. 

The Bernstein polynomials can also be viewed from a probabilistic perspective. For any $x\in [0,1]$, the quantity $p_{n,k}(x)$ is the probability that $k$ successes occur in $n$ independent Bernoulli trials each with probability of success $x$. Since the expected value is $nx$, for each integer $s\geq 0$, the $s$-th central moment is 
\begin{equation*}
	T_{n,s}(x) := \sum_{k=0}^n (k-nx)^{s} p_{n,k}(x). 
\end{equation*}
It is known that $T_{n,s}$ is a polynomial in $x$ of degree at most $s$ and in $n$ of degree at most $\lfloor s/2 \rfloor$. We will employ a few specific formulas for small $s$. For each $x\in[0,1]$, with the short hand notation $X:=x(1-x)$, we have
\begin{equation}
	\label{eq:Tprop0}
	\begin{split}
		&T_{n,0}(x) = 1, \quad T_{n,1}(x) = 0, \quad T_{n,2}(x) = nX, \\
		&T_{n,3}(x) = n(1-2x)X, \quad T_{n,4}(x)=3n^2X^2+n(X-6X^2).
	\end{split}
\end{equation}
For each integer $s\geq 0$, there is a constant $A_s$ such that for all $n \geq 1$ and $x \in [0,1]$, we have
\begin{equation} 
	\label{eq:Tbound}
	0\leq  T_{n,2s}(x) \leq A_s n^{s}.
\end{equation}
See \cite[Chapter 10]{DL} for proofs of the above results. Although we will not need explicit upper bounds for $A_s$, some can be found in \cite{adell2015estimates,molteni2022explicit}. 

For dimension $d\geq 1$ and integer $n\geq 1$, Bernstein polynomials are defined to be tensor products of single variable Bernstein polynomials. 

\begin{definition}
	\label{def:Bernstein}
	The multivariate Bernstein polynomials of order $n$ is $\calB_n:=\{p_{n,\bfk}\}_{0\leq \bfk\leq n}$, where each $p_{n,\bfk}\colon [0,1]^d\to [0,1]$ is defined as, 
	$$
	p_{n,\bfk}(\bfx)
	:=p_{n,k_1}(x_1) \cdots p_{n,k_d}(x_d)
	=\binom{n}{\bfk} \bfx^{\bfk} (1-\bfx)^{n-\bfk}. 
	$$
\end{definition}

Whenever there is a $1\leq\ell\leq d$ such that either $k_\ell<0$ or $k_\ell > n$, then we define $p_{n,\bfk} = 0$. It is important to mention that we exclusively use multivariate Bernstein polynomials that are formed as tensor products. They are significantly different from Bernstein polynomials on the canonical $d$-dimensional simplex. 

Several properties of multivariate Bernstein polynomials can be deduced by exploiting their tensor product structure. They form a partition of unity for $[0,1]^d$, and in particular, 
\begin{equation}
	\label{eq:partitionunity}
	\sum_{k_\ell=0}^n p_{n,k_\ell}(x_\ell) = 1 \forallspace 1\leq \ell\leq d \andspace \bfx\in[0,1]^d.
\end{equation}
The Bernstein polynomial of a multivariate $f$ is defined similar to before, 
\begin{align*}
	B_n(f)(\bfx)
	&:=\sum_{0\leq \bfk\leq n} f\(\frac{\bfk}{n}\) p_{n,\bfk}(\bfx)
	=\sum_{k_1=0}^n \cdots \sum_{k_d=0}^n f\(\frac{k_1}{n},\cdots,\frac{k_d}{n}\)\, p_{n,k_1}(x_1)\cdots p_{n,k_d}(x_d). 
\end{align*}
Central moments of multivariate Bernstein polynomials can be readily extracted. For any $\bfalpha\in\N^d$, we define
$$
T_{n,\bfalpha}(\bfx)
:=\sum_{0\leq \bfk\leq n} (\bfk-n\bfx)^{\bfalpha} p_{n,\bfk}(\bfx)
=T_{n,\alpha_1}(x_1)\cdots T_{n,\alpha_d}(x_d).
$$
Letting $A_{\bfalpha}:=A_{\alpha_1}\cdots A_{\alpha_d}$, it follows from \eqref{eq:Tbound} that 
\begin{equation}
	\label{eq:Tbound2}
	0
	\leq T_{n,2\bfalpha}(\bfx)
	%\leq A_{\alpha_1}\cdots A_{\alpha_d} n^{\alpha_1}\cdots n^{\alpha_d}
	\leq A_{\bfalpha}n^{|\bfalpha|}.	
\end{equation}

\subsection{Approximation by the Bernstein operator}

We start by investigating the approximation properties of the Bernstein operator. We say a function $f\colon [0,1]^d\to\R$ is Lipschitz continuous if there is a $L\geq 0$ for which $|f(\bfx)-f(\bfy)|\leq L\|\bfx-\bfy\|_2$ for all $\bfx,\bfy\in [0,1]^d$. We let $|f|_\Lip$ be the smallest such $L$ for which this inequality holds. The following is an elementary observation regarding the approximation of  Lipschitz functions by the Bernstein operator, see also \cite{heitzinger2002simulation}. 

\begin{proposition}
	\label{prop:lipbern}
	For any Lipschitz continuous $f$ on $[0,1]^d$, we have
	$$
	\|f-B_n(f)\|_\infty 
	\leq \frac{|f|_{\Lip}}{2}\sqrt{\frac{d}{n}}. 
	$$
\end{proposition}

\begin{proof}
	Fix any $\bfx\in [0,1]^d$. By the partition of unity property \eqref{eq:partitionunity}, we have 
	$$
	\Big| f(\bfx) - \sum_{0\leq \bfk\leq n} f\( \frac{\bfk}{n}\) p_{n,\bfk}(\bfx) \Big|
	= \Big|\sum_{0\leq \bfk\leq n} \(f(\bfx)-f\( \frac{\bfk}{n}\)\) \ p_{n,\bfk}(\bfx)\Big|.
	$$
	We use the Lipschitz condition on $f$ to see that
	$$
	\Big|\sum_{0\leq \bfk\leq n} \(f(\bfx)-f\( \frac{\bfk}{n}\)\) \ p_{n,\bfk}(\bfx)\Big|
	\leq \frac{|f|_{\Lip}}{n} \sum_{0\leq \bfk\leq n} \|n\bfx-\bfk\|_2 \ p_{n,\bfk}(\bfx).
	$$
	By Cauchy-Schwarz, partition of unity property \eqref{eq:partitionunity}, and moments identity \eqref{eq:Tprop0}, we have
	\begin{align*}
		&\sum_{0\leq \bfk\leq n} \|n\bfx - \bfk\|_2 \ p_{n,\bfk}(\bfx) 
		\leq \( \sum_{0\leq \bfk\leq n} \|nx - \bfk\|^2_2 \ p_{n,\bfk}(\bfx) \)^{1/2} \( \sum_{0\leq \bfk\leq n}  p_{n,\bfk}(\bfx) \)^{1/2} \\
		&=\( \sum_{\ell=1}^d \sum_{0\leq \bfk\leq n}  (nx_\ell-k_\ell)^2\ p_{n,\bfk}(\bfx) \)^{1/2} 
		=\( \sum_{\ell=1}^d T_{n,2}(x_\ell) \)^{1/2} \leq \frac{\sqrt{n d}}{2}.
	\end{align*}
\end{proof}

For the one dimensional case, if the target function is twice differentiable, then it is possible to obtain a faster rate of decay in $n$, see \cite[Chapter 10]{DL}. We proceed to derive an analogous result for the multivariate case. Before we state it, for concreteness, let us introduce some notation. 

For a $\bfalpha\in\N^d$, we use the shorthand notation $\partial^{\bfalpha}_{\bfx}:=\partial^{\bfalpha}:=\partial^{\alpha_1}_{x_1}\cdots \partial^{\alpha_d}_{x_d}$. For any integer $s\geq 1$, we let $C^s([0,1]^d)$ be the set of continuous real functions $f$ defined on a $[0,1]^d$ such that $\partial^{\bfalpha} f$ is continuous for all $|\bfalpha|\leq s$. We define the semi-norm $|f|_{\dot C^s}
:=\sup_{|\bfalpha|=s} \|\partial^{\bfalpha} f\|_{\infty}$ and equip $C^s(U)$ with the norm,
$$
\|f\|_{C^s}
=\max \Big\{ \|f\|_\infty, \, \max_{1\leq k\leq s} |f|_{\dot C^k} \Big\}. 
$$
We also define $C^{s}\Lip$ to be the set of functions that are $s$ times differentiable and whose $s$-th order partial derivatives are Lipschitz continuous, and we equip $C^s\Lip$ with the norm
$$
\|f\|_{C^s\Lip}
=\max \Big\{ \|f\|_{C^s}, \, \max_{|\bfalpha|=s} |\partial^{\bf\alpha} f|_{\Lip} \Big\}.
$$

\begin{proposition}
	\label{prop:C2bern}
	For any $f\in C^1\Lip([0,1]^d)$, we have
	$$
	\|f-B_n(f)\|_\infty 
	\leq \frac{d}{8n}\|f\|_{C^1\Lip}. 
	$$
\end{proposition}

\begin{proof}
	Fix any $\bfx\in [0,1]^d$. For each $0\leq \bfk\leq n$, by Taylor's theorem, 
	\begin{align*}
		f\( \frac{\bfk}{n}\)
		&=f(\bfx)+\sum_{\ell=1}^d \partial_{\ell} f(\bfx) \(\frac{k_\ell}{n}-x_\ell\)  \\
		&\quad + \sum_{\ell=1}^d \(\frac{k_\ell}{n}-x_\ell\) \int_0^1 \Big[ \partial_\ell f\(\bfx+t\(\frac{k_\ell}{n}-x_\ell\)\boldsymbol{e}_\ell\)-\partial_\ell f(\bfx) \) \Big] \ dt.
	\end{align*}
	This equation, the partition of unity property \eqref{eq:partitionunity}, and the moments identity \eqref{eq:Tprop0} (in particular that $T_{n,1}(x_\ell)=0$ for each $\ell$), we deduce 
	\begin{align*}
		&B_n(f)(\bfx)-f(\bfx)
		=\sum_{0\leq \bfk\leq n} \Big( f\( \frac{\bfk}{n}\)-f(\bfx)\Big)\, p_{n,\bfk}(\bfx)\\
		%		&=\sum_{0< \alpha\leq 1} \frac{\partial^\alpha f(x)}{\alpha!} \sum_{0\leq k\leq n} \(\frac{k}{n}-x\)^\alpha p_{n,k}(x) \\
		%		&\quad + \sum_{0\leq k\leq n} \sum_{|\alpha|=2} \frac{2}{\alpha!} \(\frac{k}{n}-x\)^\alpha p_{n,k}(x)\int_0^1 (1-t)^r \partial^\alpha f\(x+t \, \big( \frac{k}{n} -x \big)\)\ dt \\
		&\quad = \sum_{0\leq \bfk\leq n} \sum_{\ell=1}^d \(\frac{k_\ell}{n}-x_\ell\)p_{n,\bfk}(x) \int_0^1 \Big[ \partial_\ell f\(\bfx+t\(\frac{k_\ell}{n}-x_\ell\)\boldsymbol{e}_\ell\)-\partial_\ell f(\bfx) \) \Big] \ dt. 
	\end{align*}
	From this equation, that the Bernstein polynomials are nonnegative, the Lipschitz assumption on the partial derivatives of $f$, and partition of unity \eqref{eq:partitionunity}, we see that 
	\begin{align*}
		|f(\bfx)-B_n(f)(\bfx)|
		&\leq \frac{\|f\|_{C^1\Lip}}{2n^2} \sum_{0\leq \bfk\leq n} \sum_{\ell=1}^d (k_\ell-nx_\ell)^2 p_{n,\bfk}(\bfx) = \frac{\|f\|_{C^1\Lip}}{2n^2} \sum_{\ell=1}^d T_{n,2}(x_\ell). 
	\end{align*}
	Applying the moments identity \eqref{eq:Tprop0} completes the proof. 
\end{proof}

There are other properties of Bernstein operators, such as convergence of partial derivatives \cite{fellhauer2016approximation} that we will not use in this paper.

\subsection{Approximation of smooth functions with iterated Bernstein}

As mentioned earlier, due to the saturation phenomenon for the Bernstein operator, there is no hope of improving the decay rate of $1/n$ in \cref{prop:C2bern} even under additional regularity assumptions on the target function. To overcome this, iterated univariate Bernstein operators were developed by Micchelli \cite{micchelli1973} and Felbecker \cite{felbecker1979} as alternative means of approximation. By viewing $B_n$ as an operator on $C([0,1]^d)$, for any integer $r\geq 1$, we define the following iterated Bernstein operator 
\begin{equation*}
	U_{n,r}:=I-(I-B_n)^r
	=\sum_{j=1}^r (-1)^{j-1} \binom{r}{j} B_n^j. 	
\end{equation*}
Note that $U_{n,1}=B_n$ coincides with the usual Bernstein operator. The following theorem generalizes the core results of Micchelli-Felbecker to higher dimensions.

\begin{theorem}
	\label{thm:approxsmooth}
	For any integers $s,d\geq 1$ and any $f\in C^{s}([0,1]^d)$, it holds that 
	$$
	\|f-U_{n,\lceil s/2 \rceil}(f)\|_\infty 
	\lesssim_{s,d} \|f\|_{C^s}n^{-s/2}. 
	$$
\end{theorem}

\begin{proof}
	It suffices to prove that, for all $s\geq 1$ and $f\in C^s([0,1]^d)$, we have 
	\begin{equation}
		\label{eq:Uhelp1}
		\|f-U_{n,\lceil s/2 \rceil}(f)\|_\infty 
		=\|(I-B_n)^{\lceil s/2\rceil}(f)\|_\infty
		\lesssim_{s,d} \|f\|_{C^s}n^{-s/2}.
	\end{equation}

	We proceed by strong induction. The $s=1,2$ cases hold by Propositions \ref{prop:lipbern} and \ref{prop:C2bern}. Hence, assume that inequality \eqref{eq:Uhelp1} holds for $s=1,2,\dots,r$ for some $r\geq 2$. Now we will prove that the statement holds for $s=r+1$.  
	
	To this end, fix a $f\in C^{r+1}([0,1]^d)$ and any $\bfx\in [0,1]^d$. For each $0\leq \bfk\leq n$, there is a $\xi_{\bfk,\bfx}$ such that 
	\begin{align*}
		f\( \frac{\bfk}{n}\)
		&=f(\bfx)+\sum_{0<|\bfalpha| \leq r} \frac{\partial^{\bfalpha} f(\bfx)}{\bfalpha!} \(\frac{\bfk}{n}-\bfx\)^{\bfalpha} + \sum_{|\bfalpha|=r+1} \frac{\partial^{\bfalpha} f(\xi_{\bfk,\bfx})}{\bfalpha!} \(\frac{\bfk}{n}-\bfx\)^{\bfalpha}.
	\end{align*}
	From here, we see that
	\begin{gather*}
		\label{eq:Uhelp-1}
		\begin{split}
			&(B_n-I)(f)(\bfx) \\
			&=\sum_{0<|\bfalpha| \leq r} \sum_{0\leq \bfk\leq n} \frac{\partial^{\bfalpha} f(\bfx)}{\bfalpha!} \(\frac{\bfk}{n}-\bfx\)^{\bfalpha}p_{n,\bfk}(\bfx) + \underbrace{\sum_{|\bfalpha|=r+1} \sum_{0\leq \bfk\leq n} \frac{\partial^{\bfalpha} f(\xi_{\bfk,\bfx})}{\bfalpha!} \(\frac{\bfk}{n}-\bfx\)^{\bfalpha} p_{n,\bfk}(\bfx)}_{:=G_{n,r+1}(\bfx)}.
			%&=\sum_{0<|\bfalpha|\leq r} \frac{\partial^{\bfalpha} f(\bfx)}{n^{|\bfalpha|} \bfalpha!} \ T_{n,\bfalpha}(\bfx)+G_{n,r+1}(\bfx),
		\end{split}
	\end{gather*}
	To simplify the first term, notice that since the first central moment of Bernstein polynomials is identically zero, $T_{n,1}=0$. Hence all terms for which $|\bfalpha|=1$ disappear. For each $2\leq |\bfalpha|\leq r$ and $2\leq m\leq r$, we define the functions,
	$$
	F_{n,\bfalpha}
	:=\frac{\partial^{\bfalpha} f \ T_{n,\bfalpha}}{n^{|\bfalpha|/2} \bfalpha!}, 
	\quad\text{and}\quad
	F_{n,m}:= \sum_{|\bfalpha|=m} F_{n,\bfalpha}. 
	$$
	It follows that 
	$$
	(B_n-I)(f)(\bfx)
	=\sum_{2\leq |\bfalpha|\leq r} \frac{F_{n,\bfalpha}(\bfx)}{n^{|\bfalpha|/2}} +G_{n,r+1}(\bfx)
	=\sum_{m=2}^r \frac{F_{n,m}(\bfx)}{n^{m/2}}  + G_{n,r+1}(\bfx). 
	$$
	This identity holds for each $\bfx$. Noting that $\lceil (r+1)/2\rceil \geq 1$ since $r\geq2$, We have 
	\begin{equation}
		\label{eq:Uhelp2}
		(B_n-I)^{\lceil (r+1)/2 \rceil}(f)
		=\sum_{m=2}^r \frac{(B_n-I)^{\lceil (r+1)/2 \rceil-1} F_{n,m}}{n^{m/2}} + (B_n-I)^{\lceil (r+1)/2 \rceil-1} G_{n,r+1}. 
	\end{equation}
	
	We concentrate on the primary term in \eqref{eq:Uhelp2} first. To apply the inductive hypothesis, we claim that for each $2\leq m\leq r$, we have $F_{n,m}\in C^{r+1-m}$ and 
	\begin{equation}
		\label{eq:Uhelp3}
		\|F_{n,m}\|_{C^{r+1-m}}
		\leq \|f\|_{C^{r+1}} 2^{r+1-m} \sum_{|\bfalpha|=m} \sqrt{A_{\bfalpha}}. 
	\end{equation} 
	The key part of this assertion is that the upper bound for $\|F_{n,m}\|_{C^{r+1-m}}$ does not depend on $n$. Note that for each $|\bfalpha|=m$, we have $\partial^{\bfalpha} f\in C^{r+1-m}$ due to the initial assumption that $f\in C^{r+1}$. Also $T_{n,\bfalpha}$ is infinitely differentiable since it is a multinomial, so we see that $F_{n,m}\in C^{r+1-m}$. By Leibniz, for each $|\bfbeta|\leq r+1-m$, we have
	$$
	\partial^{\bfbeta} F_{n,m}
	=\sum_{|\bfalpha|=m} \frac{1}{n^{m/2} \bfalpha!} \sum_{0\leq \bfgamma\leq \bfbeta} \binom{\bfbeta}{\bfgamma} \partial^{\bfalpha+\bfbeta-\bfgamma} f \ \partial^{\bfgamma} T_{n,\bfalpha}. 
	$$
	Since $T_{n,\bfalpha}(\bfx)$ is a polynomial in $x_\ell$ of degree at most $\alpha_\ell$, the inside summation can be taken over $0\leq \bfgamma\leq \min(\bfalpha,\bfbeta)$. Additionally, Bernstein's inequality for algebraic polynomials and central moment bounds \eqref{eq:Tbound} yield
	\begin{align*}
		\|\partial^{\bfgamma} T_{n,\bfalpha}\|_\infty
		&\leq \prod_{\ell=1}^d \|\partial^{\gamma_\ell} T_{n,\alpha_\ell}\|_\infty
		\leq \prod_{\ell=1}^d \frac{\alpha_\ell!}{(\alpha_\ell-\gamma_\ell)!} \|T_{n,\alpha_\ell}\|_\infty \\
		&\leq \prod_{\ell=1}^d \frac{\alpha_\ell!}{(\alpha_\ell-\gamma_\ell)!} \sqrt{\|T_{n,2\alpha_\ell}\|_\infty}
		\leq \frac{\bfalpha!}{(\bfalpha-\bfgamma)!} \sqrt{ A_{\bfalpha}}\, n^{|\bfalpha|/2}. 
	\end{align*}
	From this and that $|\bfalpha+\bfbeta-\bfgamma|\leq r+1$, it follows that 
	\begin{align*}
		\|\partial^{\bfbeta} F_{n,m}\|_\infty
		&\leq \sum_{|\bfalpha|=m} \sqrt{A_{\bfalpha}} \sum_{0\leq \bfgamma\leq \min(\bfalpha,\bfbeta)} \binom{\bfbeta}{\bfgamma} \frac{1}{(\bfalpha-\bfgamma)!} \, \|\partial^{\bfalpha+\bfbeta-\bfgamma} f\|_\infty \\
		&\leq \|f\|_{C^{r+1}} \sum_{|\bfalpha|=m} \sqrt{A_{\bfalpha}} \sum_{0\leq \bfgamma\leq \bfbeta} \binom{\bfbeta}{\bfgamma}
		= \|f\|_{C^{r+1}} 2^{|\bfbeta|} \sum_{|\bfalpha|=m} \sqrt{A_{\bfalpha}}. 
	\end{align*}
	This completes the proof of \eqref{eq:Uhelp3}.
	
	Returning back to the proof at hand, notice that for each $2\leq m\leq r$, if we define
	$$
	q(m)
	:=\Big\lceil \frac{r+1}{2} \Big\rceil-1 - \Big\lceil \frac{r+1-m}{2}  \Big\rceil,  
	$$
	then we have $q(m)\geq 0$ since $m\geq 2$, and $q(m)\leq (m-1)/2$.
	%	\begin{align*}
		%		q(m) 
		%		&\leq \frac{r+1}{2}+\frac{1}{2}-1-\frac{r+1-q}{2}
		%		=\frac{q-1}{2}. 
		%	\end{align*}
	It follows from the inductive hypothesis that there exist constants $M_{r+1-m,d}>0$ for each $2\leq m\leq r$ such that
	\begin{gather}
		\label{eq:Uhelp4}
		\begin{split}
			\big\|(B_n-I)^{\lceil (r+1)/2 \rceil-1 }F_{n,m}\big\|_\infty 
			&\leq \|B_n-I\|_\infty^{q(m)} \Big\|(B_n-I)^{\lceil (r+1-m)/2 \rceil}F_{n,m} \Big\|_\infty \\
			&\leq 2^{(m-1)/2} M_{r+1-m, d} \|F_{n,m}\|_{C^{r+1- q}} n^{-(r+1-m)/2}. 
		\end{split}
	\end{gather}
	
	We next control the remainder term involving $G_{n,r+1}$. We have the following upper bound for $\|G_{n,r+1}\|_\infty$. By Cauchy-Schwarz, the partition of unity property \eqref{eq:partitionunity}, and central moment bounds \eqref{eq:Tbound2}, we see that for each $\bfx$,
	\begin{gather}
		\label{eq:Uhelp0}
		\begin{split}
			|G_{n,r+1}(\bfx)|
			&\leq \frac{|f|_{\dot C^{r+1}}}{n^{r+1}}\sum_{0\leq \bfk\leq n} \sum_{|\bfalpha|=r+1} \frac{1}{\bfalpha!} \, \big|(\bfk-n\bfx)^{\bfalpha}\big| \ p_{n,\bfk}(\bfx) \\
			&\leq  \frac{|f|_{\dot C^{r+1}}}{n^{r+1}} \sum_{|\bfalpha|=r+1} \frac{1}{\bfalpha!} \   \(\sum_{0\leq \bfk\leq n} (\bfk-n\bfx)^{2\bfalpha} \, p_{n,\bfk}(\bfx)\)^{1/2} 
			%&= \frac{|f|_{\dot C^{r+1}}}{n^{r+1}} \sum_{|\bfalpha|=r+1} \frac{1}{\bfalpha!} \, \big(T_{n,2\alpha_1}(x_1)\cdots T_{n,2\alpha_d}(x_d)\big)^{1/2} \\
			\leq \frac{|f|_{\dot C^{r+1}}}{ n^{(r+1)/2}} \sum_{|\bfalpha|=r+1} \frac{\sqrt{A_{\bfalpha}}}{\bfalpha!}.
		\end{split}
	\end{gather}
	Moreover, $\|B_n-I\|_\infty \leq 2$, and $\lceil (r+1)/2\rceil-1\leq r/2$. Hence we have 
	\begin{equation}
		\label{eq:Uhelp-2}
		\|(B_n-I)^{\lceil (r+1)/2 \rceil-1} G_{n,r+1}\|_\infty 
		\leq 2^{\lceil (r+1)/2 \rceil-1} \|G_{n,r+1}\|_\infty
		\leq 2^{r/2} \|G_{n,r+1}\|_\infty. 
	\end{equation}
	
	Now we are ready to complete the proof. Combining \eqref{eq:Uhelp2}, \eqref{eq:Uhelp3},  \eqref{eq:Uhelp4}, \eqref{eq:Uhelp0}, and \eqref{eq:Uhelp-2}, we see that 
	\begin{align*}
		&\|(B_n-I)^{\lceil (r+1)/2 \rceil}(f)\|_\infty \\
		&\quad\leq \sum_{m=2}^r \frac{\|(B_n-I)^{\lceil (r+1)/2 \rceil-1} F_{n,m}\|_\infty}{n^{m/2}} + \|(B_n-I)^{\lceil (r+1)/2 \rceil-1} G_{n,r+1}\|_\infty \\ 
		&\quad\leq \sum_{m=2}^r \frac{2^{(m-1)/2} M_{r+1-m, d} \|F_{n,m}\|_{C^{r+1-m}}}{n^{(r+1)/2}} + \frac{2^{r/2} |f|_{\dot C^{r+1}}}{ n^{(r+1)/2}} \sum_{|\bfalpha|=r+1} \frac{\sqrt{A_{\bfalpha}}}{\bfalpha!} \\ 
		&\quad \lesssim_{r+1,d} \frac{\|f\|_{C^{r+1}}}{ n^{(r+1)/2}}. 
	\end{align*}
	This completes the proof by induction.
	
\end{proof}
	
%\cref{thm:approxsmooth} is stated for $C^s([0,1]^d)$ functions for simplicity and ease of exposition. The proof holds under a weaker regularity assumption that $f\in C^{s-1}\Lip([0,1]^d)$. To see this, one uses the integral form of the Taylor remainder theorem instead of the pointwise one in the theorem's proof. This is also consistent with \cref{prop:lipbern}.

Let us briefly discuss our result in the context of classical approximation results of Micchelli \cite{micchelli1973} and Felbecker \cite{felbecker1979}. The main distinction is that our result holds for iterates of multivariate Bernstein polynomials (formed as tensor products), while the classical papers only treat the univariate case. One small improvement we made deals the parity of $s$. Micchelli only treated the even case, which loses a $n^{-1/2}$ factor for the odd cases, whereas the odd case was satisfactorily derived by Felbecker. In our proof, we combined the even and odd cases together seamlessly. The extension of these classical results to higher dimensions was perhaps known by experts, though we were unable to find a reference.

On the other hand, there are alternative generalizations of Bernstein polynomials to the canonical simplex in $\R^d$, as opposed to the unit cube. Such polynomials are significantly different from the tensor product ones employed in this paper. For Bernstein polynomials on the simplex, approximation rates of Micchelli-Felbecker iterations have been studied, see \cite{gonska1994approximation,ding2008k} and references therein. 

\subsection{From iterated Bernstein to linear combinations}

In the next section on quantization, we will show to quantize the coefficients of a function written in the Bernstein basis. For this reason, we show how to relate the iterated Bernstein approximation $U_{n,r}(f)$ to the Bernstein polynomial $B_n(f_{n,r})$ of a possibly different function $f_{n,r}$, which can be found constructively via the formula, 
\begin{equation}
	\label{eq:fnr}
	f_{n,r}	:=\( I+\sum_{m=1}^{r-1} (I-B_n)^m\)(f). 
\end{equation}

\begin{theorem}
	\label{thm:iteratetolinear}
	For any integers $r,d\geq 1$, any $f\in C([0,1]^d)$, and any $n\geq 1$, we have 
	$$
	U_{n,r}(f)
	=B_n(f_{n,r}),
	$$
	where $f_{n,r}$ is defined in \eqref{eq:fnr}. Further, it holds that
	$$
	\|f_{n,r}\|_\infty 
	\leq \|f\|_\infty + (2^{r-1}-1)\|f-B_n(f)\|_\infty. 
	$$
\end{theorem}

\begin{proof}
	\cite[Theorem 5]{onebitBernstein} proved this for $d=1$ case, but the same argument extends to the Bernstein operator on $C([0,1]^d)$ without any modifications to the proof. 
\end{proof}

This theorem shows that not only $U_{n,r}(f)=B_n(f_{n,r})$, but it also implies that the coefficients $\{a_{\bfk}\}_{0\leq \bfk\leq n}$ in the following theorem are not much larger than $\|f\|_\infty$. This will be important in the next section, where the coefficients will be fed into a particular quantization algorithm called $\Sigma\Delta$, which requires a $\ell^\infty$ assumption on the coefficients. 

\begin{theorem}
	\label{thm:bapprox}
	Let $d,s,n\geq 1$, $\mu,\delta\in (0,1)$, and $f\in C^s([0,1]^d)$ with $\|f\|_\infty\leq \mu$. If $s\geq 3$, also assume that 
	$
	\displaystyle n\geq \frac{\sqrt{2^{s+1}} d}{8\delta} \|f\|_{C^1\Lip}.
	$ 
	There exist $\{a_{\bfk}\}_{0\leq \bfk\leq n}$ such that $\|a\|_\infty\leq \mu+\delta$ and 
	\begin{equation}
		\label{eq:bapprox}
		\Big\|f-\sum_{0\leq \bfk\leq n} a_{\bfk} p_{n,\bfk}\Big\|_\infty
		\lesssim_{s,d} \|f\|_{C^s} n^{-s/2}. 	
	\end{equation}
\end{theorem}

\begin{proof}
	We consider different cases depending on $s$. 
	
	If $s=1$, then we let $a_{\bfk}=f(\bfk/n)$, so $\|a\|_\infty\leq \|f\|_\infty\leq \mu$. The conclusion follows from \cref{prop:lipbern}. In this case, the implicit constant in \eqref{eq:bapprox} is $\sqrt d/2$. 
	
	If $s=2$, then we let $a_{\bfk}=f(\bfk/n)$, so $\|a\|_\infty\leq \|f\|_\infty\leq \mu$. The conclusion follows from \cref{prop:C2bern}. In this case, the implicit constant in \eqref{eq:bapprox} is $d^2/8$.
	
	Suppose $s\geq 3$. Consider the function $h:=U_{n,\lceil s/2\rceil}(f)$. From \cref{thm:approxsmooth}, we have 
	$$
	\|f-h\|_\infty \lesssim_{s,d} \|f\|_{C^s}  n^{-s/2}. 
	$$
	From \cref{thm:iteratetolinear} and \cref{prop:C2bern}, we have $h=B_n(f_{n,\lceil s/2\rceil})$, where $f_{n,\lceil s/2\rceil}$ is defined in \eqref{eq:fnr}, and 
	\begin{align*}
		\|f_{n,\lceil s/2\rceil}\|_\infty
		&\leq \|f\|_\infty + (2^{\lceil s/2\rceil}-1) \,\|f-B_n(f)\|_\infty 
		\leq \mu + \frac{\sqrt{2^{s+1}}d}{8n} \|f\|_{C^1\Lip}. 
	\end{align*}
	Pick any $n$ sufficiently large so that the right hand side is bounded above by $\mu+\delta$. We set $a_{\bfk}:=f_{n,\lceil s/2\rceil}(\bfk/n)$ so that
	$$
	h
	=B_n(f_{n,\lceil s/2\rceil})
	=\sum_{0\leq \bfk\leq n} a_{\bfk} p_{n,\bfk}, 
	$$
	which completes the proof. 
\end{proof}

\section{Quantization error}
\label{sec:Bquan}

The goal of this section is to show that, given a polynomial whose coefficients in the Bernstein basis is $\{a_{\bfk}\}_{0\leq \bfk\leq n}$ where $\|a\|_\infty \leq \mu$, for a prescribed $\mu\in (0,1)$, we can find a sequence $\{\sigma_{\bfk}\}_{0\leq \bfk\leq n}\subset \{\pm 1\}$ such that the quantization error induced on the Bernstein basis, 
\begin{equation*}
	\calE_{n,a}^{\text{quan}}(\bfx)
	:=\sum_{0\leq \bfk\leq n} (a_{\bfk}-\sigma_{\bfk}) \ p_{n,\bfk}(\bfx),
\end{equation*}
is small in a suitable sense. \cref{thm:rsigmadelta} will provide us with an explicit algorithm that will convert $\{a_{\bfk}\}_{0\leq \bfk\leq n}$ into an appropriate $\{\sigma_{\bfk}\}_{0\leq \bfk\leq n}$.

\subsection{Background on $\Sigma\Delta$ quantization}

\label{sec:quanprelim}

$\Sigma\Delta$ quantization (or modulation) refers to a large family of algorithms designed to convert any given sequence $y:=\{y_k\}_{k\in\N}$ of real numbers in a given set $\calY$ to another sequence $q:=\{q_k\}_{k\in\N}$ taking values in a prescribed discrete set $\A$, typically selected as an arithmetic progression. This is done in such a way that the quantization error is a ``high-pass'' sequence, in the sense that inner products of $q$ with slowly varying sequences are small. 

A canonical way of ensuring this is to ask that for any $y\in \calY$, there exist $q$ and a ``state" sequence $\{u_k\}_{k\in \Z}$ that satisfy the $r$-th order difference equation
\begin{equation} 
	\label{sigmadeltaeq}
	y - q = \Delta^r u, 
	\andspace 
	(\Delta u)_k:= u_k - u_{k-1}.
\end{equation}
If this is possible, we then say that $q$ is an $r$-th order noise-shaped quantization of $y$. When \eqref{sigmadeltaeq} is implemented recursively, it means that each $q_k$ is found by means of a quantization rule of the form
\begin{equation*} 
	q_k = F(u_{k-1},u_{k-2},\dots,y_k,y_{k-1},\dots),
\end{equation*}
and $u_k$ is updated via
\begin{equation*}
	u_k = \sum_{j=1}^r (-1)^{j-1} \binom{r}{j} u_{k-j} + y_k - q_k.
\end{equation*}

\begin{definition}
	A quantization rule is stable for $\calY$ if for each $y\in \calY$, there is a sequence $u$ satisfying \eqref{sigmadeltaeq} such that $\|u\|_\infty$ is bounded uniformly in $y$.
\end{definition} 

%The ``greedy'' quantization rule refers to the function $F$ which outputs any minimizer $q_k \in \A$ of the right hand side of \eqref{sigmadeltaupdate}. More precisely, it sets 
%\begin{equation}
%	q_k:=\mathrm{round}_\A \(\sum_{j=1}^r (-1)^{j-1} \binom{r}{j} u_{k-j} + y_k\),
%\end{equation}
%where $\mathrm{round}_\A(v)$ stands for any element in $\A$ that is closest to $v$. Difficulties start when $\A$ is a fixed and finite set. In the extreme case being a set of two elements, $\A$ can be taken to be $\{\pm 1\}$ without loss of generality. In this case, it is a major challenge to design quantization rules that are stable for any order $r$ and for arbitrary bounded inputs in a range $[-\mu,\mu]$. 

Stability is a desirable property for a quantization algorithm as it allows one to control the error $y-q$ uniformly in $y\in\calY$. Establishing the existence of a stable $r$-th order scheme for fixed and finite $\calA$, especially in the extreme one-bit case where $\calA=\{\pm a\}$ for some $a\not=0$, is difficult. The first breakthrough on this problem was made in the seminal paper of Daubechies and DeVore \cite{DD}. There it was shown for $\A = \{\pm 1\}$, any $r\geq 1$, and any $\mu \in (0,1)$, there exists a stable $r$-th order $\Sigma\Delta$ quantizer such that whenever $\|y\|_\infty \leq \mu$, it holds that $\|u\|_\infty \leq C_{r,\mu}$. The constant $C_{r,\mu}$ depends only on $r$ and $\mu$, and blows up as $r\to\infty$ or $\mu \to 1$, except for when $r=1$. For $r=1$, it suffices to take $C_{r,\mu}=1$ and $\mu\in (0,1]$. Another family of stable $\Sigma\Delta$ quantizers, but with more favorable $C_{r,\mu}$, was subsequently proposed in \cite{exp_decay}.

In this paper, $r$-th order $\Sigma\Delta$ quantization refers to either of the stable rules in \cite{DD} and \cite{exp_decay} for when $\A = \{\pm 1\}$. Except for \cref{sec:finalremarks}, we will use $\Sigma\Delta$ solely as a method of approximation, and will not need explicit descriptions of these rules or precise estimates for $C_{r,\mu}$. If new stable $r$-th order $\Sigma\Delta$ quantization methods are developed, those can be used instead of the ones mentioned here, without changes to the remaining parts of this paper. 

%$\Sigma\Delta$ quantization holds for sequences indexed by $\N$. Their extensions to sequences indexed byn $\N^d$, the main downside is that the only known stable $\Sigma\Delta$ schemes require $d$-bit alphabets. Hence, they are not helpful whenever $d$ exceeds, say 16 or 32, which is essentially always the case in machine learning applications. To circumvent this issue, we employ a directional $\Sigma\Delta$ scheme, and while naive, it will suffice for purpose of quantizing the coefficients in the Bernstein basis. 

We will need to extend $\Sigma\Delta$ to sequences indexed by $\N^d$. Let $u$ be a function on $\N^d$. For each integer $1\leq \ell\leq d$, we denote the finite difference operator in the $\ell$-th coordinate by $\Delta_\ell$, which acts on $u$ by the formula 
$$
(\Delta_\ell u)_{\bfk}
:= u_{k_1,\dots,k_{\ell-1},k_\ell,k_{\ell+1},\dots,k_d} - u_{k_1,\dots,k_{\ell-1},k_\ell-1,k_{\ell+1},\dots,k_d}
$$
The following shows that there exists a stable directional $r$-th order $\Sigma\Delta$ quantizer.  

\begin{proposition}
	\label{prop:sigdeltaell}
	For any $\mu\in (0,1)$, and integers $r,d\geq 1$, there exists a $C_{r,\mu}>0$ that depends only on $r$ and $\mu$ such that the following hold. For any function $y$ defined on $\N^d$ such that $\|y\|_\infty \leq \mu$ and any integer $1\leq \ell\leq d$, there exist a function $q$ on $\N^d$ where $q_{\bfk}\in \{\pm 1\}$ for all $\bfk\in\N^d$ and a function $u$ on $\Z^d$ supported in $\N^d$, such that $\|u\|_{\infty}\leq C_{r,\mu}$ and
	$$
	y-q = \Delta_\ell^r u.
	$$
	For $r=1$, we only require that $\mu\in (0,1]$ and the statement holds for $C_{1,\mu}=1$.
\end{proposition}

\begin{proof}
	Fix an integer $1\leq \ell \leq d$, and for each $\bfk\in\N^d$, we define 
	$$
	\bfk'=(k_1,\dots,k_{\ell-1},k_{\ell+1},\dots,k_d)\in \N^{d-1}. 
	$$
	Now we form a sequence from $y$ by setting all indices except for the $\ell$-th one to be $\bfk'$, 
	$$
	y_{\bfk'}:= \big\{y_{k_1,\dots, k_{\ell-1}, j,k_{\ell+1},\dots,k_d} \big\}_{j\in\N} . 
	$$
	Using either one of the two quantization rules pointed out in \cref{sec:quanprelim}, we define a $q$ on $\N^d$ such that $q_{\bfk'}\in \{\pm1\}$ and is defined as a solution to the $r$-th order difference equation
	$$
	y_{\bfk'}-q_{\bfk'}
	=\Delta^r (u_{\bfk'}).  
	$$
	This scheme is stable, in the sense that there exists $C_{r,\mu}>0$ such that for each $\bfk'$, we have  
	$
	\|u_{\bfk'}\|_\infty 
	\leq C_{r,\mu}.
	$
	Hence $\|u\|_\infty\leq C_{r,\mu}$. 
\end{proof}

It follows from this proposition that any stable $r$-th order $\Sigma\Delta$ can be used to generate a stable $r$-th order scheme in arbitrary dimensions. 

\begin{definition}
	\label{def:dirSigDelta}
	A $r$-th order $\Sigma\Delta$ applied to the $\ell$-th direction is a map that satisfies the conclusions of \cref{prop:sigdeltaell}. 
\end{definition}

\subsection{Directional $\Sigma\Delta$ quantization on Bernstein polynomials}

Going back to our original goal of quantizing the coefficients $a=\{a_{\bfk}\}_{0\leq \bfk\leq n}$ of a polynomial in the Bernstein basis, we process the coefficients with a $r$-th order $\Sigma\Delta$ applied to the $\ell$-th direction, as in \cref{prop:sigdeltaell}. Letting $\{\sigma_{\bfk}\}_{0\leq\bfk\leq n}$ be the $\pm 1$ sequence produced by this algorithm and $u$ denote the state, the quantization error induced on the Bernstein basis is
\begin{align}
	\label{eq:summedbyparts}
	\calE_{n,a}^{\text{quan}}(\bfx) 
	= \sum_{0\leq \bfk\leq n} (\Delta_\ell ^r u)_{\bfk} \, p_{n,\bfk}(\bfx)
	= \sum_{0\leq \bfk\leq n} u_{\bfk} \, \big((\Delta^*_\ell)^r p_{n,\cdot}(\bfx)\big)_{\bfk}. 
\end{align}
Here, we let $\Delta_\ell^*$ be the adjoint of $\Delta_\ell$, and 
$$
(\Delta_\ell^* \, p_{n,\cdot}(\bfx))_{\bfk}
:= \big(p_{n,k_\ell}(x_\ell)-p_{n,k_\ell+1}(x_\ell)\big) \, \prod_{j\not=\ell} p_{n,k_j}(x_j).
$$
Notice that with our convention that $p_{n,{\bfk}}=0$ if there is a $\ell$ such that $k_\ell>n$, all the boundary terms are correctly included in \eqref{eq:summedbyparts}. We are now ready to state the following result for first order $\Sigma\Delta$ on the Bernstein basis. 

\begin{proposition}
	\label{prop:1stsigdelta}
	For any integers $n,d\geq 1$, $\{a_{\bfk}\}_{0\leq \bfk\leq n}$ with $\|a\|_\infty \leq 1$, and $1\leq \ell\leq d$, if $\{\sigma_{\bfk}\}_{0\leq \bfk\leq n}$ is the output of a stable first order $\Sigma\Delta$ quantizer applied in the $\ell$-th direction with input $a$, then 
	$$
	\Big|\sum_{0\leq \bfk\leq n} (a_{\bfk}-\sigma_{\bfk}) \ p_{n,\bfk}(\bfx)\Big|
	\leq \min\big(2, n^{-1/2} x_\ell^{-1/2}  (1-x_\ell)^{-1/2} \big). 
	$$
	
\end{proposition}

\begin{proof}
	It follows from the inequality $\|u\|_\infty \leq C_{1,\mu}=1$ and identity \eqref{eq:summedbyparts} for $r=1$, we have 
	\begin{equation*}
		\Big|\sum_{0\leq \bfk\leq n} (a_{\bfk}-\sigma_{\bfk}) \ p_{n,\bfk}(\bfx)\Big| 
		\leq \sum_{0\leq \bfk\leq n} \big|(\Delta^*_\ell p_{n,\cdot}(\bfx))_{\bfk} \big|.
	\end{equation*}
	The consecutive differences of the $p_{n,k}$, for the one dimensional case, satisfies the following identity: for all $x\in [0,1]$ and $0\leq k\leq n$, 
	\begin{equation}
		\label{eq:Bdiff}
		p_{n,k}(x) - p_{n,k+1}(x) = \frac{(k+1)-(n+1)x}{(n+1)x(1-x)} \
		p_{n+1,k+1}(x).
	\end{equation}
	See \cite[Chapter 10]{DL} and \cite{lorentz1} for this identity. When interpreting the right hand side of this equation for $x=0$ or $x=1$, it should be observed that the polynomial $ ((k{+}1){-}(n{+}1)x)p_{n+1,k+1}(x)$ is divisible by $x(1-x)$ for each $k$. We proceed to extend \eqref{eq:Bdiff} to the multivariate case. For each integer $1\leq \ell\leq d$, we have 
	\begin{equation}
		\label{eq:adjointB}
		\big(\Delta_\ell^* \, p_{n,\cdot}(\bfx)\big)_{\bfk}
		= \( \frac{(k_\ell+1)-(n+1)x_\ell}{(n+1)x_\ell(1-x_\ell)} \
		p_{n+1,k_\ell+1}(x_\ell) \) \prod_{j\not=\ell} p_{n,k_j}(x_j).
	\end{equation}

	By identity \eqref{eq:adjointB}, the partition of unity property \eqref{eq:partitionunity}, Cauchy-Schwarz, and central moment bounds \eqref{eq:Tbound}, we have
	\begin{align*}
		\sum_{0\leq \bfk\leq n} \big|(\Delta^*_\ell p_{n,\cdot}(\bfx))_{\bfk} \big|
		&=\sum_{0\leq \bfk\leq n}  \Big| \frac{(k_\ell+1)-(n+1)x_\ell}{(n+1)x_\ell(1-x_\ell)} \Big| \
		p_{n+1,k_\ell+1}(x_\ell)  \, \prod_{j\not=\ell} p_{n,k_j}(x_j) \\
		&=\frac{1}{(n+1)x_\ell (1-x_\ell)}\sum_{k_\ell=0}^n \big|(k_\ell+1)-(n+1)x_\ell\big| \
		p_{n+1,k_\ell+1}(x_\ell) \\
		&\leq \frac{\sqrt{T_{n+1,2}(x_\ell)}}{(n+1)x_\ell (1-x_\ell)} \(\sum_{k_\ell=0}^n p_{n+1,k_\ell+1}(x_\ell)\)^{1/2} \\
		&= \frac{1}{\sqrt{(n+1)x_\ell (1-x_\ell)}}. 
	\end{align*}
	On the other hand, we have the trivial bound for the quantization error, 
	$$
	\Big|\sum_{0\leq \bfk\leq n} (a_{\bfk}-\sigma_{\bfk}) \ p_{n,\bfk}(\bfx)\Big| 
	\leq \|a-\sigma\|_\infty \sum_{0\leq \bfk\leq n} p_{n,\bfk}(\bfx)
	= \|\Delta_\ell \, u\|_\infty 
	\leq 2. 
	$$
\end{proof}

For larger values of $r$, due to an increasing complexity in the formulas for $(\Delta^*_\ell)^rp_{n,\cdot}(\bfx)$, we do not provide explicit upper bounds. Our strategy for deriving an upper bound for the quantization error builds upon a corresponding one-dimensional result in \cite[Theorem 6]{onebitBernstein}. 

\begin{theorem}
	\label{thm:rsigmadelta}
	For any integers $n,d,r\geq 1$, $\mu\in (0,1)$, $\{a_{\bfk}\}_{0\leq \bfk\leq n}$ with $\|a\|_\infty\leq\mu$, and $1\leq \ell\leq d$, if $\{\sigma_{\bfk}\}_{0\leq \bfk\leq n}$ is the output of a stable $r$-th order $\Sigma\Delta$ quantization applied in the $\ell$-th direction with input $\{a_{\bfk}\}_{0\leq \bfk\leq n}$, then
	$$
	\Big|\sum_{0\leq \bfk\leq n} (a_{\bfk}-\sigma_{\bfk}) \ p_{n,\bfk}(\bfx)\Big|
	\lesssim_{r,\mu} \min \big( 1, n^{-r/2} x_\ell^{-r}(1-x_\ell)^{-r}\big).
	$$
\end{theorem}

\begin{proof}
	Since $\Delta_\ell^*$ is applied to the $\ell$-th direction and the multivariate Bernstein polynomials are tensor products, we have 
	$$
	\big((\Delta^*_\ell)^r \, p_{n,\cdot}(\bfx)\big)_{\bfk}
	=\big((\Delta^*)^r p_{n,\cdot}(x_\ell) \big)_{k_\ell} \(\prod_{j\not=\ell} p_{n,k_j}(x_j) \) .
	$$
	Employing \eqref{eq:summedbyparts}, non-negativity of the Bernstein polynomials, and partition of unity \eqref{eq:partitionunity},
	\begin{align*}
	\Big|\sum_{0\leq \bfk\leq n} (a_{\bfk}-\sigma_{\bfk}) \ p_{n,\bfk}(\bfx)\Big|
	&\leq \|u\|_\infty \sum_{0\leq \bfk\leq n} \big|((\Delta^*_\ell)^r \, p_{n,\cdot}(\bfx))_{\bfk} \big|
	\leq C_{\mu,r} \sum_{k_\ell=0}^n \Big|\big((\Delta^*)^r p_{n,\cdot}(x_\ell)\big)_{k_\ell}\Big|. 
	\end{align*}
	This quantity was upper bounded in \cite[Theorem 6]{onebitBernstein}, and by the referenced bound, we have
	$$
	\sum_{k_\ell=0}^n \Big|\big((\Delta^*)^r p_{n,\cdot}(x_\ell)\big)_{k_\ell}\Big| \lesssim_{r} n^{-r/2} x_\ell^{-r}(1-x_\ell)^{-r}. 
	$$
	On the other hand, we have the trivial upper bound for the quantization error, 
	$$
	\Big|\sum_{0\leq \bfk\leq n} (a_{\bfk}-\sigma_{\bfk}) \ p_{n,\bfk}(\bfx)\Big|
	\leq \|a-\sigma\|_\infty \sum_{0\leq \bfk\leq n} p_{n,k}(\bfx)
	= \|\Delta^r_\ell \, u\|_\infty 
	\leq 2^r\|u\|_\infty
	\leq 2^r C_{r,\mu}. 
	$$
\end{proof}

\subsection{Comments on the quantization results}

In contrast to \cref{thm:rsigmadelta} for the Bernstein basis, we explain why $\{\pm1\}$ linear combinations of the power basis $\mathcal{P}_n:=\{x\mapsto x^k\}_{k=0}^n$, by which we mean any function of the form 
$
p(x)=\sum_{k=0}^n \sigma_k x^k
$
where $\sigma_k\in \{\pm 1\}$ for each $k$, cannot accurately approximate continuous functions on $[0,1]$. We provide two different set of explanations. 

All possible real numbers that can be realized as an output of a $\{\pm 1\}$ linear combination in the power basis up to degree $n$ with input $x$ is the set
$$
P_n(x)=\Big\{\sum_{k=0}^n \sigma_k x^k \colon \sigma_k\in \{\pm 1\} \text{ for all } k \Big\}.
$$
A plot of $P_{10}$ is shown in \cref{fig:Px}. It is straightforward to see that 
$$
P_n(x) \subset [1-r_n(x),\, 1+r_n(x)]\cup [-1-r_n(x),\, -1+r_n(x)], \quad\text{where}\quad r_n(x)=x\frac{1-x^n}{1-x}. 
$$
Hence, $P_n(x)$ is a strict subset of $[-2,2]$ whenever $x\in (0,1/2)$ and it becomes an increasingly smaller subset as $x\to 0$. For any $c\in (0, 1/2)$, all sufficiently small $\epsilon>0$, and all $n\geq 1$, we have 
$$
\sup_{f\in C([0,1])} \, \inf_{\sigma_0,\dots,\sigma_n\in \{\pm 1\}} \, \sup_{x\in [0,c]} \, \Big| f(x)-\sum_{k=0}^n \sigma_k x^k \Big| >\epsilon. 
$$
The key component of this statement is that $c$ is fixed independent of $n$. In contrast to \cref{thm:rsigmadelta}, the regions for which the approximation error is large shrinks to zero as $n\to\infty$. For all sufficiently large $n$, this region consists of two $d$-dimensional rectangles  whose measures are $O(1/\sqrt n)$.

\begin{figure}[h]
	\centering
	\includegraphics[width=0.5\textwidth]{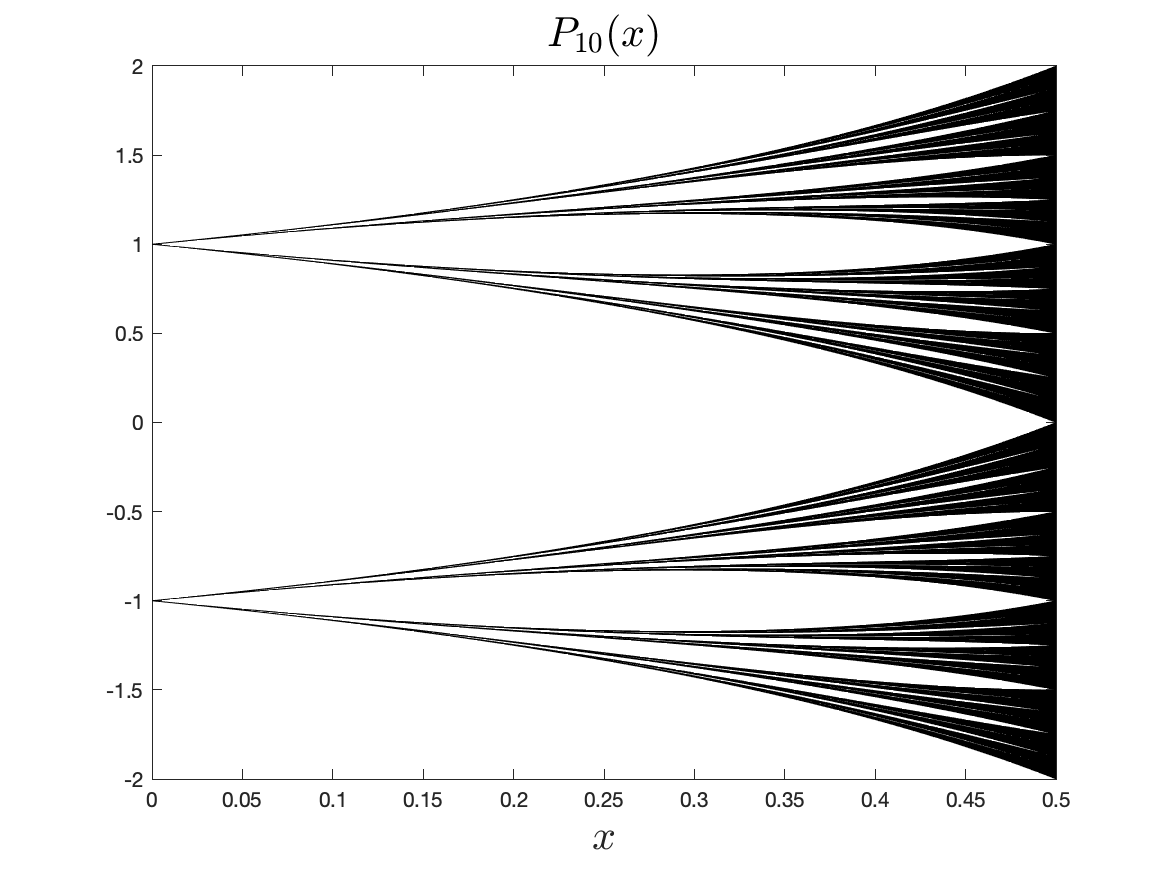}
	\label{fig:Px}
	\caption{Plot of $P_{10}$ on $[0,\frac{1}{2}]$.}
\end{figure} 

Related to this discussion is the observation that $P(x):=\bigcup_{n=0}^\infty P_n(x)$ for $x\in (0,1/2)$ has a fractal structure and is totally disconnected. Indeed, $P(x)$ is the set of all real numbers that can be written in base $x$ with $-1,0,+1$ digits and of length $n$.  for any $a,b\in P(x)$ with $a<b$, we can express them as a sequence in base $x$, and let $m$ be the first digit for which their base $x$ expansions disagree. Call their first $m-1$ digits $\sigma_0,\dots,\sigma_{m-1}$. Any $c\in P(x)$ with $a<c<b$ can be written in base $x$ as $\sigma_0,\dots,\sigma_{m-1},\epsilon_m,\dots$, where $\epsilon_k\in \{-1,0,1\}$. We recall the basic observation that $\sum_{k=m}^\infty \epsilon_k x^k \in (-x^m/(1-x),x^m/(1-x))\subset (-x^{m-1},x^{m-1})$ since $x\in (0,1/2)$. Hence we can always find a $c$ such that $a<c<b$ and $c\not\in P(x)$. 

Of course, the previous discussion only pertains to the power basis' inability to approximate continuous functions near the origin. It was shown in \cite{gunturk2005approximation} that $\{\pm 1\}$ linear combinations in the power basis are able to approximate certain power series in the complex plane near the point $z=1$, so approximation is possible in other regions.  

From a quantization perspective, \cref{thm:rsigmadelta} is perhaps surprising because previous applications of noise-shaping quantization (including $\Sigma\Delta$) \cite{DD,exp_decay,gunturk2005approximation,noiseshaping,gunturk2022quantization}, utilize some notion of redundancy in the system. However, the multivariate Bernstein polynomials of order $n$ forms a basis for the space of multivariate polynomials whose degree in each variable is at most $n$, so it does not exhibit redundancy in the traditional sense.

Instead, the Bernstein system exhibits a different type of redundancy. To make this notion more precise, we define the synthesis operator $S_n\colon \R^{n+1}\to \calP_n$ by 
$
S_nu := \sum_{k=0}^n u_k p_{n,k}. 
$
Using the usual inner products on both $\R^{n+1}$ and $\calP_n$, a direct calculation shows that the frame operator $S_n^*S_n\colon \R^{n+1}\to \R^{n+1}$ is represented as the matrix $B$ such that $B_{j,k}:=\int_0^1 p_{n,j}(x) p_{n,k}(x)\, dx$. The $\epsilon$ numerical rank of a $(n+1)\times (n+1)$ matrix $A$ by 
$$
d_\epsilon(A):=\max\{0\leq m\leq n \colon \sigma_m(A)\geq \epsilon_0(A)\}.
$$
Then \cite[Appendix B]{onebitBernstein} showed that for fixed $\epsilon>0$, we have
$$
d_\epsilon(B)= \sqrt{2\ln(1/\epsilon)}\sqrt{n}(1+o(1))
\quad\text{as}\quad n\to\infty. 
$$
This partially explains why noise-shaping quantization in the Bernstein basis is possible and why Theorems \ref{thm:approxsmooth} and \ref{thm:rsigmadelta} exhibit decay rates of $(\sqrt n)^{-s}$ and $(\sqrt n)^{-r}$ respectively. 

Continuing this line of discussion and to connect it with the prior comparison between the Bernstein and power basis, notice that the synthesis operator for the power basis is $T_n\colon \R^{n+1}\to \mathcal{P}_n$ where $T_nu=\sum_{k=0}^n u_kx^k$. Then the frame operator $T_n^*T_n$ can be identified with the matrix $H$ which has entries $H_{j,k}:=\int_0^1 x^jx^k\, dx=1/(j+k+1)$. This is precisely the Hilbert matrix, and it follows from \cite[Corollary 4.2]{beckermann2017singular} that
$$
d_{\epsilon}(H)
\leq \Big\lceil \frac{\log(8n-4)\log(4/\epsilon)}{\pi^2} \Big\rceil.
$$
Hence, the numerical rank for the power basis is roughly $\log(n)$ compared to that of $\sqrt n$ for the Bernstein basis.

\subsection{Proof of Theorem \ref{thm:mainbernstein}}
\label{sec:proofthmmain1}

\begin{proof}
	According to \cref{thm:bapprox}, using $(1-\mu)/2$ as $\delta$ in the referenced theorem and the assumption that
	$\displaystyle
	n\geq \frac{\sqrt{2^{s+1}}d}{4(1-\mu)} \|f\|_{C^1\Lip},
	$ if $s\geq 3$, there exist $\{a_{\bfk}\}_{0\leq \bfk\leq n}$ such that $\|a\|_\infty \leq (\mu+1)/2<1$ and  
	$$
	\Big\|f-\sum_{0\leq \bfk \leq n} a_{\bfk} p_{n,\bfk}\Big\|_\infty 
	\lesssim_{s,d} \|f\|_{C^s} n^{-s/2}. 
	$$
	For any $1\leq \ell\leq d$, applying $s$-th order $\Sigma\Delta$ in the $\ell$-th direction on $\{a_{\bfk}\}_{0\leq \bfk\leq n}$, noting that $\|a\|_\infty \leq (\mu+1)/2<1$, we obtain $\{\sigma_{\bfk}\}_{0\leq \bfk\leq n}\subset \{\pm 1\}$. According to \cref{thm:rsigmadelta}, the induced quantization error satisfies for all $\bfx\in [0,1]^d$, 
	$$
	\Big|\sum_{0\leq \bfk \leq n} a_{\bfk} p_{n,\bfk}(\bfx) -\sum_{0\leq \bfk \leq n} \sigma_{\bfk} p_{n,\bfk}(\bfx) \Big|
	\lesssim_{s,\mu} \min\big(1,n^{-s/2}x_\ell^{-s} (1-x_\ell)^{-s} \big). 
	$$
	Combining these inequalities completes the proof.
\end{proof}

\section{Implementation error}

\label{sec:Bimplement}

In this section, we concentrate on the implementation error by one-bit quantized neural networks. That is, suppose for some $\{\sigma_{\bfk}\}_{0\leq \bfk\leq n}$ such that $\sigma_{\bfk}\in \calA$, we would like find functions $\{b_{n,\bfk}\}_{0\leq \bfk\leq n}$ such that $\sum_{0\leq \bfk\leq n} \sigma_{\bfk} b_{n,\bfk}$ is implementable by a strict $\calA$-quantized neural network with activation $\beta$ for which the implementation error 
$$
\calE^{\text{imp}}_{n,\sigma,\calA,\beta}(\bfx):= \sum_{0\leq \bfk\leq n} \sigma_{\bfk} p_{n,\bfk}(\bfx)- \sum_{0\leq \bfk\leq n} \sigma_{\bfk} b_{n,\bfk}(\bfx),
$$
is suitably controlled. We address two pairs of activation functions and one-bit alphabets. 
\begin{enumerate}[(a)]
	\item 
	The first pair is the quadratic activation $\rho\colon \R\to\R$ defined as $\rho(t)=\frac{1}{2} t^2$, together with the one-bit alphabet $\calA_1 := \{\pm 1\}$. 
	
	%\cref{thm:bernquad} shows that the sum $\sum_{0\leq \bfk\leq n} \sigma_{\bfk} b_{n,\bfk}(\bfx)$ can be implemented by a $\calA_1$-quantized quadratic neural network. Hence, there is no implementation error in this case.
	
	\item 
	The second pair is the ReLU activation $\sigma\colon\R\to\R$ defined as $\sigma(t)=\max(t,0)$, together with the one-bit alphabet $\calA_{1/2}:=\{\pm \frac 1 2\}$. 
	
	%\cref{thm:bernrelu} constructs functions $\{b_{n,\bfk}\}_{0\leq \bfk\leq n}$ that are implementable by a $\calA_{1/2}$-quantized ReLU neural network and closely approximate the Bernstein polynomials.  
\end{enumerate}

For both activation functions, the main mechanism behind our construction hinges on a Pascal triangle interpretation of the Bernstein polynomials, which we proceed to explain in the univariate case first. Each Bernstein polynomial of degree $m+1$ can be made by multiplying at most two pairs of Bernstein polynomials, with degrees $m$ and $1$, and adding them together. The key formulas are, for each $m\geq 1$,
\begin{equation}
	\label{eq:bernrecurrence}
	p_{m+1,k}(x)=
	\begin{cases}
		\ (1-x)p_{m,0}(x) &\quad\quad \text{if } k=0,\\
		\ x p_{m,k-1}(x)+(1-x)p_{m,k}(x) &\quad\quad \text{if } 0<k<m, \\
		\ x p_{m,m}(x) &\quad\quad  \text{if } k=m+1.
	\end{cases}
\end{equation} 
These recurrence relations are summarized in a Pascal-like network as shown in Figure \ref{fig:Bpascal}. Due to the combinatorial factors that appear in the Bernstein polynomials, the Pascal triangle interpretation is crucial when attempting to implement approximations of the Bernstein polynomials in a stable way when using only parameters from a small set.  

Definitions of strict (unquantized and quantized) neural networks were provided in Definitions \ref{def:strictNN} and \ref{def:quanNN}. As discussed earlier, strict neural networks do not use skip connections and all intermediate layers use the same activation. Here, we introduce additional definitions and some basic concepts that will help facilitate the subsequent proofs. 

One cumbersome issue is that if $F\colon \R^d\to\R^m$ and $G\colon \R^m\to \R^n$ are both implementable by strict neural networks, then their composition $G\circ F$ is not necessarily implementable because by using the output of $F$ as the input of $G$ results in network with a layer without activation. For this reason, we introduce the following definition. 

\begin{definition}
	\label{def:activatedNN}
	An activated neural network with activation $\beta$ is any function $F\colon\R^d\to\R^m$ of the form, 
	\begin{equation*}
		F(\bfx):= \beta(W_L \beta(W_{L-1} \cdots \beta(W_1(\bfx)))), \quad W_\ell(\bfu):=A_\ell \bfu + \bfb_\ell,	\quad \text{for} \quad  \ell=1,\dots,L. 	
	\end{equation*}	
\end{definition}

Hence, an activated neural network also does not have skip connections, but each layer, including the output layer, uses the same activation function $\beta$. We say the network is unquantized if the weights and biases are allowed to use any real number, while it is $\calA$-quantized if they are selected from $\calA$ only. The number of layers, nodes, and parameters of an activated neural network are defined in the same way as for strict neural networks. With this terminology in place, we make the following basic observations. 
\begin{enumerate}[(a)]
	\item 
	If $F\colon \R^d\to\R^m$ is implementable by an activated network of size $(L_1,N_1,P_1)$ and $G\colon \R^m\to\R^n$ is implementable by a strict (resp., activated) neural network of size $(L_2,N_2,P_2)$, then $G\circ F$ is implementable by a strict (resp., activated) network of size $(L_1+L_2,N_1+N_2,P_1+P_2)$. Naturally, we say that the second network is appended to the first or that the networks are composed. 
	\item
	If $F\colon \R^d\to\R^m$ and $G\colon \R^d\to\R^n$ are both implementable by  strict (resp., activated) networks of size $(L,N_1,P_1)$ and $(L,N_2,P_2)$, then the function $F\oplus G\colon \R^d\to \R^{m+n}$ is implementable by a strict (resp., activated) network of size $(L,N_1+N_2,P_1+P_2)$. We say that these networks are placed in parallel. 
\end{enumerate}

\subsection{Overview for $\{\pm 1\}$-quantized quadratic networks}
\label{sec:quadratic}

In this subsection we provide a high level discussion of our constructions for $\calA_1$-quantized neural networks with quadratic activation $\rho$, and postpone the details for \cref{sec:quadappendix}. In the expository portions, we just use the term ``network" since the alphabet and activation do not change in this subsection. 

To turn the schematic diagrams shown in Figures \ref{fig:Bnetwork} and \ref{fig:Bpascal} into a proper network, it suffices to convert multiplications by $1-x$ and $x$ into neural network operations. Our methodology is inspired by the basic identity, 
\begin{equation}
	\label{eq:multbinaryquad}
	ab = \rho(a+b) - \rho(a) - \rho(b).
\end{equation}
It is important to remark that this identity cannot be (directly) used for intermediate layers, since strict neural networks require every node to be activated and hence such a node can only produce $\rho(ab)$, not $ab$ itself. Hence, we cannot realize the product function as an activated neural network.

This technical issue can be circumvented under many situations in light of the following observation. If an intermediate layer's nodes output $\rho(a+b)$, $\rho(a)$, $\rho(b)$, and $c$, then a subsequent intermediate layer can implement $\rho(ab+c)$ in view the identity
$$
\rho(ab+c)
=\rho(\rho(a+b) - \rho(a) - \rho(b)+c). 
$$
Hence, although $ab$ cannot be directly implemented in an intermediate layer, it can be used as an input into a subsequent intermediate or final layer. This type of reasoning can be adapted to more general situations beyond implementing $\rho(ab+c)$. 

To see how this observation is relevant to Bernstein, we define the functions 
\begin{align*}
	U(x)&:= \rho(1-x), \quad 
	V(x):=\rho(x), \quad 
	X_{0,0}(x):= \rho(1), \\
	Y_{0,0}(x)&:= \rho(1-x+1), \quad  Z_{0,0}(x):=\rho(x+1). 
\end{align*}
For each integer $m\geq 1$ and $k=0,\dots,m$, we define
\begin{equation*}
	\label{eq:XYZ}
	X_{m,k}(x) = \rho(p_{m,k}(x)), \quad 
	Y_{m,k}(x) = \rho(1-x+p_{m,k}(x)), \quad 
	Z_{m,k}(x) = \rho(x+p_{m,k}(x)).
\end{equation*} 
It follows from the recurrence relation \eqref{eq:bernrecurrence} and identity \eqref{eq:multbinaryquad} that 
\begin{equation}
	\label{eq:bernrecurrence2}
	p_{m,k} = 
	\begin{cases}
		\ Y_{m-1,0} - U - X_{m-1,0} &\quad\text{if } k=0,\\
		\ Z_{m-1,k-1} - V - X_{m-1,k-1} + Y_{m-1,k} - U - X_{m-1,k} &\quad\text{if }  0 < k < m, \\
		\ Z_{m-1,m-1} - V - X_{m-1,m-1} &\quad\text{if } k=m.
	\end{cases}
\end{equation}
These formulas show that each $U,V,X,Y,Z$ can be realized as an intermediate layer's output by composing networks, e.g., $X_{m+1,k}$ can be produced given $X_{m,j},Y_{m,j},Z_{m,j},U,V$ for each $0\leq j\leq m$. While each Bernstein polynomial does not actually correspond to the output of a network used in our final construction, it will be a node's preactivation. This implies that any linear combination 
$
\sum_{k=0}^n \sigma_k p_{n,k} 
$
with $\sigma_k\in \{\pm 1\}$ for each $k$, is implementable by a network. 

For the multivariate case, in view of \eqref{eq:bernrecurrence2}, each $p_{n,\bfk}$ is a sum of tensor products consisting of various combinations from $U,V,X,Y,Z$. To implement tensor products, we need multiplication. Similar to the bivariate case \eqref{eq:multbinaryquad}, a multivariate product cannot be realized as an activated network, and it will instead be made implicitly in a subsequently layer. For example, to use the product $abc$ as an argument in a future intermediate layer, if one layer outputs  
$$
\rho(a), \, \rho(a+b), \, \rho(b), \, \rho(c), \, \rho(c+1), \, \rho(1),
$$ then it possible for the next layer to output 
$
\rho(ab+c), \, \rho(ab), \, \rho(c),
$
due to the identities 
\begin{align*}
	\rho(ab+c)&=\rho\big( \rho(a+b)-\rho(a)-\rho(b) + \rho(c+1)-\rho(c)-\rho(1)\big), \\
	\rho(ab)&= \rho(a+b)-\rho(a)-\rho(b), \\
	\rho(c) &= \rho(c+1)-\rho(c)-\rho(1). 
\end{align*}
The significance here is that a $\{\pm 1\}$ linear combination of $\rho(ab+c), \, \rho(ab), \, \rho(c),$ is $abc$ which can be implicitly formed in the layer afterwards. By continuing this process, it is possible to create the terms necessary to produce multivariate multiplication. After these considerations, it is possible to show that for any $\{\sigma_{\bfk}\}_{0\leq \bfk\leq n}$ with $\sigma_{\bfk}\in \{\pm 1\}$, the sum $\sum_{0\leq \bfk\leq n} \sigma_{\bfk} p_{n,\bfk}$ is implementable.

\subsection{Overview for $\{\pm \frac{1}{2}\}$-quantized ReLU networks}
\label{sec:relu}

In this subsection we provide a high level discussion of our constructions for $\calA_{1/2}$-quantized neural networks with ReLU activation $\sigma$, and postpone the details for \cref{sec:reluappendix}. In the expository portions, we just use the term ``network" since the alphabet and activation do not change in this subsection. 

Since ReLU is the conventional activation used in practice, the approximation properties of unquantized ReLU networks have been thoroughly studied. We use some standard ideas popularized by \cite{yarotsky2017error}. The starting point is the tent function $\phi\colon [0,1]\to [0,1]$, where $\phi(x):=2x$ for $x\in [0,\frac 1 2]$ and $\phi(x):=2-2x$ for $x\in [\frac 1 2,1]$, and the identity,
\begin{equation}
	\label{eq:blancmange}
	x(1-x)
	=\sum_{k=1}^\infty \frac{\phi^{\circ k}(x)}{4^k},
\end{equation}
where the series converges uniformly in $x$ and $\phi^{\circ k}$ refers to the composition of $\phi$ with itself $k$-times, with the convention that $\phi^{\circ 1}:=\phi$. Since $\phi$ can be implemented, this identity provides a natural method for implementation of the squaring function and hence multiplication via the basic identity 
\begin{equation}
	\label{eq:reluprod}
	ab = \frac{1}{2}(a+b)^2-\frac{1}{2}a^2-\frac{1}{2}b^2. 
\end{equation}

Let us make some comments about how this strategy can be modified in order to account for the lack of skip connections and one-bit quantization. For technical reasons, we observed that approximation of the squaring function is insufficient as it leads to requiring additional bits. We instead approximate $x^2$ from both above and below by nonnegative $S^+_\epsilon$ and $S^-_\epsilon$ respectively, with error $\epsilon$ uniformly in $x$. The implementations of $S_\epsilon^+$ and $S_\epsilon^-$ are also further complicated by the constraint of not being able to use skip connections. Although we use ``duplication" networks to pass values down a specified number of layers, exaggerated use results in bloated networks and care is taken to use them sparingly. By using $S_\epsilon^{\pm}$ and mimicking \eqref{eq:reluprod}, we define the approximate multiplication function
$$
P_\epsilon(x,y) 
=\sigma\( 2 S_{\epsilon/6}^-\( \frac{x+y}{2}\)- 2S_{\epsilon/6}^+\( \frac{x}{2}\)- 2S_{\epsilon/6}^+\( \frac{y}{2}\)\). 
$$
Not only does $P_\epsilon$ approximate the product function uniformly with error $\epsilon$, it also satisfies the inequalities $0\leq P_\epsilon(x,y)\leq xy$. This is important since we use strict neural networks and the ReLU function is the identity on nonnegative real numbers. 

With approximate bivariate multiplication at hand, approximations of univariate Bernstein polynomials can be constructed mimicking recurrence \eqref{eq:bernrecurrence}. That is, we define $\{b_{m,k}\}_{0\leq k\leq n}$ recursively starting with $b_{1,0}(x)=1-x$ and $b_{1,1}(x)=x$ and for $m\geq 1$, 
\begin{equation*}
	b_{m+1,k}(x)=
	\begin{cases}
		\ P_{\epsilon}(1-x, b_{m,0}(x)) &\quad\quad \text{if } k=0,\\
		\ P_{\epsilon} (x, b_{m,k-1}(x)) + P_{\epsilon}(1-x, b_{m,k}(x)) &\quad\quad \text{if } 0<k<m, \\
		\ P_{\epsilon}(x, b_{m,m}(x)) &\quad\quad  \text{if } k=m+1.
	\end{cases}
\end{equation*} 
These can be approximately multiplied together to yield a collection $\{b_{n,\bfk}\}_{0\leq \bfk\leq n}$ that approximate the multivariate Bernstein polynomials $\{p_{n,\bfk}\}_{0\leq \bfk\leq n}$ uniformly on $[0,1]^d$. Finally, we approximate any $f=\sum_{0\leq \bfk\leq n} \sigma_{\bfk}p_{n,\bfk}$ by $\sum_{0\leq \bfk\leq n} \sigma_{n,\bfk} b_{n,\bfk}$, which is implementable under the assumption that $\sigma_{\bfk}\in \{\pm \frac 1 2\}$.

\subsection{Proof of Theorem \ref{thm:mainimplementation}}

\label{sec:proofthmmain2}

\begin{proof}
		\underline{Quadratic case and $\{\pm 1\}$ alphabet}. From identity \eqref{eq:bernrecurrence2}, we see that each univariate Bernstein polynomial $p_{n,k}$ is a $\pm1$ linear combination of at most six functions in the set 
		$$
		\{U, V, X_{n-1,k}, Y_{n-1,k}, Z_{n-1,k}\colon k=0,\dots,n-1\}.
		$$
		To simplify the notation, denote this set of $3n+2$ functions by $\{\psi_j\}_{j=1}^{3n+2}$. Then for each $0\leq k\leq n$, there exists $\{\epsilon_{k,j}\}_{j=1}^{3n+2}$ such that $\epsilon_{k,j}\in\{0,1\}$ with at most six that are nonzero, and
		$
		p_{n,k}=\sum_{j=1}^{3n+2} \epsilon_{k,j} \psi_j. 
		$
		
		To handle the multivariate case, we set $\psi_{\bfj}(\bfx):=\psi_{j_1}(x_1)\cdots\psi_{j_d}(x_d)$ and $\epsilon_{\bfk,\bfj}:=\epsilon_{k_1,j_1}\cdots \epsilon_{k_d,j_d}$. Then we have 
		\begin{equation}
			\label{eq:psihelp}
			\sum_{0\leq \bfk \leq n} \sigma_{\bfk} p_{n,\bfk}
			=\sum_{0\leq \bfk \leq n} \sum_{1\leq \bfj\leq 3n+2} \sigma_{\bfk} \epsilon_{\bfk,\bfj}\psi_{\bfj}
			=\sum_{1\leq \bfj\leq 3n+2}  \Big( \sum_{0\leq \bfk \leq n} \sigma_{\bfk} \epsilon_{\bfk,\bfj} \Big) \psi_{\bfj}.	
		\end{equation}
		It is important to remark that this is not necessarily a $\pm 1$ combination of $\psi_{\bfj}$'s because it is possible for a $\psi_{\bfj}$ to repeat. In fact, there will be many repetitions as $U(x_1)U(x_2)\cdots U(x_d)$ appears $(n-1)^d$ times. However, it is possible to express the right side of \eqref{eq:psihelp} as a $\pm 1$ sum of $\psi_{\bfj}$'s if we allow for repetitions, and that this $\pm 1$ sum only requires at most $6^d(n+1)^d$ terms, in view of the observation that 
		$
		|\{\bfj \colon \epsilon_{\bfk,\bfj}\not=0\}|\leq 6^d 
		$
		uniformly in $\bfk$. Hence, there is a finite sequence $I$ with $|I|\leq 6^d(n+1)^d$ and $\{\tilde\sigma_{\bfj}\}_{\bfj\in I}$ with $\tilde\sigma_{\bfj}\in \{\pm 1\}$ such that 
		\begin{equation*}
			\sum_{0\leq \bfk \leq n} \sigma_{\bfk} p_{n,\bfk}(\bfx)
			=\sum_{\bfj \in I} \tilde \sigma_{\bfj} \psi_{\bfj}(\bfx).
		\end{equation*}
		
		We are now ready to implement this approximation strategy as a neural network. If $d=1$, notice that 
		\begin{align*}
			\sum_{k=0}^n \sigma_k p_{n,k} 
			&= \sigma_0 Y_{n-1,0} - \sigma_0 U - \sigma_0 X_{n-1,0} \\ 
			&\quad + \,\sum_{k=1}^{n-1} (\sigma_k Z_{n-1,k-1} - \sigma_k V - \sigma_k X_{n-1,k-1} + \sigma_k Y_{n-1,k} - \sigma_k U - \sigma_k X_{n-1,k}) \\
			&\quad  + \,\sigma_{n-1} Z_{n-1,n-1} - \sigma_{n-1} U - \sigma_{n-1} X_{n-1,n-1}.	
			\label{output-layer}
		\end{align*}
		Notice that each $X_{n-1,k}$ appears twice. We use two networks described in \cref{prop:XYZ} placed parallel with each other, so that they outputs all terms on the right hand side, and a linear layer produces the final summation. This network has size $O(n)$.  
				
		For $d\geq 2$, we use \cref{prop:XYZ} to generate $d$ networks in parallel, so that layer $n+1$ produces, for each $1\leq \ell\leq d$, the outputs
		\begin{equation}
			\label{eq:XYZl}
			U(x_\ell), \, V(x_\ell), \, \{X_{n-1,k}(x_\ell)\}_{k=0}^{n-1}, \, \{Y_{n-1,k}(x_\ell)\}_{k=0}^{n-1}, \, \{Z_{n-1,k}(x_\ell)\}_{k=0}^{n-1}.
		\end{equation}
		Doing so requires a network of size $O(n)$. Let us momentarily fix a $\bfj\in I$. Since $\psi_{\bfj}(\bfx)=\psi_{j_1}(x_1)\cdots \psi_{j_d}(x_d)$ and each $\psi_{j_\ell}(x_\ell)$ is a $\pm 1$ summation of terms in \eqref{eq:XYZl} due to identity \eqref{eq:bernrecurrence2}, we use the network constructed in \cref{prop:multseveralquadratic} that outputs the quantities
		$$
		\rho\( \prod_{\ell\leq d_*} \psi_{\bfj}(x_\ell) + \prod_{\ell >d_*} \psi_{\bfj}(x_\ell)\) , \, \rho\( \prod_{\ell \leq d_*} \psi_{\bfj}(x_\ell)\), \, \rho\( \prod_{\ell>d_*} \psi_{\bfj}(x_\ell) \).
		$$
		We do this for each $\bfj\in I$ and place these networks in parallel, hence further requiring a network with $O(1)$ layers and $O(n^d)$ nodes and parameters. The final linear layer produces 
		$
		\sum_{\bfj \in I} \tilde \sigma_{\bfj} \psi_{\bfj}(\bfx)
		$
		since $\tilde \sigma_{\bfj}\in \{\pm 1\}$ and $\psi_{\bfj}(\bfx)$ is a $\pm 1$ linear combination of the above terms.

		\medskip 
		
		\underline{ReLU case and $\{\pm \frac{1}{2}\}$ alphabet}. Let $\epsilon>0$. As shown in \cref{prop:bernrelu}, there is an activated $\{\pm \frac{1}{2}\}$-quantized ReLU neural network that implements a set of functions $\{b_{n,\bfk}\}_{0\leq \bfk\leq n}$ such that $\|p_{n,\bfk}-b_{n,\bfk}\|_\infty \leq \epsilon$ for each $0\leq \bfk\leq n$. This network has $O(n\log (n/\epsilon))$ layers and $O( n^2\log (n/\epsilon) + n^d\log(1/\epsilon))$ nodes and parameters, as $n\to\infty$ and $\epsilon\to 0$. We use two copies of these networks placed in parallel, so that the function 
		$$
		f_{NN,\sigma}
		:=\sum_{0\leq \bfk \leq n} \sigma_{\bfk} b_{n,\bfk}
		=\sum_{0\leq \bfk \leq n} \frac{\sigma_{\bfk}}{2} b_{n,\bfk}+\sum_{0\leq \bfk \leq n} \frac{\sigma_{\bfk}}{2} b_{n,\bfk}
		$$
	 	is still implementable by a $\{\pm \frac{1}{2}\}$-quantized neural network after placing a linear layer with weights given by two copies of $\{\sigma_{\bfk}/2\}_{0\leq \bfk\leq n}$. This last layer has size $(1,1,2(n+1)^d)$. In total, $f_{NN,\sigma}$ is implementable by a $\{\pm \frac{1}{2}\}$-quantized neural network with $O(n\log (n/\epsilon))$ layers and $O( n^2\log (n/\epsilon) + n^d\log(1/\epsilon))$ nodes and parameters. Since $|\sigma_{\bfk}|=1$, we have 
		$$
		\Big\| \sum_{0\leq \bfk \leq n} \sigma_{\bfk} p_{n,\bfk}-\sum_{0\leq \bfk \leq n} \sigma_{\bfk} b_{n,\bfk}\Big\|_\infty
		\leq \sum_{0\leq \bfk \leq n} \|p_{n,\bfk}-b_{n,\bfk}\|_\infty
		\leq (n+1)^d\epsilon.
		$$
	
\end{proof}

\subsection{Proof of Theorem \ref{thm:main}}
\label{sec:proofthmmain3}

\begin{proof}
	\underline{Quadratic case and $\{\pm 1\}$ alphabet}. It follows from \cref{thm:mainbernstein} that for any $f\in C^s([0,1]^d)$ with $\|f\|_\infty \leq \mu$, there exist $\{\sigma_{\bfk}\}_{0\leq \bfk\leq n}\subset \{\pm 1\}$ such that for all $\bfx\in [0,1]^d$, 
	$$
	\Big|f(\bfx) -\sum_{0\leq \bfk \leq n} \sigma_{\bfk} p_{n,\bfk}(\bfx) \Big|
	\lesssim_{s,\mu} \|f\|_{C^s} \min\big(1,n^{-s/2}x_\ell^{-s} (1-x_\ell)^{-s} \big). 
	$$
	Using \cref{thm:mainimplementation} for the quadratic case completes the proof.
	
	\underline{ReLU case and $\{\pm \frac{1}{2}\}$ alphabet}. We apply \cref{thm:mainbernstein} to obtain there exist $\{\sigma_{\bfk}\}_{0\leq \bfk\leq n}\subset \{\pm 1\}$ such that for all $\bfx\in [0,1]^d$, 
	$$
	\Big|f(\bfx) -\sum_{0\leq \bfk \leq n} \sigma_{\bfk} p_{n,\bfk}(\bfx) \Big|
	\lesssim_{s,\mu} \|f\|_{C^s} \min\big(1,n^{-s/2}x_\ell^{-s} (1-x_\ell)^{-s} \big). 
	$$
	Using \cref{thm:mainimplementation} for the ReLU case with $\epsilon = \|f\|_{C^s} n^dn^{-s/2}$ completes the proof.		  
\end{proof}

\section{Final Remarks}
\label{sec:finalremarks}

\subsection{Algorithm for computing the one-bit coefficients and stability to noise}

The binary sequence $\{\sigma_{\bfk}\}_{0\leq \bfk\leq n}$ that appears in \cref{thm:mainbernstein}, which is also used in the neural network constructions in \cref{thm:main}, can be numerically computed from samples of $f$ on the lattice $\{\bfk/n\}_{0\leq \bfk\leq n}$ without any other additional information about $f$, and is summarized in \cref{alg:binarybernalg}. 

\begin{algorithm}[h]
	\begin{algorithmic}
		\Require smoothness of the target function $s$, direction $\ell$ for $\Sigma\Delta$, samples $\{f(\frac{\bfk}{n})\}_{0\leq \bfk\leq n}$.
		
		\State 1. Calculate $\{a_{\bfk}\}_{0\leq \bfk\leq n}$ defined to be 
		$
		a_{\bfk}:=f_{n,\lceil s/2 \rceil }(\frac{\bfk}{n}).
		$
		\State2. Abort if $\|a\|_\infty\geq 1$. 
		\State3. Apply $s$-th order $\Sigma\Delta$ in direction $\ell$ on $\{a_{\bfk}\}_{0\leq \bfk\leq n}$. 
		\Ensure One-bit coefficients $\{\sigma_{\bfk}\}_{0\leq \bfk\leq n}$ and approximant $\sum_{0\leq \bfk\leq n} \sigma_{\bfk} p_{n,\bfk}(\bfx)$.			
	\end{algorithmic}
	\caption{Binary Bernstein algorithm}
	\label{alg:binarybernalg}
\end{algorithm}

Tracing through the proof of \cref{thm:mainbernstein}, the first step of our approximation scheme is to compute the real coefficients
$$
a_{\bfk}:=f_{n,\lceil s/2 \rceil}\(\frac{\bfk}{n} \)
:=\( \sum_{m=0}^{\lceil s/2\rceil-1} (I-B_n)^m(f)\)\(\frac{\bfk}{n} \).
$$

We explain how to compute these from the samples $\{f(\bfk/n)\}_{0\leq \bfk\leq n}$. Fix any bijection from $\{\bfk\colon 0\leq \bfk\leq n\}$ to $\{1,\dots,(n+1)^d\}$, such as the lexicographic ordering. Let $\bff\in \R^{(n+1)^d}$ be the vector such that $\bff_{\bfk}=f(\bfk/n)$, where the subscript on $\bff$ should be understood as the image of $\bfk$ under whichever ordering was selected. Let $P$ be the $(n+1)^d\times (n+1)^d$ matrix whose $(\bfj,\bfk)$ entry is $p_{n,\bfk}(\bfj/n)$. It follows from a direct calculation that 
$
B_n(f)({\bfk}/{n})
%=\sum_{0\leq \bfk\leq n} f\(\frac{\bfk}{n}\)p_{n,\bfk}\(\frac{\bfj}{n}\)
=(P\bff)_{\bfk}. 
$
Furthermore, it holds that for each $m\geq 1$, 
$$
B_n^m(f)\(\frac{\bfk}{n}\)
=(P^m \bff)_{\bfk}.
$$
This can be shown by induction since by linearity of the Bernstein operator, 
$$
B_n^{m+1}(f)\(\frac{\bfk}{n}\)
=\sum_{0\leq\bfl\leq n} B_n^m(f)\(\frac{\bfl}{n}\) p_{n,\bfl}\(\frac{\bfk}{n}\)
=\sum_{0\leq\bfl\leq n} (P^m\bff)_{\bfl} P_{\bfk,\bfl}
=(P^{m+1} \bff)_{\bfk}. 
$$
It follows from the above that
$$
a_{\bfk}
=f_{n,\lceil s/2 \rceil}\(\frac{\bfk}{n}\)
%=\( \sum_{k=0}^{r-1} (I-B_n)^k(f)\)\(\frac{\bfj}{n}\)
=\(\sum_{m=0}^{\lceil s/2 \rceil-1} (I-P)^m \bff\)_{\bfk}. 
$$ 
Hence computation of $\{a_{\bfk}\}_{0\leq\bfk\leq n}$ amounts to matrix vector operations. 

It follows from \cref{thm:iteratetolinear} that for sufficiently large $n$, we can guarantee that $\|a\|_\infty<1$. From here, calculation of $\{\sigma_{\bfk}\}_{0\leq \bfk\leq n}$ just requires feeding them into a stable $s$-th order $\Sigma\Delta$ quantization scheme applied to the $\ell$-th direction, for any $\ell$ chosen beforehand. Since directional $\Sigma\Delta$ follows directly from its corresponding one-dimensional version, we drop the dependence on $\ell$ in this expository portion. 

For the reader's convenience, let us fully describe the scheme introduced in \cite{exp_decay}. Instead of solving \eqref{sigmadeltaeq} directly, one considers the equation (where $y,q,h,v$ are sequences indexed by $\Z$),
\begin{equation}
	\label{eq:sigmadelta2}
	y_k-q_k = v_k-(h*v)_k, \quad\text{where}\quad (h*v)_k = \sum_{j=1}^k h_j v_{k-j}.
\end{equation}
Let $\mu\in (0,1)$, which will serve as an upper bound for $\|y\|_\infty$, and fix an integer $r\geq 1$. Pick any natural number $\gamma>6$ such that $\mu\leq 2-\cosh(\pi\gamma^{-1/2})$. Define the integers $z_k:=\gamma(k-1)^2+1$ for $k=1,\dots,r$, and let $h$ be a sequence supported in $\{z_1,\dots,z_r\}$ with $h_{z_k}=d_k$, where $\{d_1,\dots,d_r\}$ are found as solutions to the Vandermonde system,
$$
\begin{bmatrix}
	1 &1 &\cdots &1 \\
	z_1 &z_2 &\cdots & z_r \\
	\vdots &\vdots & &\vdots \\
	z_{1}^{r-1} &z_2^{r-1} &\cdots &z_r^{r-1} 
\end{bmatrix}
\begin{bmatrix}
	d_1 \\ d_2 \\ \vdots \\ d_r
\end{bmatrix}
=\begin{bmatrix}
	1 \\ 0 \\ \vdots \\ 0
\end{bmatrix}.
$$
Hence $h$ is readily computed by solving a linear system. Given input $y$, we compute the $\{\pm1\}$ sequence $q$ recursively by
$$
q_k:=\sign((h*v)_k+y_k), \andspace v_k:= y_k-q_k+(h*v)_k.
$$
Here, we use the convention that $\sign(0)=1$. As for computations, this is enough since we are only interested in computing $q$, which serves as the $\{\sigma_{\bfk}\}_{k_\ell=0}^n$ for fixed $\bfk'=(k_1,\dots,k_{\ell-1},k_{\ell+1},k_d)$. 

For theoretical purposes, in order to control the quantization error, it was shown in the referenced paper that there exists an auxiliary sequence $g$ for which $v_k-(h*v)_k=\Delta^r(g*v)$, so \eqref{eq:sigmadelta2} is in fact a $\Sigma\Delta$ scheme as in \eqref{sigmadeltaeq} with $u=g*v$ instead. With this transformation in place, it was shown that 
$$
\|u\|_\infty 
\leq \frac{3}{\sqrt{2\pi r}} (\gamma e)^r r^r.
$$
Hence, the right hand side term serves as the implicit constant $C_{\mu,r}$ that controls the stability of this particular quantization scheme. 

Next, we examine the stability of our approximation scheme to perturbations of the input function $f$. Since the approximation method only depends on the samples $\{f(\bfk/n)\}_{0\leq \bfk\leq n}$, we consider noisy samples of the form, 
$$
\tilde y_{\bfk} = f\(\frac{\bfk}{n}\)+\eta_{\bfk},
$$
where $\{\eta_{\bfk}\}_{0\leq\bfk\leq n}$ represents any unknown perturbation. We measure the noise level in via the $\ell^\infty$ norm $\|\eta\|_\infty$. The following theorem  evaluates the resulting approximation error if we use the perturbed samples of $f$ in \cref{alg:binarybernalg}.

\begin{theorem}
	\label{thm:approxnoise}
	Let $s,d,n\geq 1$, $\epsilon>0$, and $f\in C^{s}([0,1]^d)$ such that 
	\begin{equation}
		\label{eq:beta}
		\beta:=\|f\|_\infty + \sqrt{2^{s+1}}\epsilon + \frac{\sqrt{2^{s-1}}d}{8n}\|f\|_{C^1\Lip} <1. 
	\end{equation}
	For any $1\leq \ell\leq d$ and $\{\eta_{\bfk}\}_{0\leq \bfk\leq n}$ such that $\|\eta\|_\infty\leq \epsilon$, let $\{\tilde \sigma_{\bfk}\}_{0\leq \bfk\leq n}\subset \{\pm 1\}$ be the output of \cref{alg:binarybernalg} given inputs $\tilde y_{\bfk}=f(\bfk/n)+\eta_{\bfk}$ for each $0\leq\bfk\leq n$. Then we have 
	$$
	\Big| f(x)-\sum_{0\leq\bfk\leq n} \tilde\sigma_{\bfk} p_{n,\bfk}(x)\Big|
	\lesssim_{s,d,\beta} \|f\|_{C^s}  \min\big(1,n^{-s/2}x_\ell^{-s} (1-x_\ell)^{-s} \big)+\epsilon.
	$$
\end{theorem}

\begin{proof}
	Let $\{\tilde a_{\bfk}\}_{0\leq\bfk\leq n}$ be the real coefficients produced by the first step of \cref{alg:binarybernalg} given input $\{\tilde y_{\bfk}\}_{0\leq \bfk\leq n}$. It will be helpful to produce a function $\tilde f$ for which $\tilde f(\bfk/n)=y_{\bfk}$. Let $\varphi$ be any $C^\infty$ function compactly supported in $[-\frac{1}{4},\frac{1}{4}]^d$ such that $\varphi(0)=1$ and $0\leq \varphi\leq 1$. Define 
	$$
	\tilde f(x) := f(x)+\sum_{0\leq \bfk\leq n} \eta_{\bfk} \varphi(nx - \bfk),
	$$
	so that $\tilde f\in C^s([0,1]^d)$ and $\tilde f(\bfk/n)=\tilde y_{\bfk}$. Note that the set of functions $\{\varphi(n\cdot -\bfk)\}_{0\leq \bfk\leq n}$ have disjoint supports, and consequently,
	$$
	\|\tilde f-f\|_\infty
	=\Big\| \sum_{0\leq \bfk\leq n} \eta_{\bfk} \varphi(n\cdot  - \bfk)\Big\|_\infty 
	\leq \|\eta\|_\infty
	\leq \epsilon. 
	$$
	To simplify the following notation, let $r=\lceil s/2\rceil$. Define the function 
	$$
	\tilde f_{n,r}
	=\sum_{m=0}^{r-1} (I-B_n)^m(\tilde f). 
	$$
	By \cref{thm:iteratetolinear}, we have that $B_n(\tilde f_{n,\lceil s/2\rceil})=U_{n,\lceil s/2\rceil}(f)$ and 
	\begin{align*}
		\|\tilde f_{n,r}\|_\infty 
		&\leq \|\tilde f\|_\infty + (2^{r-1}-1)\|\tilde f-B_n(\tilde f)\|_\infty \\
		&\leq \|f\|_\infty + \|\eta\|_\infty +(2^{r-1}-1)\big( \|\tilde f- f\|_\infty +\|f-B_n(f)\|_\infty +\|B_n(f)-B_n(\tilde f)\|_\infty \big) \\
		&\leq \|f\|_\infty + (2^r-1)\|\eta\|_\infty + \frac{(2^{r-1}-1)d}{8n} \|f\|_{C^1\Lip},
	\end{align*} 
	where the final inequality follows from \cref{prop:C2bern}. 
	
	Note that $\tilde a_{\bfk}=\tilde f_{n,r}(\bfk/n)$. Hence, under assumption \eqref{eq:beta}, we see that 
	$
	\|\tilde a\|_\infty 
	\leq \|\tilde f_{n,\lceil s/2\rceil}\|_\infty
	\leq \beta
	<1. 
	$
	This permits us to employ a stable $s$-th order $\Sigma\Delta$ scheme in direction $\ell$ to generate signs $\{\tilde\sigma_{\bfk}\}_{0\leq \bfk\leq n}$ from $\{\tilde a_{\bfk}\}_{0\leq \bfk\leq n}$, which satisfy the difference equation 
	$
	\tilde a-\tilde\sigma=(\Delta_\ell)^r \tilde u,
	$
	for a bounded $\tilde u$. We proceed to compare $\{\tilde a_{\bfk}\}$ with $\{a_{\bfk}\}_{0\leq \bfk\leq n}$, which are coefficients generated by \cref{alg:binarybernalg} given noiseless samples. Note that $a_{\bfk}:=f_{n,r}(\bfk/n)$ and
	$$
	|a_{\bfk}-\tilde a_{\bfk}|
	=\Big|\tilde f_{n,r}\(\frac{\bfk}{n}\) -  f_{n,r}\(\frac{\bfk}{n}\)\Big|
	\leq \Big\|\sum_{m=0}^{r-1} (I-B_n)^m(\tilde f-f) \Big\|_\infty
	\leq 2^r \|\tilde f-f\|_\infty 
	\leq 2^r\epsilon. 
	$$
	Using this bound on $\|a-\tilde a\|_\infty$ and that the Bernstein polynomials are nonnegative and form a partition of unity, we have  
	\begin{align*}
	\Big| f(x)-\sum_{0\leq\bfk\leq n} \tilde\sigma_{\bfk} p_{n,\bfk} \Big|
	&\leq \Big| f(x)-\sum_{0\leq\bfk\leq n} a_{\bfk} p_{n,\bfk} \Big|+\Big| \sum_{0\leq\bfk\leq n} \big((\Delta_\ell)^r \tilde u\big)_{\bfk} \, p_{n,\bfk} \Big|+\Big| \sum_{0\leq\bfk\leq n} (a_{\bfk}-\tilde a_{\bfk}) \, p_{n,\bfk} \Big| \\
	&\leq \Big| f(x)-\sum_{0\leq\bfk\leq n} a_{\bfk} p_{n,\bfk} \Big|+\Big| \sum_{0\leq\bfk\leq n} \big((\Delta_\ell)^r \tilde u\big)_{\bfk} \, p_{n,\bfk} \Big|+2^r\epsilon. 
	\end{align*}
	Notice that the first term is the approximation error by iterated Bernstein operators and can be controlled using \cref{thm:bapprox}. The second term is the quantization error on the Bernstein basis and is upper bounded using \cref{thm:rsigmadelta}. Doing so completes the proof. 
\end{proof}

Several comments about this theorem are in order. First, notice that the error bound in \cref{thm:approxnoise} is that of \cref{thm:mainbernstein} plus a contribution from the noise. This theorem holds for arbitrary (hence deterministic and adversarial) perturbations and perhaps could be improved if additional assumptions are made, such as a statistical model. Moreover, the upper bound does not increase in $n$, which is not obvious since the noise energy $\|\eta\|_2$ may grown in $n$ without additional assumptions on the noise. Second, condition \eqref{eq:beta} ensures that the samples $\{\tilde f_{n,r}(\bfk/n)\}$ have absolute value strictly less than 1, which is only used to guarantee that a $s$-th order directional $\Sigma\Delta$ scheme is stable. Such a condition is not always necessarily since it may be plausible that $\Sigma\Delta$ is stable for a larger class of sequences, especially those generated from uniformly sampling smooth functions which have significant correlation between samples. Third, notice that condition \eqref{eq:beta} becomes easier, not more difficult, to satisfy for fixed $s$ and increasing $n$. Again, it is not obvious that such a behavior is possible since the noise energy $\|\eta\|_2$ increases in $n$.

\subsection{Bits and minimax code length}

In this subsection, we discuss our main results on approximation by quantized neural networks in the context of codes. Any function that can be implemented by a strict quantized neural network can be identified by its parameters and topology. Due to \cref{thm:main}, every smooth function can be encoded with small loss of information, and then approximately reconstructed by mapping back to its associated network. We must first explain what we mean by the number of bits. 

We define the minimax code length in an abstract setting before returning back to neural networks. Let $(X,d_X)$ be a metric space and $\calF\subset X$ a class of functions. Suppose that for any $\epsilon>0$, there is an integer $B_\epsilon\geq 0$ and maps 
$
E\colon \calF\to \{\pm 1\}^{B_\epsilon}$ and $D\colon \{\pm 1\}^{B_\epsilon}\to X$, for which 
$$
\sup_{f\in\calF} \ \ d_X\big(f,D(E(f))\big) \leq \epsilon.
$$
In which case, we call $(E,D)$ as an $\epsilon$-approximate encoder-decoder pair and $B_\epsilon$ the number of bits required for this $\epsilon$-approximate encoder-decoder pair. Usually the quantity of interest is the growth rate of $B_\epsilon$ as $\epsilon\to 0$. To examine the optimality of a given pair, the smallest $B_\epsilon$ for which the above holds is defined to be the minimax code length, 
$$
B_\epsilon^*
:=\min\Big\{B\colon \exists \, (E,D) \text{ for which } \sup_{f\in\calF} \  d\big(f,D(E(f))\big)\leq \epsilon \Big\}. 
$$
Of course $B_\epsilon\geq B^*_\epsilon$, but to determine $B^*_\epsilon$, it is well known that the minimax code length is precisely the smallest natural number that upper bounds the Kolmogorov $\epsilon$-entropy of $\calF$. 

Let us now discuss how this is connected to quantized neural networks, in an abstract setting. Let $\calA$ be finite, and for each $\epsilon>0$, suppose $\calN_\epsilon\subset X$ is a finite set of $\calA$-quantized neural networks, where $\calA$ is assumed to be finite, for which we have the approximation property that: for each $f\in\calF$, there is a $g\in \calN_\epsilon$ such that $d_X(f,g)\leq \epsilon$. Since $\calA$ is finite and any $g\in\calN_\epsilon$ can be identified by its network parameters and topology, there is a $B_\epsilon$ and injective $E\colon N_{\epsilon}\to\{\pm 1\}^{B_\epsilon}$. By definition, there is a decoder $D$ such that $D\circ E$ is the identity map on $\calN_\epsilon$. We extend $E$ to $\calF$ by first mapping $f$ to an $\epsilon$ approximation $g$ and then using the bit representation of $g$. Hence $(E,D)$ is an $\epsilon$-approximate encoder-decoder pair, and we call $B_\epsilon$ the number of bits used by the neural network. 
%\subsection{Remarks on number of bits}

When one refers to the ``number of bits", perhaps an immediate inclination is to envision the number of nonzero bits required to store the weights of a single network into memory. This is usually the perspective taken if the goal is to quantize a particular network to reduce its memory cost, see \cite{courbariaux2015binaryconnect}. The number of bits to specify a single neural network is (often significantly) smaller than the minimax code length, which we proceed to explain below. 

For each $\calA$-quantized neural network $g\in\calN_\epsilon$, we can list out the parameters and topology of the network as a finite sequence $y(q)$ in $\calA\cup \{0\}$, where `0' is used to specify that a weight or bias is not being used from one node to another. Since $\calA$ is a finite set, it can be encoded with $\lceil\log|\calA|\rceil$ bits. Discarding each $0$ in $y$ and replacing each nonzero term in $y$ with its corresponding $\pm 1$ representation, we obtain a finite sequence $q(g)\in \{\pm 1\}^{B_\epsilon(g)}$. This quantity $B_\epsilon(g)$ is sometimes referred to as the number of bits required to store the network $g$. However, if we compare $B^\circ_\epsilon:=\sup_{g\in \calN_\epsilon} B_\epsilon(g)$ to the number of bits $B_\epsilon$, it is possible that $B^\circ_\epsilon = o(B_\epsilon)$ as $\epsilon\to0$. The short explanation is that the map $\calN_\epsilon\mapsto \{\pm 1\}^{B^\circ_\epsilon}$ given as $g\mapsto q(g)$ is not necessarily injective because `0' is required to specify the network's topology, which has been omitted from the quantity $B_\epsilon(g)$. In this case, if $\calN_\epsilon$ contains networks with vastly different topologies, then additional bits may be required to distinguish between the networks' topologies and it may not be feasible to simply discard all zeros. 

Going back to the content of this paper, our strategy is to approximate $f\in C^s([0,1]^d)$ by a one-bit linear combination of (possibly approximate) Bernstein polynomials. The latter set only depends on the function class and prescribed error, and not on a particular $f$. Hence the networks used to approximate $C^s([0,1]^d)$ all have the same topology, and for this reason, we do not need to include `0' as a bit. The following theorem quantifies the number of bits.   

\commentout{

To make this discussion more concrete, let us consider the approximation strategy developed by Yartosky \cite{yarotsky2017error}, whereby each $f\in C^s([0,1])$ with $\|f\|_{C^s}\leq 1$ is approximated by an unquantized ReLU neural network that implements the approximant 
$$
\sum_{0\leq \bfk\leq N} \sum_{|\bfalpha|\leq s} c_{\bfk,\bfalpha} \phi_{\bfk}(\bfx) \(\bfx-\frac{\bfk}{N}\)^{\bfalpha}.
$$
Here, $\{\phi_{\bfk}\}_{0\leq \bfk\leq N}$ forms a partition of unity of $[0,1]^d$, $\phi_{\bfk}$ is compactly supported in the cube of side length $3/N$ centered at $\bfk/N$, only the coefficients $\{c_{\bfk,\bfalpha}\}_{0\leq \bfk\leq N, |\bfalpha|\leq s}$ depend on $f$, and $N$ is a parameter that is chosen appropriately as to balance the approximation and size of the resulting network.

Although Yartosky did not consider quantized networks, the results in \ref{sec:relu} can be adapted to show that $\big\{ \phi_{\bfk}(\bfx) \(\bfx-\frac{\bfk}{N}\)^{\bfalpha}\}_{0\leq \bfk\leq N, |\bfalpha|\leq s}$ are implementable by one-bit ReLU networks. This is not central to the discussion, as this set of functions does not depend on $f$ and any encoding is an overhead cost. The core of this discussion revolves around the coefficients $\{c_{\bfk,\bfalpha}\}_{0\leq \bfk\leq N, |\bfalpha|\leq s}$. In fact, $c_{\bfk,\bfalpha}$  is the Taylor coefficient of $f$ expanded at $\bfk/N$ and of order $\bfalpha$. Due to the assumption that $\|f\|_{C^s}\leq $, we have $|c_{\bfk,\bfalpha}|\leq 1$, but we cannot say more than that. 
}

\begin{theorem}
	\label{thm:mainbits}
	For any $s,d\geq 1$ and sufficiently small $\mu>0$, the following hold. There exist an encoder $E\colon C^s([0,1]^d)\to \{\pm 1\}^{B}$ and decoder $D\colon \{\pm 1\}^B\to\calN(\{\pm 1\},\rho,L,N,P)$ such that $L=O(\epsilon^{-2/s})$ and $\max(B,N,P)=O(\epsilon^{-2d/s})$ as $\epsilon\to0$, and
	$$
	\sup_{\|f\|_\infty\leq \mu, \, \|f\|_{C^s}\leq 1} \|f-D(E(f))\|_\infty \leq \epsilon.
	$$
\end{theorem}

\begin{proof}
	For any $f\in C^s([0,1]^d)$, Whitney's extension theorem provides us with a $F\in C^s(\R^d)$ such that $F=f$ on $[0,1]^d$ and $\|F\|_{C^s(\R^d)}\lesssim_{d,s} \|f\|_{C^s([0,1]^d)}$, where importantly, the implicit constant that appears in this inequality does not depend on $f$, see \cite{whitney1934analytic} and also \cite[Chapter 6, Theorem 4]{stein1970singular}. 
	
	Consider the set $U:=[-\frac{1}{2},\frac{3}{2}]\times[0,1]^{d-1}$. We will only consider Whitney extensions $f$ of that are compactly supported in $U$ -- this is possible since given any Whitney extension of $f$, which might not necessarily be compactly supported, we can multiply it by a smooth function that is compactly supported in $U$ and identically equal to 1 on $[0,1]^d$. 
	
	Hence, let $W_0$ and $W_2$ be the best possible constants for which $\|F\|_{L^\infty(E)} \leq W_0\|f\|_\infty$ and $\|F\|_{C^2(E)} \leq W_2\|f\|_{C^2}$, where $F$ is a Whitney extension of $f$ that is compactly supported in $U$. Note that $W_0$ and $W_2$ only depend on $d,s,U$. 
	
	For now, fix a $f\in C^s([0,1]^d)$ such that $\|f\|_\infty\leq \mu$ and $\|f\|_{C^s}\leq 1$. Let $F\in C^s(\R^d)$ be a Whitney extension of $f$ that is compactly supported in $U$. Consider the function $\tilde F\in C^s([0,1]^d)$ defined as $\tilde F(\bfx):=F(2x_1-\frac{1}{2},x_2,\dots,x_d)$. Notice that $\|\tilde F\|_\infty =\|F\|_\infty \leq W_0\|f\|_\infty \leq W_0\mu$ and that $\|\tilde F\|_{C^2([0,1]^d)}\leq 4\|F\|_{C^2}\leq 4W_2 \|f\|_{C^2}$.
	
	From here onward, assume that $\mu$ is sufficiently small so that $W_0\mu \leq 1/2$, and note that $\mu$ only needs to be sufficiently small depending on $s,d$ since $W_0$ only depends on these quantities. For any integer $\displaystyle n\geq 2sd^2W_2$, so that $n\geq {sd^2 W_2}/({1-W_0\mu})$, we apply \cref{thm:mainbernstein} to $\tilde F$ with $\ell=1$. Hence, there exists $\{\sigma_{\bfk}\}_{0\leq \bfk\leq n}\subset\{\pm 1\}$ such that for all $\bfx\in [0,1]^d$ with $\frac{1}{4}\leq x_1\leq \frac{3}{4}$, we have 
	\begin{equation}
		\label{eq:thm3help}
		\big|\tilde F(\bfx)-\tilde H(\bfx)\big|
		%\lesssim_{s,d,\mu} \|\tilde F\|_{C^s}  \min\big(1,n^{-s/2}x_1^{-s} (1-x_1)^{-s} \big)
		\lesssim_{s,d} n^{-s/2}, \quad\text{where}\quad \tilde H(\bfx):=\sum_{0\leq \bfk\leq n} \sigma_{\bfk} p_{n,\bfk}(\bfx).
	\end{equation}
	Define $H\colon [0,1]^d\to\R$ by $H(\bfx):=\tilde H(\frac{1}{2}x_1+\frac{1}{4},x_2,\dots,x_d)$. For each $\bfx\in [0,1]^d$, 
	\begin{equation}
		\label{eq:thm3help2}
		|f(\bfx)-H(\bfx)|
		=\Big|\tilde F\Big(\frac{1}{2}x_1+\frac{1}{4},x_2,\dots,x_d\Big) - \tilde H\Big(\frac{1}{2}x_1+\frac{1}{4},x_2,\dots,x_d\Big)\Big|
		\lesssim_{s,d} n^{-s/2},
	\end{equation}
	where the final inequality follows from \eqref{eq:thm3help}. Hence, for any $\epsilon$ sufficiently small, by making $n$ large enough depending only on $\epsilon$, $d$, and $s$, the right hand side of \eqref{eq:thm3help2} can be made smaller than $\epsilon$. We see that $n=O(\epsilon^{-2/s})$ as $\epsilon\to0$ suffices. 
	
	We next show that $H$ can be implemented by a neural network and we determine the number of bits used in this encoding. Recall that 
	$$
	H(\bfx)=\sum_{0\leq \bfk\leq n} \sigma_{\bfk} p_{n,\bfk}\Big(\frac{1}{2}x_1+\frac{1}{4},x_2,\dots,x_d\Big). 
	$$
	The function $x_1\mapsto \frac{1}{2}x_1+\frac{1}{4}$ can be made by a neural of size $(4,4,9)$ because we first make $\frac{1}{2}=\rho(0\cdot x_1+1)$, then $\frac{1}{4}=\rho(\frac{1}{2})$, then $\rho(x_1+\frac{1}{4})$ and $\rho(x_1-\frac{1}{4})$, and finally $\frac{1}{2}x_1+\frac{1}{4}=\rho(x_1+\frac{1}{4})-\rho(x_1-\frac{1}{4})+\frac{1}{4}$. As shown in \cref{thm:mainimplementation}, the multivariate Bernstein polynomials $\{p_{n,\bfk}\}_{0\leq \bfk\leq n}$ are implementable by a $\{\pm 1\}$-quantized quadratic neural network that does not depend on the target function $f$, and it has $O(n)$ layers and $O(n^d)$ nodes and parameters as $n\to\infty$. Hence, we can implement the functions
	$$
	\Big\{ \bfx\mapsto p_{n,\bfk}\Big(\frac{1}{2}x_1+\frac{1}{4},x_2,\dots,x_d\Big)\Big\}_{0\leq \bfk\leq n}
	$$
	using a $\{\pm 1\}$-quantized quadratic network that only depends on $n$ (which is selected in terms of $s,d$) and not on $f$. Hence the number of bits to encode this part is $O(\frac{2d}{s}\log(1/\epsilon))$. By incorporating a linear last layer of size $(1,1,(n+1)^d)$, whose weights are specified by $\{\sigma_{\bfk}\}_{0\leq \bfk\leq n}$, we see that $H$ is implementable. The signs $\{\sigma_{\bfk}\}_{0\leq \bfk\leq n}$ depend on $f$ and $\sigma_{\bfk}\in \{\pm 1\}$, so the total number of bits used to encode this part is $(n+1)^d=O(\epsilon^{-2d/s})$ as $\epsilon\to0$.	

%		\underline{ReLU activation with three-bit alphabet}. Fix $\epsilon\in (0,1)$ and let $\calA_3:=\{\pm 1/2, \pm 1, \pm 2\}$ for convenience. As shown in \cref{thm:bernrelu}, there is a $\calA_3$-quantized ReLU neural network that implements a set of functions $\{b_{n,\bfk}\}_{0\leq \bfk\leq n}$ such that $\|p_{n,\bfk}-b_{n,\bfk}\|_\infty	\leq \epsilon$ for each $0\leq \bfk\leq n$. Then we define the function  
%		$$
%		f_{NN}
%		:=\sum_{0\leq \bfk \leq n} \sigma_{\bfk} b_{n,\bfk},
%		$$
%		which is still implementable by a $\calA_3$-quantized neural network after placing a linear layer with $\{\pm 1\}$ weights and of size $(1,1,(n+1)^d)$ at the end. In total, $f_{NN}$ is implementable by a $\calA_3$-quantized neural network with
%		\begin{align*}
%			O\big(n\log n + (n+d)(d +\log(1/\epsilon))\big) \quad &\text{layers}, \\
%			O\big( n^2\log n + (n^2+dn^d)(d +\log(1/\epsilon))\big) \quad&\text{nodes and parameters}. 	
%		\end{align*}
%		
%		For the analysis of the implementation error, since $|\sigma_{\bfk}|=1$, we have 
%		$$
%		\Big\| \sum_{0\leq \bfk \leq n} \sigma_{\bfk} p_{n,\bfk}-\sum_{0\leq \bfk \leq n} \sigma_{\bfk} b_{n,\bfk}\Big\|_\infty
%		\leq \sum_{0\leq \bfk \leq n} \|p_{n,\bfk}-b_{n,\bfk}\|_\infty
%		\leq (n+1)^d\epsilon.
%		$$
%		From here, we select $\epsilon= (n+1)^{-d}n^{-s/2}$. Note that the network which implements $\{b_{n,\bfk}\}_{0\leq \bfk\leq n}$ does not depend on $f$, only on the function class $C^s([0,1]^d)$. Since only the signs $\{\sigma_{\bfk}\}_{0\leq \bfk\leq n}$ depend on $f$, the total number of bit used is $(n+1)^d$.

\end{proof}

To put \cref{thm:mainbits} in context, Kolmogorov entropy tells us that the minimum number of bits necessary to encode the unit ball of $C^s([0,1]^d)$ with error at most $\epsilon$ measured in the uniform norm is $O(\epsilon^{-d/s})$, regardless of which encode-decoder pair is used, see \cite[Chapter 7.5]{shiryayev1993selected} and \cite{kolmogorov1959varepsilon} for the one-dimensional version and \cite[page 86]{vitushkin1961theory} for a full proof the multidimensional case. Although our approximation strategy via Bernstein polynomials does not attain the entropy rate, it has several other desirable features that are not captured by entropy considerations. Our method only uses a one-bit alphabet (which is not enforced by bit counting considerations), can be readily computed from queries, and is stable to perturbations of the unknown function. 

\subsection{Beyond directional $\Sigma\Delta$?}

One unsatisfying aspect of the quantization schemes employed and analyzed in Section \ref{sec:Bquan} is that they are inherently one-dimensional. It is due to this that the quantization error bound in \cref{thm:rsigmadelta} contains the 
$
x_\ell^{-r}(1-x_\ell)^{-r}
$
term which blows up at the boundaries $x_\ell=0$ and $x_\ell=1$. This term is then propagated to Theorems \ref{thm:mainbernstein} and \ref{thm:main}. 

On one hand, the inability to approximate the target function at the endpoints is unavoidable. To see why, $p_{n,k}(0)=\delta_{0,k}$ and $p_{n,k}(1)=\delta_{n,k}$, and consequently, if $\sigma_k\in \{\pm 1\}$ for each $k$, then $\sum_{k=0}^n \sigma_k p_{n,k}$ can only be equal to $\pm 1$ at the endpoints. This means that a $\{\pm 1\}$ linear combination of Bernstein polynomials cannot possibly approximate any continuous function uniformly on $[0,1]$, without additional assumptions on the target function near the endpoints. 

On the other hand, one may wonder if it is possible to spread the error out to other faces of the $d$-dimensional cube, instead of concentrating all it to the faces where $x_\ell=0$ and $x_\ell=1$. More precisely, perhaps one could replace the term $x_\ell^{-r}(1-x_\ell)^{-r}$ with 
$$\bfx^{-\bfr}(1-\bfx)^{-\bfr} = \prod_{\ell=1}^d x_\ell^{-r_\ell} (1-x_\ell)^{-r_\ell}, \quad\text{for any}\quad  |\bfr| = r.$$

The main bottleneck of carrying out the more general case stems from difficulties with $\Sigma\Delta$ quantization. Currently, there is no such available stable one-bit multidimensional higher order $\Sigma\Delta$ scheme, beyond the directional one used in this paper. More precisely, we define a stable one-bit $\bfr$-th order $\Sigma\Delta$ quantizer with alphabet $\calA$ to be a map that takes any sequence $y$ with $\|y\|_\infty\leq\mu<1$ and outputs $q$, such that $q_{\bfk}\in\calA$ and there is an associated $u$ with $\|u\|_\infty\leq C_{\mu,\bfr,\calA}$ (that does not depend on $y$) satisfying the equation 
$$
y-q = \Delta^{\bfr}u, \quad \text{where}\quad \Delta^{\bfr}=\Delta_d^{r_d}\Delta_{d-1}^{r_{d-1}} \cdots \Delta_1^{r_1}.
$$
We say this scheme is truly multidimensional if there are $j\not=k$ such that $r_j,r_k\geq 1$. 

If a stable $\bfr$-th order $\Sigma\Delta$ scheme exists, then it is straightforward to modify the analysis given in the proof of \cref{thm:rsigmadelta}. Indeed, let $\{\sigma_{\bfk}\}_{0\leq \bfk\leq n}$ with $\sigma_{\bfk}\in \calA$ be the output given input $\{a_{\bfk}\}_{0\leq \bfk\leq n}$. Since $a-\sigma=\Delta^{\bfr}u$ for some $\|u\|_\infty \leq C_{\mu,\bfr,\calA}$, a summation by parts yields
\begin{equation*}
	\Big|\sum_{0\leq \bfk\leq n} (a_{\bfk}-\sigma_{\bfk}) \ p_{n,\bfk}(\bfx)\Big| 
	=\Big|\sum_{0\leq \bfk\leq n} (\Delta^{\bfr} u )_{\bfk} \ p_{n,\bfk}(\bfx)\Big|
	\lesssim_{\mu,\bfr,\calA} \sum_{0\leq \bfk\leq n} \big|\big((\Delta^{\bfr})^* p_{n,\cdot} (\bfx)\big)_{\bfk} \big|.
\end{equation*}
Note that $(\Delta^{\bfr})^*=(\Delta_d^*)^{r_d} (\Delta_{d-1}^*)^{r_{d-1}} (\Delta_1^*)^{r_1}$. The Bernstein polynomials are tensor products, so 
$$
\sum_{0\leq \bfk\leq n} \big|\big((\Delta^{\bfr})^* p_{n,\cdot} (\bfx)\big)_{\bfk} \big|
=\prod_{\ell=1}^d \sum_{k_\ell=0}^n  \Big| \big( (\Delta^*)^{r_\ell} p_{n,\cdot}(x_\ell)\big)_{k_\ell} \Big| \lesssim_{r_1,\dots,r_d} \prod_{\ell=1}^d n^{-r_\ell} x_{\ell}^{-r_\ell} (1-x_\ell)^{-r_\ell},
$$
where we used the one dimensional result in \cite[Theorem 6]{onebitBernstein}. Together with the trivial estimate that the quantization error is bounded by 2, we have proved that for all $\bfx\in [0,1]^d$,
$$
\Big|\sum_{0\leq \bfk\leq n} (a_{\bfk}-\sigma_{\bfk}) \ p_{n,\bfk}(\bfx)\Big|
\lesssim_{\mu,\bfr,\calA} \min\big(1, n^{-|\bfr|/2}\bfx^{-\bfr}(1-\bfx)^{-\bfr} \big).
$$
With this at hand, Theorem \cref{thm:mainbernstein} can be modified as follows. Under the same assumptions on $d,s,n,\mu,f$, for any $\bfs$ with $|\bfs|=s$, there exist $\{\sigma_{\bfk}\}_{0\leq \bfk\leq n}$ with $\sigma_{\bfk}\in \calA$ such that 
$$
\Big|f(\bfx)-\sum_{0\leq \bfk\leq n} \sigma_{\bfk} p_{n,\bfk}(\bfx)\Big|
\lesssim_{\bfs,d,\mu,\calA} \|f\|_{C^s}  \min\big(1,n^{-s/2} \bfx^{-\bfs} (1-\bfx)^{-\bfs} \big).
$$
In the second step of \cref{alg:binarybernalg}, one can use a $\bfs$-th order $\Sigma\Delta$ scheme instead of a directional one. It is important to emphasize again that there is no known stable one-bit $\Sigma\Delta$ quantizer, so this discussion cannot be extended to the one-bit case.

\appendix
\section{Neural network constructions}

\label{sec:NNappendix}

This appendix constructs the desired one-bit neural networks. In the subsequent proofs, we just use the term ``network" since the alphabet and activation do not vary in each subsection. To visualize the topology of network, we typically draw a schematic diagram. Nodes belonging to the same layer are placed horizontally, layers are stacked vertically, with the input nodes on top and output nodes on bottom. We say the $k$-th node in layer $\ell-1$ is connected to the $j$-th node in layer $\ell$ if either $(A_\ell)_{j,k}$ or $(\bfb_\ell)_{j}$ is nonzero. Two connected nodes are represented in a schematic diagram by a line segment adjoining them. 

\subsection{Implementation of $\{\pm 1\}$-quantized quadratic networks}

\label{sec:quadappendix}

\begin{figure}[h]
	\centering
	\begin{subfigure}[b]{0.45\textwidth}
		\centering
		\includegraphics[width=0.8\textwidth]{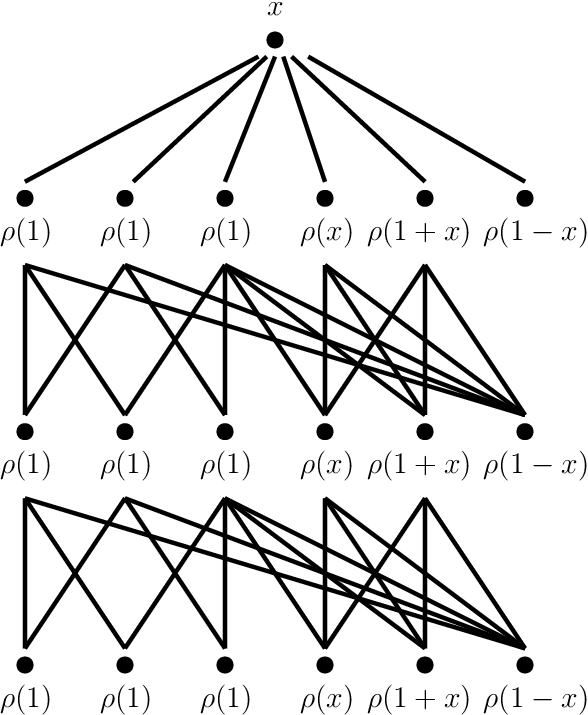}
		\vspace{1em}
		\caption{$\rho(1),\rho(x),\rho(1+x),\rho(1-x)$}	
		\label{fig:6node}
	\end{subfigure} \quad \quad 
	\begin{subfigure}[b]{0.45\textwidth}
		\centering
		\includegraphics[width=0.9\textwidth]{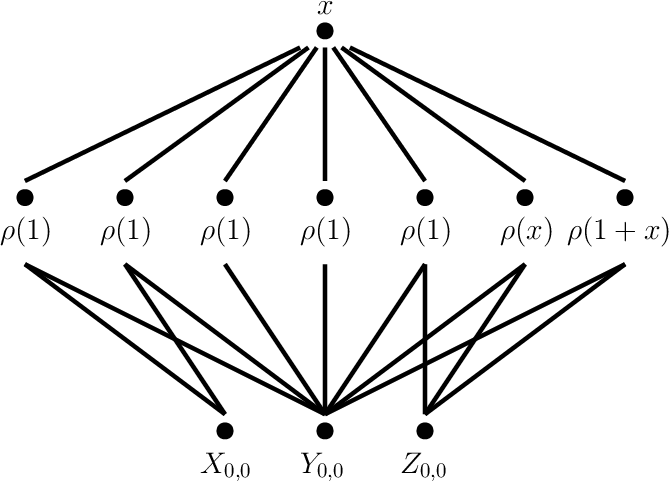}
		\vspace{3em}
		\caption{$X_{0,0}, Y_{0,0}, Z_{0,0}$}
		\label{fig:XYZ}
	\end{subfigure}
	\caption{Helper activated one-bit quadratic networks}
\end{figure}

\begin{lemma}
	\label{prop:XYZ}
	For any integer $n\geq 1$, there exists an activated $\calA_{1}$-quantized quadratic neural network with $n+1$ layers and $O(n)$ nodes and parameters that implements the function 
	$$
	x\mapsto \(U(x), \, V(x), \, \{X_{n-1,k}(x)\}_{k=0}^{n-1}, \, \{Y_{n-1,k}(x)\}_{k=0}^{n-1}, \, \{Z_{n-1,k}(x)\}_{k=0}^{n-1}\).
	$$
\end{lemma}

\begin{proof}
	We will create an activated network such that $\{X_{m,k}\}_{k=0}^{m}$, $\{Y_{m,k}\}_{k=0}^{m}$ and $\{Z_{m,k}\}_{k=0}^{m}$ are outputs in layer $m+2$. Throughout this proof, we will implicitly use the identity $x=\rho(x+1)-\rho(x)-\rho(1)$ without explicit mention. To carry out this strategy, notice that \eqref{eq:XYZ} and \eqref{eq:bernrecurrence2} imply
	
	\begin{gather}
		\begin{split}
			X_{m,k} & = 
			\begin{cases}
				\ \rho(Y_{m-1,0}-U-X_{m-1,0}) &\text{if } k=0,\\
				\ \rho(Z_{m-1,k-1}-V-X_{m-1,k-1}+Y_{m-1,k}-U-X_{m-1,k}) &\text{if } 0 < k < m, \\
				\ \rho(Z_{m-1,m-1}-V-X_{m-1,m-1}) &\text{if } k=m.
			\end{cases} \\
			Y_{m,k} & =  
			\begin{cases}
				\ \rho(1-x+Y_{m-1,0}-U-X_{m-1,0}) & \text{if } k=0,\\
				\ \rho(1-x+Z_{m-1,k-1}-V-X_{m-1,k-1}+Y_{m-1,k}-U-X_{m-1,k}) &\text{if }  0 < k < m, \\
				\ \rho(1-x+Z_{m-1,m-1}-V-X_{m-1,m-1}) &\text{if } k=m.
			\end{cases} \\
			Z_{m,k} & =  
			\begin{cases}
				\ \rho(x+Y_{m-1,0}-U-X_{m-1,0}) &\text{if } k=0,\\
				\ \rho(x+Z_{m-1,k-1}-V-X_{m-1,k-1}+Y_{m-1,k}-U-X_{m-1,k}) &\text{if }  0 < k < m, \\
				\ \rho(x+Z_{m-1,m-1}-V-X_{m-1,m-1}) &\text{if } k=m.
			\end{cases} 
		\end{split}
		\label{eq:XYZrecurrence} 	
	\end{gather}

	To employ the recurrence \eqref{eq:XYZrecurrence}, we first construct a network with $n+1$ layers such that each layer has six nodes that output $\rho(1)$, $\rho(1)$, $\rho(1)$, $\rho(x)$, $\rho(1+x)$, and $\rho(1-x)$. For layer 1, these six outputs are readily made from the input $x$. The same six outputs can be generated by $\rho$ applied to $\pm 1$ linear combinations of the same terms, since 
	\begin{align*}
		\rho(1)&=\rho(\rho(1)+\rho(1)), \quad 
		\rho(1-x)=\rho(\rho(1)+\rho(1)+\rho(1)-\rho(1+x)+\rho(x)), \\
		\rho(x)&=\rho(\rho(1+x)-\rho(x)-\rho(1)), \quad
		\rho(1+x)=\rho(\rho(1+x)-\rho(x)+\rho(1)).
	\end{align*}
	This is shown in Figure \ref{fig:6node}. Since a constant number of nodes and parameters are added to increase the depth by one, terminating at layer $n+1$, this network has size $O(n)$. This establishes that $U,V$ are outputs of layer $n+1$.
	
	We can produce $X_{0,0},Y_{0,0},Z_{0,0}$ with a network of size $(2,10,20)$, as shown in Figure \ref{fig:XYZ}, since 
	\begin{align*}
		X_{0,0}&=\rho(1)=\rho( \rho(1)+\rho(1)), \\
		Y_{0,0}&=\rho(1-x+1)=\rho(\rho(1)+\rho(1)-\rho(x+1)+\rho(x)+\rho(1)+\rho(1)+\rho(1)), \\
		Z_{0,0}&=\rho(x+1)=\rho(\rho(x+1)-\rho(x)+\rho(1)).
	\end{align*} 
	It follows from recurrence \eqref{eq:XYZrecurrence} that if layer $m$ outputs the quantities $\{X_{m-1,k}\}_{k=0}^{m-1}$, $\{Y_{m-1,k}\}_{k=0}^{m-1}$, $\{Z_{m-1,k}\}_{k=0}^{m-1}$ and $\rho(1),\rho(1),\rho(1),\rho(x),\rho(1+x)$, then it is possible to generate the quantities $\{X_{m,k}\}_{k=0}^m$, $\{Y_{m,k}\}_{k=0}^{m}$, $\{Z_{m,k}\}_{k=0}^{m}$ in the subsequent layer $m+1$ by adding only a constant number of nodes and parameters.
\end{proof}

%The preceding lemma allows us to build the ingredients that are necessary to produce the univariate Bernstein polynomials by an activated network. To deal with the multivariate case, we seek to find a suitable way to multiply real numbers.

\begin{figure}
	\centering
	\begin{subfigure}{0.45\textwidth}
		\centering
		\includegraphics[width=0.9\textwidth]{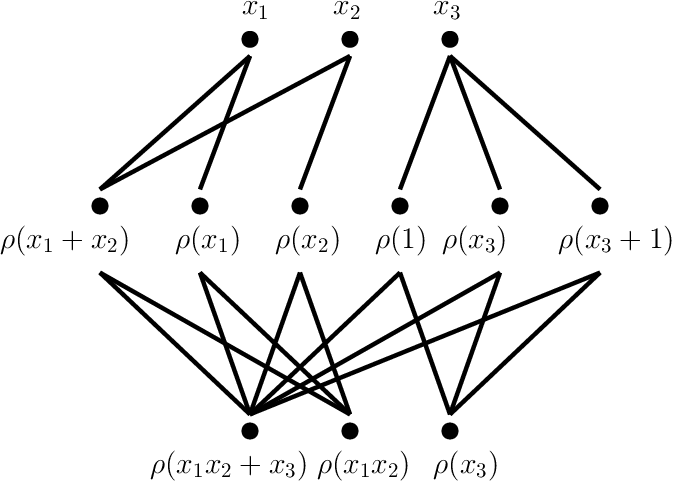}
		\smallskip 
		\caption{Three terms}
		\label{fig:3mult} 
	\end{subfigure}
	\begin{subfigure}{0.45\textwidth}
		\centering
		\includegraphics[width=0.8\textwidth]{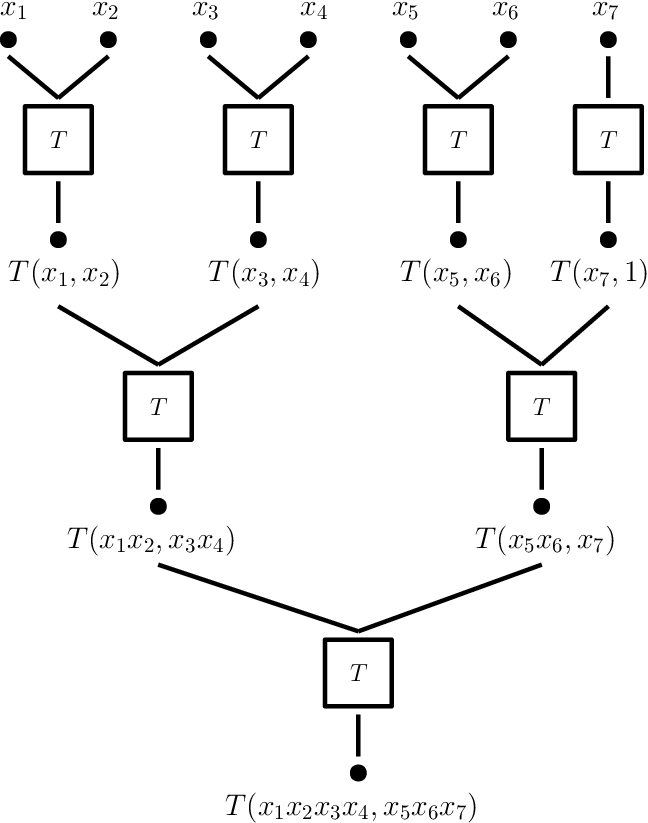}
		\smallskip 
		\caption{Seven terms}
		\label{fig:7mult} 
	\end{subfigure}
	\caption{Pseudo-multiplication networks}
\end{figure}

\begin{lemma}
	\label{prop:multseveralquadratic}
	For any $d\geq 2$, there exist an integer $d_*$ and an activated $\calA_1$-quantized quadratic neural network with $O(\log d)$ layers and $O(d)$ nodes and parameters, as $d\to\infty$, that implements the function 
	$$
	\bfx =(x_1,\dots,x_d)
	\mapsto \Big( \rho\( \prod_{j\leq d_*} x_j + \prod_{j>d_*} x_j \), \, \rho\( \prod_{j\leq d_*} x_j\), \, \rho\( \prod_{j>d_*} x_j \)\Big). 
	$$
\end{lemma}

\begin{proof}	
	The $d=2$ case is handled by \eqref{eq:multbinaryquad} whereby $d_*=1$. For $d=3$, we create a two layer network of size (2,9,20) whose output is $(\rho(x_1x_2+x_3), \rho(x_1x_2), \rho(x_3))$, as shown in Figure \ref{fig:3mult}. This is possible because $
	x_1x_2=\rho(x_1+x_2)-\rho(x_1)-\rho(x_2)$ and $x_3=\rho(x_3+1)-\rho(x_3)-\rho(1)$, which are $\pm 1$ linear combinations of outputs from the first layer. Hence, $d_*=2$ when $d=3$.
	
	From now on, assume that $d\geq 4$. The main idea is that each layer, we generate the three terms on the right hand side \eqref{eq:multbinaryquad} for $\lfloor d/2 \rfloor$ pairs of products, and then iterate the identity, which finally terminates after $O(\log d)$ iterations. Each step of the iteration depends on the number theoretic properties of $d$. The basic observation that allows for termination of this argument is that: if $abcd=\prod_{j=1}^d x_j$ and there is a layer that outputs 
	$$
	\rho(a),\rho(a+b),\rho(b), \rho(c),\rho(c+d),\rho(d),
	$$
	then by appending another network of size $(1,3,12)$, we can implement 
	$\rho(ab)$, $\rho(ab+cd)$, and $\rho(cd)$, which would complete the proof. To simplify the resulting argument, we introduce the shorthand notation 
	$$
	T(a,b):=(\rho(a),\rho(a+b),\rho(b)).
	$$
	
	Each step of the reduction depends on the factorization of $d$. For layer 1, if $d$ is even, we create $3d/2$ nodes whose outputs are 
	$$
	T(x_1,x_2), T(x_3,x_4), \dots, T(x_{d-1},x_d).
	$$
	If $d$ is odd, in layer 1, we create $3\lceil d/2\rceil +1$ nodes whose outputs are 
	$$
	T(x_1,x_2), T(x_3,x_4), \dots, T(x_{d-2},x_{d-1}), T(x_d,1), \rho(1). 
	$$
	Let $d_1:=\lceil d/2\rceil$, and we can assume that $d_1\geq 4$, otherwise we are finished by the above basic operation. The construction of layer 2 depends on both the parities of $d_1$ and $d$. If both $d$ and $d_1$ are even, we use \eqref{eq:multbinaryquad} again to create $3d_1/2$ nodes whose outputs are 
	$$
	T(x_1x_2,x_3x_4), \dots, T(x_{d-3} x_{d-2},x_{d-1} x_d).
	$$
	The remaining three cases are fairly similar except for when both $d$ and $d_1$ are odd. In which case, we can simply reproduce $T(x_d,1),\rho(1)$ in the next layer from the $T(x_d,1), \rho(1)$ nodes in the previous layer. Indeed, in layer 2, we have nodes that output 
	$$
	T(x_1x_2,x_3x_4), \dots, T(x_{d-4} x_{d-3},x_{d-2} x_{d-1}),T(x_d,1), \rho(1).
	$$
	
	We continue this strategy until we end with the situation expressed in the basic observation. Notice that there are at most $O(\log d)$ iterations, and each iteration requires appending a network with only a single layer. Furthermore, the number of nodes and parameters used in layer $\ell$ is $O(d/2^\ell)$, so there are at most $O(d)$ many nodes and parameters in total.
\end{proof}

Figure \ref{fig:7mult} displays the corresponding network constructed in \cref{prop:multseveralquadratic} for $d=7$ whereby $d_*=4$. The next theorem shows that any $\pm 1$ linear combination of multivariate Bernstein polynomials is implementable by a neural network. We remark that the following does not claim that the Bernstein polynomials are implementable by an activated neural network.

\subsection{Implementation of $\{\pm \frac{1}{2}\}$-quantized ReLU networks}
\label{sec:reluappendix}

%In this subsection we concentrate on $\calA_{1/2}$-quantized neural networks with ReLU activation $\sigma$. In the expository portions, we just use the term ``network" since the alphabet and activation do not change in this subsection. To turn the schematic diagrams shown in Figures \ref{fig:Bnetwork} and \ref{fig:Bpascal} into a proper neural network, it suffices to convert multiplications by $1-x$ and $x$ into neural network operations.   

We begin with some basic properties of the ReLU function, which we will use without explicit reference. For any $t\geq 0$, $s\in\R$, and $\bfx,\bfy\in\R^n$, it holds that 
$$
\sigma(t\bfx)=t\sigma(\bfx), \quad \sigma(\sigma(\bfx))=\sigma(\bfx), \quad \sigma(\bfx+\bfy)\leq \sigma(\bfx)+\sigma(\bfy), \andspace |\sigma(s)-t|\leq |s-t|.
$$
We next describe a few functions that are implementable by an activated $\calA_{1/2}$-quantized ReLU network. For any input $x\in\R$, the number $1/2$ can be implemented by a network of size $(1,1,1)$ since $1/2=\sigma(0x+1/2)$. Implementation of the map $x\mapsto 2^{-m}\sigma(x)$ for any integer $m\geq 1$ can be achieved by a network of size $(m,m,m)$ by
$$
x\mapsto \sigma\(\frac{1}{2}\sigma\(\frac{1}{2}\sigma\(\cdots \sigma\(\frac{1}{2}x\)\)\)\)=\frac{1}{2^m}\sigma(x).  
$$
Summation of nonnegative numbers can be carried out by a $\calA_{1/2}$-quantized network of size $(2,9,16)$ because for nonnegative $a,b$, we have 
$$
a+b = \sigma\( \frac{1}{2}\sigma\(\frac{1}{2}a\) + \cdots + \frac{1}{2}\sigma\(\frac{1}{2}a\) +  \frac{1}{2}\sigma\(\frac{1}{2}b\) + \cdots + \frac{1}{2}\sigma\(\frac{1}{2}b\)\). 
$$
Likewise, for nonnegative $a,b$, the quantity $\sigma(a-b)$ is implementable as well.  

\begin{figure}[t]
	\centering 
	\begin{subfigure}[b]{0.45\textwidth}
		\centering
		\includegraphics[width=0.9\textwidth]{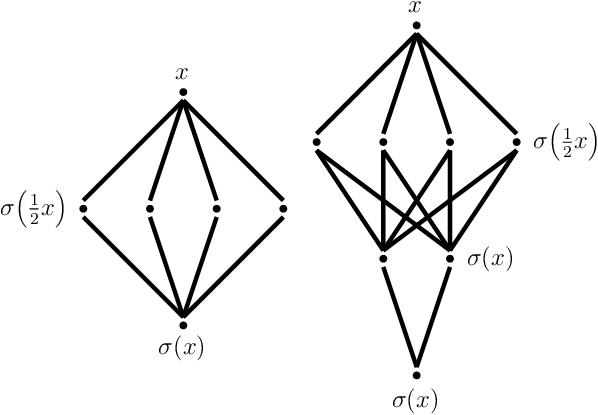}
		\caption{Duplication networks}
		\label{fig:duplication}
	\end{subfigure}	\quad
	\begin{subfigure}[b]{0.45\textwidth}
		\centering
		\includegraphics[width=\textwidth]{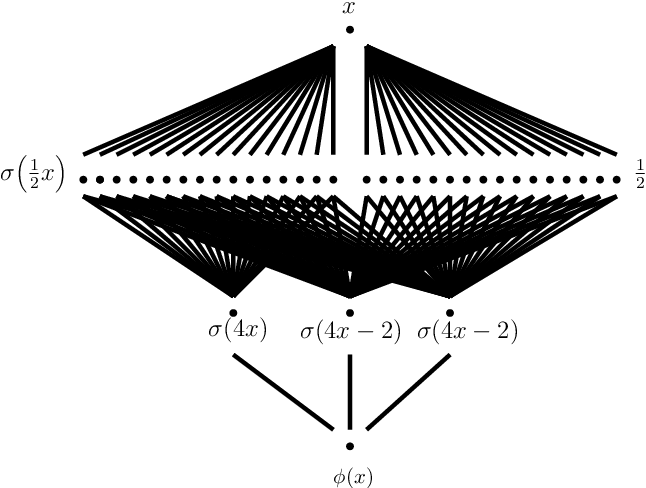}
		\caption{$\phi$-block }	
		\label{fig:phi}
	\end{subfigure} 
	
	\bigskip \bigskip 
	
	\begin{subfigure}{0.45\textwidth}
		\centering
		\includegraphics[width=0.5\textwidth]{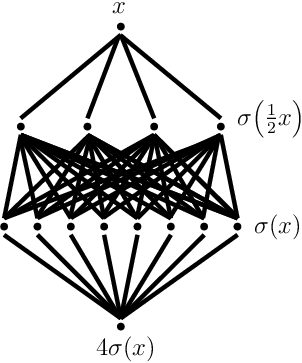}
		\caption{$\times 4$ network}
		\label{fig:times4}
	\end{subfigure} \quad 
	\begin{subfigure}{0.45\textwidth}
		\centering
		\includegraphics[width=\textwidth]{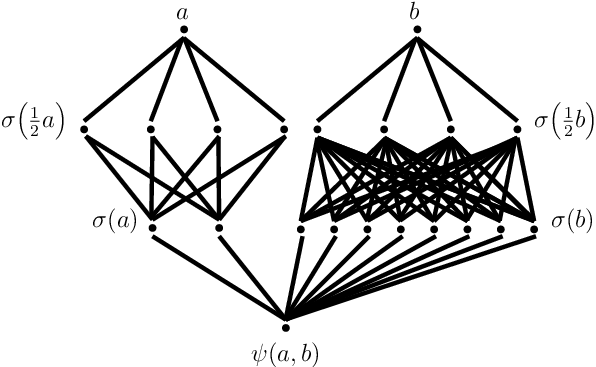}
		\caption{$\psi$ network}
		\label{fig:psi}
	\end{subfigure}
	\caption{Implementations of basic functions}
\end{figure}

For any integer $L\geq 2$, there is a $L$ layer activated $\calA_{1/2}$-quantized ReLU network $\zeta_L$ with $O(L)$ nodes and parameters that implements the map $x\mapsto\sigma(x)$. For $L=2$ and $L=3$, this can be done via the two networks shown in Figure \ref{fig:duplication}, which have size $(2,5,8)$ and $(3,7,14)$ respectively. For any $L\geq 4$, we can compose these maps so that the resulting $\zeta_L$ network has $L$ layers and $O(L)$ nodes and parameters. We call $\zeta_L$ a $L$ layer duplication network, and will be used to propagate nonnegative numbers down an arbitrary number of layers without the use of skip connections. 

The tent function $\phi$ can be implemented by a ReLU network with real parameters since $\phi(x)=\sigma(2x) - 2\sigma(2x-1)$. By making some adjustments, it is not difficult to see that $\phi$ can be implemented with an activated $\calA_{1/2}$-quantized network of size $(3,36,107)$ because of the identities
\begin{gather*}
	\begin{split}
		\phi(x)
		&=\sigma\( \frac{1}{2}\sigma(4x) - \frac{1}{2}\sigma(4x-2)- \frac{1}{2}\sigma(4x-2)\), \\
		\sigma(4x)&=\sigma\(\underbrace{\frac{1}{2}\sigma\(\frac{1}{2}x\)+\cdots+\frac{1}{2}\sigma\(\frac{1}{2}x\)}_{\text{16 times}}\), \\
		\sigma(4x-2)&= \sigma\(\underbrace{\frac{1}{2}\sigma\(\frac{1}{2}x\)+\cdots+\frac{1}{2}\sigma\(\frac{1}{2}x\)}_{\text{16 times}}-\underbrace{\frac{1}{2}\cdot\frac{1}{2}-\cdots-\frac{1}{2}\cdot\frac{1}{2}}_{\text{8 times}}\) . 
	\end{split}	
\end{gather*}
For convenience, we call this network a $\phi$-block and it is shown in \cref{fig:phi}.

\begin{lemma}
	\label{prop:squaring}
	For any $\epsilon>0$, there exist nonnegative functions $S_\epsilon^+,S_\epsilon^-\colon [0,1]\to\R$ such that 
	$$
	\text{for all } x\in [0,1], \quad S_\epsilon^-(x)\leq x^2\leq S_\epsilon^+(x), \quad  
	|S_\epsilon^+(x)-x^2|\leq \epsilon, \quad
	|S_\epsilon^-(x)-x^2|\leq \epsilon. 
	$$
	Furthermore, there exist two activated $\calA_{1/2}$-quantized ReLU neural network both with the same number of layers and each of size $O(\log(1/\epsilon))$ that implement $S_\epsilon^+$ and $S_\epsilon^-$.
\end{lemma}

\begin{proof}
	Fix $\epsilon>0$ and $m\geq 1$ will be an integer chosen later depending only on $\epsilon$. We define
	\[
	S_\epsilon^+(x)
	:=\sigma\(x-\sum_{k=1}^m \frac{\phi^{\circ k}(x)}{4^k}\) \quad\text{and}\quad S_\epsilon^-(x)
	:=\sigma\(S_\epsilon^+(x)-\frac{1}{2}\frac{1}{4^m}\). 
	\]
	Since $\phi^{\circ k}$ is nonnegative for all $k\geq 1$, it follows from \eqref{eq:blancmange} that $S_\epsilon^+(x)\geq x^2$ and
	$$
	|S_\epsilon^+(x)-x^2|
	=S_\epsilon^+(x)-x^2
	=\sum_{k=m+1}^\infty \frac{\phi^{\circ k}(x)}{4^k}
	\leq \sum_{k=m+1}^\infty  \frac{1}{4^k} 
	=\frac{1}{3}\frac{1}{4^m}.
	$$
	We first show that $S_\epsilon^-(x)\leq x^2$. This trivially holds if $S_\epsilon^+(x)\leq 1/(2\cdot 4^m)$. Otherwise, for $x\in [0,1]$ such that $S_\epsilon^+(x)\geq 1/(2\cdot 4^m)$, we have 
	\[
	S_\epsilon^-(x)
	=S_\epsilon^+(x)-\frac{1}{2}\frac{1}{4^m}
	=x^2+\sum_{k=m+1}^\infty \frac{\phi^{\circ k}(x)}{4^k}-\frac{1}{2}\frac{1}{4^m}
	\leq x^2,
	\]
	where the last inequality follows from the observation that 
	$$
	\sum_{k=m+1}^\infty \frac{\phi^{\circ k}(x)}{4^k}
	\leq \sum_{k=m+1}^\infty  \frac{\|\phi^{\circ k}\|_\infty}{4^k}
	\leq \sum_{k=m+1}^\infty  \frac{1}{4^k} 
	=\frac{1}{3}\frac{1}{4^m}
	<\frac{1}{2}\frac{1}{4^m}.
	$$
	Moreover, since $x^2\geq 0$, we have that 
	$$
	|S_\epsilon^-(x)-x^2|
	= \Big|\sigma\(S_\epsilon^+(x)-\frac{1}{2}\frac{1}{4^m}\)-x^2\Big|
	\leq |S_\epsilon^+(x)-x^2|+\frac{1}{2}\frac{1}{4^m}
	\leq \frac{1}{4^m}.
	$$
	Thus, we pick any integer $m\geq \log(1/\epsilon)/2$. 
	
	We proceed to discuss implementations of $S_\epsilon^+$ and $S_\epsilon^-$ by activated $\calA_{1/2}$-quantized ReLU neural networks, as shown in Figure \ref{fig:squaring}. The network for $S_\epsilon^+$ consists of only the black nodes and weights, and has a two column structure. The network for $S_\epsilon^-$ consists of all the nodes and weights, and has a three column structure. 
	
	We focus on $S_\epsilon^+$ first. The left column nodes implement the functions $\{\phi^{\circ k}\}_{k=1}^m$, which can be done by composing $m$ many $\phi$-blocks. Each $\phi^{\circ k}$ is produced in the $3k$-th layer, so the left column of the network has $3m$ layers and $O(m)$ nodes and parameters. The layer 3 right node outputs $4x$, which can be done by using the top network shown in Figure \ref{fig:times4}, which has size $(3,12,44)$. Recall that the layer 3 left node outputs $\phi(x)$. We can implement the function $\psi(a,b):= \sigma(4b-a)$ on nonnegative $a,b$ via a network of size $(3,19,58)$, which is the bottom network in Figure \ref{fig:psi}. If $4b-a\geq 0$, then $\psi(a,b)=4b-a$. Using this $\psi$ network, and that $16x-\phi(x)\geq 0$ by identity \eqref{eq:blancmange}, we see that the layer 6 right node produces $16x-\phi(x)$. Now we proceed a similar fashion. For $k=2,\dots,m-1$, the left and right nodes in layer $3k$ output $\phi^{\circ k}(x)$ and $4^{k-1}x-\sum_{j=1}^{k-1} 4^{k-1-j}\phi^{\circ j}(x)$, respectively. Using the $\psi$ network and identity \eqref{eq:blancmange} again, the $3k+3$ layer right node outputs $4^{k}x-\sum_{j=1}^k 4^{k-j}\phi^{\circ j}(x)$. Thus, generating $4^{m}x-\sum_{j=1}^m 4^{m-j}\phi^{\circ j}(x)$ requires a network with $3m+3$ layers and $O(m)$ nodes and parameters. Finally, we divide by $4^{m}$, which can be done with a network of size $(2m,2m,2m)$ to produce $S_\epsilon^+(x)$ in layer $5m+3$. We use a duplication network with two layers to produce $S_\epsilon^+(x)$ in layer $5m+5$. 
	
	For $S_\epsilon^-$, we use the same network as for $S_\epsilon^+$ and include a third column. The layer 1 rightmost node outputs $1/2$, which is implementable by a network of size $(1,1,1)$. We carry down $1/2$ another $3m+2$ layers by a network of size $O(m)$. Then we divide by $1/4^m$ via a network of size $(2m,2m,2m)$. In layer $5m+3$, the middle and right nodes output $S_\epsilon^+(x)$ and $1/2^{2m+1}$. From here, we can easily generate $S_\epsilon^-(x)$ using another two layer network that implements $(a,b)\mapsto \sigma(a-b)$. 
  
\end{proof}

\begin{figure}[t]
	\centering
	\includegraphics[width=0.6\textwidth]{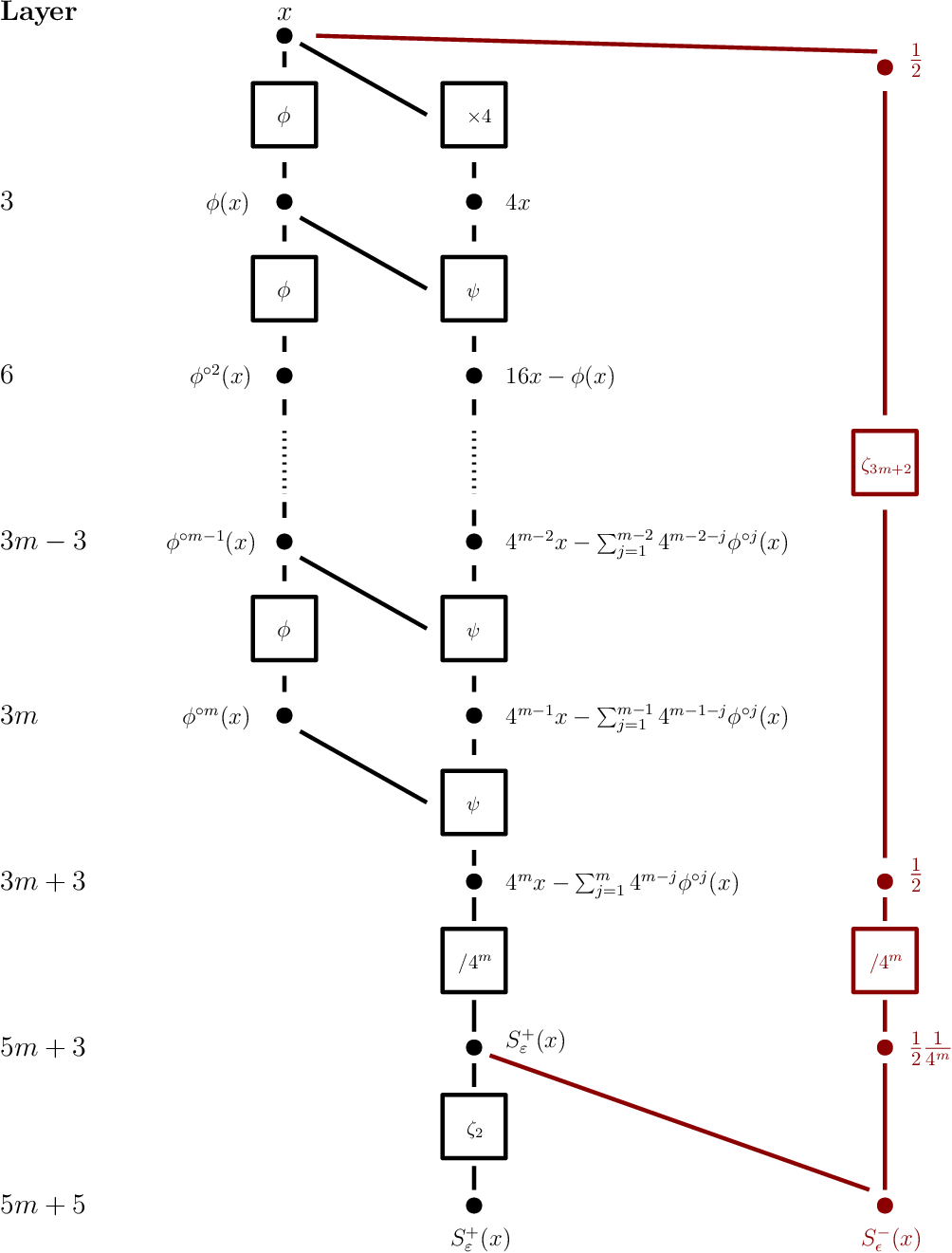}
	\caption{Implementations of $S_\epsilon^+$ and $S_\epsilon^-$}
	\label{fig:squaring}
	
\end{figure}

\begin{lemma}
	\label{prop:multbinary}
	For any $\epsilon>0$, there exists an activated $\calA_{1/2}$-quantized ReLU neural network of size $O(\log(1/\epsilon))$ that implements a nonnegative function $P_{\epsilon}\colon [0,1]^2\to\R$ such that 
	$$
	\text{for any } x,y\in [0,1], \quad  P_{\epsilon}(x,y)\leq xy, \quad\text{and}\quad |P_{\epsilon}(x,y)-xy|\leq \epsilon. 
	$$
\end{lemma}

\begin{proof}
	Let $\epsilon \in (0,1)$, and $S_\delta^+$ and $S_\delta^-$ be the approximate squaring function from \cref{prop:squaring}, where $\delta=\epsilon/6$. We define the functions
	\begin{align*}
		P_{\epsilon,0}(x,y)
		&:=2 \(S_\delta^-\( \frac{x+y}{2}\)-S_\delta^+\( \frac{x}{2}\)-S_\delta^+\( \frac{y}{2}\)\) 
		\quadand
		P_{\epsilon}(x,y)
		:=\sigma(P_{\epsilon}(x,y)). 
	\end{align*}
	Since
	$
	xy = 2(\frac{x+y}{2})^2-2( \frac{x}{2})^2- 2(\frac{y}{2})^2
	$
	and $xy\geq 0$, it follows from \cref{prop:squaring} that
	\begin{equation*} 
		\big|P_{\epsilon}(x,y) - xy\big|
		=\big|\sigma(P_{\epsilon,0}(x,y)) - xy\big|
		\leq \big|P_{\epsilon,0}(x,y) - xy\big|
		\leq 6 \delta
		= \epsilon.
	\end{equation*}
	Additionally, using that $S_\epsilon^+(x)\geq x^2$ and $S_\epsilon^-(x)\leq x^2$, we see that 
	$$
	P_{\epsilon,0}(x,y)
	\leq 2\(\(\frac{x+y}{2}\)^2-\( \frac{x}{2}\)^2- \(\frac{y}{2}\)^2\)
	=xy. 
	$$
	This inequality together implies $P_{\epsilon}(x,y)	\leq xy$ as well, because either $P_{\epsilon,0}(x,y)\leq 0$ in which case $P_\epsilon(x,y)=0\leq xy$, or $P_{\epsilon,0}(x,y)>0$ and $P_{\epsilon}(x,y)=P_{\epsilon,0}(x,y)\leq xy$. 
	
	Now we count the number of parameters. For inputs $x,y\in  [0,1]$, we first generate $x/2=\sigma(x/2)$, $y/2=\sigma(y/2)$, and $(x+y)/2=\sigma(x/2+y/2)$. This can be done with a $\calA_{1/2}$-quantized ReLU network of size $(1,3,4)$. Next, we place three networks in parallel, one $S_\delta^-$ network that takes $(x+y)/2$ as input, and two $S_\delta^+$ networks that take $x/2$ and $y/2$ as input. Note the implementations of $S_\delta^+$ and $S_\delta^-$ each has size $O(\log(1/\epsilon))$ and the same number of layers. Then we generate
	$$
	\sigma\(\frac{1}{2}S_\delta^-\( \frac{x+y}{2}\)-\frac{1}{2}S_\delta^+\( \frac{x}{2}\)-\frac{1}{2}S_\delta^+ \( \frac{y}{2}\)\)
	$$
	with a network of size $(1,1,3)$. Finally, we need a multiplication by $4$, which can be implemented by the top network shown in Figure \ref{fig:times4}.  
\end{proof}

\begin{figure}
	\centering 
	\includegraphics[width=0.7\textwidth]{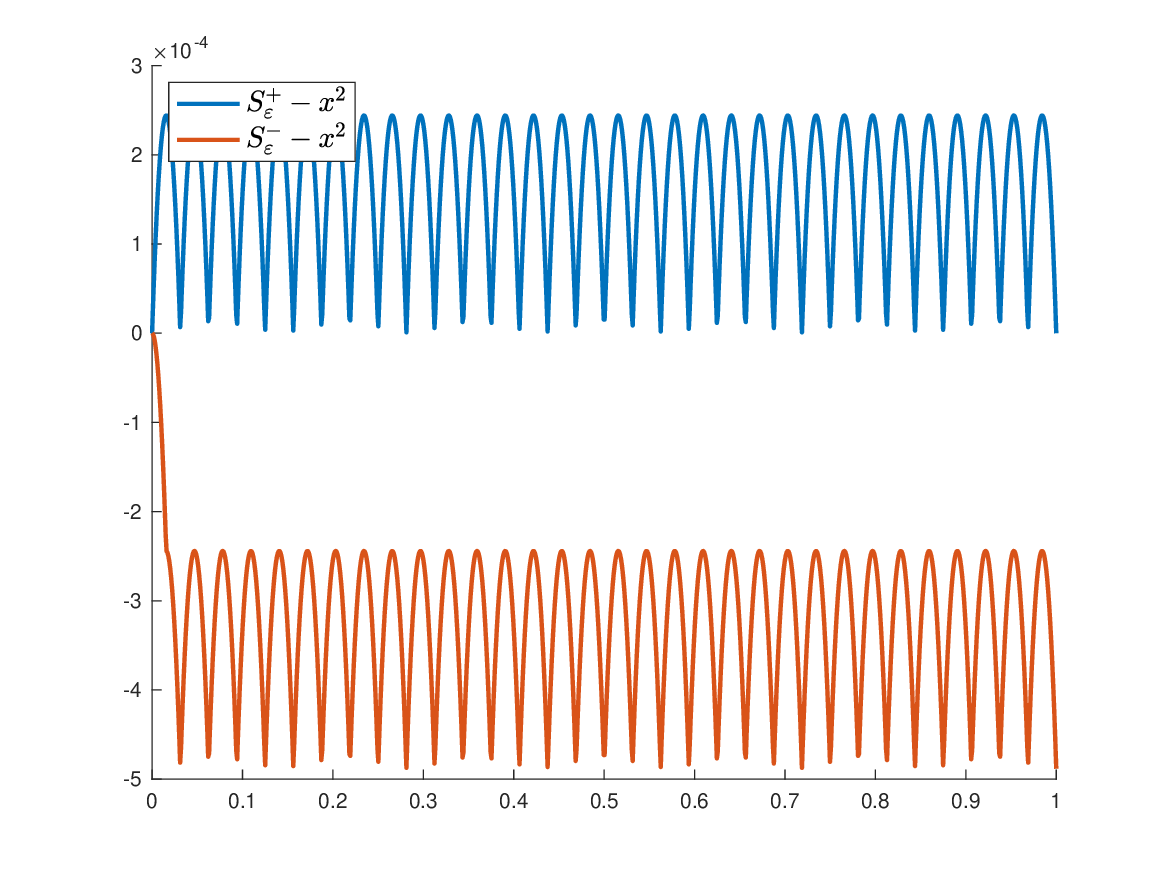}
	\caption{$S_\epsilon^+ -x^2$ and $S_\epsilon^--x^2$ for $\epsilon=10^{-3}$}
	\label{fig:Sfunctions}
\end{figure}

While $S_\epsilon^+$ approximates the quadratic function from above, we also constructed $S_\epsilon^-$, which approximates the quadratic function from below. An example of their approximation errors is shown in Figure \ref{fig:Sfunctions}. We are not aware of other neural network papers that have used $S_\epsilon^-$, but it is necessary to ensure that $P_\epsilon$ is nonnegative, which is an important property in subsequent results. 

Now we move onto the product of several real numbers. In many neural network approximation papers, approximate multiplication is done sequentially, such as $P_\delta(a,P_\delta(b,P_\delta(c,d)))$ for four inputs $a,b,c,d$. This strategy is inefficient without use of skip connections since $a,b$ would need to be brought down a considerable number of layers via duplication networks. We perform multiplication of $d$ numbers in a dyadic manner for improved efficiency. 

\begin{lemma}
	\label{prop:multseveral}
	For any integer $d\geq 2$ and $\epsilon >0$, there exists an activated $\calA_{1/2}$-quantized ReLU neural network of size $O( \log(1/\epsilon))$ that implements a nonnegative function $P_{\epsilon,d}\colon [0,1]^d\to\R$ such that  
	$$
	\text{for all } \bfx\in [0,1]^d, \quad 0\leq P_{\epsilon,d}(\bfx)\leq x_1\cdots x_d,
	\quad\text{and}\quad 
	|P_{\epsilon,d}(\bfx)-x_1\cdots x_d |\leq \epsilon. 
	$$
\end{lemma}

\begin{proof}
	Let $\epsilon\in (0,1)$, set $m=\lceil \log d \rceil$, and we will pick $\delta_1,\dots,\delta_m$ in terms of $\epsilon$ and $d$ later. We let $P_{\delta_k}$ be both the function and its network implementation from Proposition \ref{prop:multbinary}, and $L_k=O(\log(1/\delta_k))$ be the number of layers in $P_{\delta_k}$. Fix an arbitrary $\bfx\in [0,1]^d$ and for reasons that will become apparent, let $u_{0,\ell}=x_\ell$.
	
	Every consecutive pair of input nodes $x_1,\dots,x_d$ are fed into $\lfloor d/2\rfloor$ approximate multiplication networks $P_{\delta_1}$ placed in parallel. If $d$ is odd, then we simply carry $x_d$ down $L_1$ layers with a copy network of size $O(L_1)$. In layer $L_1$, there are $n_1=\lceil d/2 \rceil$ nodes whose outputs we call $u_{1,1},\dots,u_{1,n_1}$. For example, $u_{1,1}=P_{\delta_1}(x_1,x_2)$. 
	
	Each consecutive pair of $u_{1,1},\dots,u_{1,n_1}$ are fed into at most $\lfloor n_1/2 \rfloor$ networks $P_{\delta_2}$ placed in parallel, and if $n_1$ is odd, then $u_{1,n_1}$ is carried down $L_2$ layers by a copy network of size $O(L_2)$. In layer $L_1+L_2$, we call the outputs $u_{2,1},\dots,u_{2,n_2}$, where $n_2=\lceil n_1/2 \rceil$. We continue this process with $P_{\delta_3},\dots,P_{\delta_m}$ and denote the outputs following each $P_{\delta_k}$ or copy network by $u_{k,1},\dots,u_{k,n_k}$. The final output is $u_m:=P_{\epsilon,d}(\bfx)$.  
	
	We first show how to pick the $\delta_1,\dots,\delta_m$ properly and quantify the approximation error. First note that from Proposition \ref{prop:multbinary}, for any $a,b\in [0,1]$ and $\delta>0$, we have $0\leq P_\delta(a,b)\leq ab\leq 1$, so by induction, we have that $0\leq u_{k,\ell}\leq 1$. The quantities 
	$
	U:=\{u_{k,\ell}\}_{k=0,\dots,m, \  \ell=1,\dots,n_k}
	$
	that are created have a tree structure with leaves $x_1,\dots,x_d$. For a $u\in U$, we let $p(u)$ be product of all $x_1,\dots,x_d$ that are connected to $u$. We have that 
	$$
	u_{1,\ell}\leq p(u_{1,\ell})\leq 1 \quad \text{and}\quad  |u_{1,\ell}-p(u_{1,\ell})|\leq \delta_1 \quad\text{for }\quad  \ell=1,\dots,n_1. 
	$$ 
	Fix any $u_{k,\ell}$. If $u_{k,\ell}$ is not generated as an output of a $P_{\delta_k}$ network, then it is created by copying a $u_{k-1,\ell'}$, in which case, $|u_{k,\ell}-p(u_{k,\ell})|=0$. If not, $u_{k,\ell}=P_{\delta_k}(u_{k-1,a},u_{k-1,b})$ for distinct $1\leq a,b\leq n_{k-1}$. Then
	\begin{align*}
		|u_{k,\ell}-p(u_{k,\ell})|
		&=|P_{\delta_k}(u_{k-1,a},u_{k-1,b})-p(u_{k,\ell})| \\
		&=|P_{\delta_k}(u_{k-1,a},u_{k-1,b})-u_{k-1,a}u_{k-1,b}|+|u_{k-1,a}u_{k-1,b}-u_{k-1,a}p(u_{k-1,b})| \\
		&\quad\quad +|u_{k-1,a}p(u_{k-1,b})-p(u_{k-1,a})p(u_{k-1,b})| \\
		&\leq \delta_k + |u_{k-1,b}-p(u_{k-1,b})|+|u_{k-1,a}-p(u_{k-1,a})|.
	\end{align*}
	Hence, the error made by $u_{k,\ell}$ is bounded above by $\delta_k$ plus the errors made in the previous two $u_{k-1,a}$ and $u_{k-1,b}$. Inducting on $k$, we see that 
	$$
	|u_m-p(u_m)|
	\leq 2^{m-1} \delta_1 + 2^{m-2}\delta_2 + \cdots + \delta_m
	\leq d \big( \delta_1+2^{-1} \delta_2+\cdots+ 2^{-m+1}\delta_m\big) ,
	$$
	where we noted that $2^{m-1}=2^{\lceil \log d \rceil-1}\leq d$. To make the final error bounded by $\epsilon$, we pick $\delta_k=\epsilon/(2d)$ for each $k=1,\dots,m$. This proves that $P_{\epsilon,d}(\bfx)-x_1\cdots x_d|\leq \epsilon$ for all $\bfx\in [0,1]^d$.
	
	It remains to count the size of this network that implements $P_{\epsilon,d}$. Generating $\{u_{k,\ell}\}_{\ell=1,\dots,n_k}$ from $\{u_{k-1,\ell}\}_{\ell=1,\dots,n_{k-1}}$ requires $n_k$ networks each of size $O(\log(1/\delta_k))=O(\log(1/\epsilon))$. Hence, the number of layers in $P_{\epsilon,d}$ is $O(m\log(1/\epsilon))=O(\log(1/\epsilon))$.  Since $n_k\leq 2d/2^k$ and 
	$$
	\sum_{k=1}^m n_k \log(1/\delta_k)
	\leq \sum_{k=1}^m \frac{2d}{2^k} \log(2d/\epsilon)
	\leq 2d \log(2d/\epsilon),
	$$
	the resulting network that implements $P_{\epsilon,d}$ has $O(\log(1/\epsilon))$ nodes and parameters.
\end{proof}

\begin{lemma}
	\label{prop:bernrelu1}
	For any $\epsilon>0$ and integer $n\geq 1$, there exists an activated $\calA_{1/2}$-quantized ReLU neural network with $O(n \log(n/\epsilon))$ layers and $O(n^2 \log(n/\epsilon))$ nodes and parameters that implements a function $b_n\colon [0,1]\to\R^{n+1}$ such that for each $0\leq k\leq n$,
	$$
	0\leq b_{n,k}\leq p_{n,k}, \quad \text{and}\quad \|b_{n,k}-p_{n,k}\|_\infty 
	\leq \epsilon. 
	$$
\end{lemma} 

\begin{proof}
	This construction will be done recursively. Fix an $\epsilon>0$, and we will pick an appropriate $\delta>0$ depending only on $\epsilon$ and $n$ later. Let $P_\delta$ denote both the approximate binary multiplication function and its network implementation from Proposition \ref{prop:multbinary} and $L_\delta=O(\log(1/\delta))$ be the number of layers. 
	
	Let $x\in [0,1]$ denote the input. We generate $1-x$ via the formula,
	$$
	1-x
	=\sigma\( \frac{1}{2}\sigma\(-\frac{1}{2}x+\frac{1}{2}\) +\frac{1}{2}\sigma\(-\frac{1}{2}x+\frac{1}{2}\)+\frac{1}{2}\sigma\(-\frac{1}{2}x+\frac{1}{2}\)+\frac{1}{2}\sigma\(-\frac{1}{2}x+\frac{1}{2}\)\).
	$$
	This is carried out by a network of size $(2,5,12)$. We next use a copy network to produce $x$ in the second layer. Hence, the first degree Bernstein polynomials $p_{1,0}(x)=1-x$ and $p_{1,1}(x)=x$ are exactly implementable, and are outputs of nodes in the second layer. We define
	$$
	b_{1,0}(x):=\sigma(x) \andspace b_{1,1}(x):=\sigma(1-x).
	$$
	
	Since we will need to use $x,1-x$ repeatedly throughout, we will use copy networks to bring them down however many layers we need. The size of these networks will be proportional to the total number of layers in $b_n$. We will see that the size of these copy networks will be dominated by the other portions of the final network. 
	
	Set $L_1=2$. We recursively define the following. For each $1\leq m\leq n-1$ and $1\leq k\leq m$, let 
	$$
	b_{m+1,k}(x)
	:=\text{sum}\( P_{\delta}\big(x,b_{m,k-1}(x)\big),P_{\delta}\big(1-x,b_{m,k}(x)\big) \),
	$$
	where sum$(\cdot,\cdot)$ refers to the two layer summation network. This shows that $b_{m+1,k}(x)$ can be generated provided that $b_{m,k-1}(x),b_{m,k}(x),x,1-x$ are all outputs of nodes that appear in layer $L_m$. If so, $b_{m+1,1}(x),\dots,b_{m+1,m}(x)$ as outputs of nodes in layer $L_{m+1}:=L_m+L_\delta+2$. For the remaining two endpoint cases, let $\zeta_2$ be a two layer duplication network. We define 
	\begin{equation*}
		b_{m+1,0}(x):= \zeta_2(P_{\delta}(1-x,b_{m,0}(x))) 
		\andspace b_{m+1,m+1}(x):=\zeta_2(P_{\delta}(x,b_{m,k-1}(x))),
	\end{equation*}
	which are also outputs in layer $L_{m+1}$. Finally, we copy $x,1-x$ from layer $L_m$ down to layer $L_{m+1}$ which requires a network of size $O(L_\delta)$. 
	
	We still need to show that these functions are well defined, because $P_\delta$ takes inputs in $[0,1]^2$. To establish this, we prove the stronger statement that $0\leq b_{m,k}\leq p_{m,k}$ for each $1\leq m\leq n$ and $0\leq k\leq m$. We proceed by induction on $m$. For the base case, we have $b_{1,0}=p_{1,0}$, and $b_{1,1}=p_{1,1}$. Assume that for some $m$, we have that $b_{m,k}\leq p_{m,k}$ for each $0\leq k\leq m$. Now, consider any $1\leq k\leq m+1$. By \cref{prop:multbinary}, for all $x\in [0,1]$, we have
	\begin{align*}
		b_{m+1,k}(x)
		&= P_{\delta}\big(x,b_{m,k-1}(x)\big) + P_{\delta}\big(1-x,b_{m,k}(x)\big) \\
		&\leq x b_{m,k-1}(x) + (1-x)b_{m,k}(x) \\
		&\leq x p_{m,k-1}(x) + (1-x)p_{m,k}(x)
		=p_{m+1,k}(x). 
	\end{align*}
	The remaining two cases $k=0$ and $k=m+1$ follow from an analogous argument. To summarize, we have shown that for $1\leq m\leq n-1$, layer $L_m$ has $m+3$ nodes whose outputs are $b_{m,0}(x),\dots,b_{m,m}(x),x,1-x$.
	
	We proceed to examine the approximation error. For convenience, let
	$$
	\alpha_{m,k}
	:=\|b_{m,k}-p_{m,k}\|_\infty, \andspace
	\alpha_{m}:=\max_{0\leq k\leq m} \alpha_{m,k}. 
	$$
	Hence $\alpha_1=0$. Using the recurrence relation \eqref{eq:bernrecurrence}, triangle inequality, and \cref{prop:multbinary}, we have 
	\begin{equation}
		\label{eq:help1}
		\begin{aligned}
			\alpha_{m+1,k}
			%&= \sup_{x\in [0,1]} \Big|P_{\delta}\big(x,b_{m,k-1}(x)\big) - x p_{m,k-1}(x)+P_{\delta}\big(1-x,b_{m,k}(x)\big)-(1-x)p_{m,k}(x) \Big| \\
			&\leq \sup_{x\in [0,1]} \(\Big| xp_{m,k-1}(x)- x b_{m,k-1}(x)+(1-x)p_{m,k}(x)-(1-x)b_{m,k}(x) \Big| \\
			&+ \Big| P_{\delta}(x,b_{m,k-1}(x)) - x b_{m,k-1}(x)+P_{\delta}(1-x,b_{m,k}(x))-(1-x)b_{m,k}(x) \Big| \) \\
			&\leq \sup_{x\in [0,1]} \( \delta + \delta + |x|\alpha_{m,k-1}+|1-x|\alpha_{m,k} \)\\
			&\leq \sup_{x\in [0,1]} \( 2\delta + x\alpha_{m}+(1-x)\alpha_{m} \)
			= 2\delta+\alpha_{m}. 
		\end{aligned}
	\end{equation}
	Repeating the same argument for $k=0$ and $k=m+1$ provides us with 
	\begin{equation}
		\label{eq:help2}
		\max(\alpha_{m+1,0},\alpha_{m+1,m+1})
		\leq \delta +\alpha_m. 
	\end{equation}
	Combining equations \eqref{eq:help1} and \eqref{eq:help2}, we see that 
	$
	\alpha_{m+1}
	\leq 2\delta+\alpha_m,
	$
	for all $1\leq m\leq n$. A telescoping argument shows that, $\alpha_m=\alpha_m-\alpha_0
	\leq 2(m-1)\delta$. Thus, we pick $\delta=\epsilon/(2n)$ to see that
	\begin{equation*}
		\alpha_m
		=\max_{0\leq k\leq m}\|b_{m,k}-p_{m,k}\|_\infty
		\leq \epsilon 
		\quad \text{for all } m\leq n. 
	\end{equation*}
	
	Now, we proceed to count the number of parameters. For the first row of this Pascal triangle, $b_{1,0}$ and $b_{1,1}$ can be made with a network of constant size. Computing each $b_{m+1,k}$ from the previous row $\{b_{m,k}\}_{k=0}^m$ requires at most two approximate multiplication networks $P_{\delta}$ with $\delta=\epsilon/(2n)$ and a summation, which requires a network of size $O(\log (n/\epsilon))$. Hence, computing $\{b_{m=1,k}\}_{k=0}^{m+1}$ from $\{b_{m,k}\}_{k=0}^m$ requires $O(\log (n/\epsilon))$ layers and $O(m\log(n/\epsilon))$ nodes and parameters. We do this from $m=1$ to $m=n-1$. 
\end{proof}

\begin{lemma}
	\label{prop:bernrelu}
	For any $\epsilon >0$ and integers $n,d\geq 1$, there exists an activated $\calA_{1/2}$-quantized ReLU neural network with $O(n\log (n/\epsilon))$ layers and $O( n^2\log (n/\epsilon) + n^d\log(1/\epsilon))$ nodes and parameters, as $n\to\infty$ and $\epsilon\to 0$, that implements a function $b_n\colon [0,1]^d\to\R^{(n+1)^d}$ such that 
	$$
	\text{for each } 0\leq \bfk\leq n, \quad 
	b_{n,\bfk}\geq 0 \andspace
	\|b_{n,\bfk}-p_{n,\bfk}\|_\infty 
	\leq \epsilon. 
	$$
	
\end{lemma}

\begin{proof}
	Fix $\epsilon\in (0,1)$, and we will pick appropriate $\delta,\gamma \in (0,1)$ later. Let $\bfx\in [0,1]^d$ be the input. For each $1\leq \ell\leq d$, we use \cref{prop:bernrelu1} to provide us with an activated network with $O(n\log(n/\gamma))$ layers and $O(n^2\log(n/\gamma))$ nodes and parameters that produces $\{b_{n,k_\ell}(x_\ell)\}_{k_\ell=0}^n$. Each of these $d$ networks have exactly the same number of layers, so placing all $d$ of them in parallel, we obtain a network that outputs $\{b_{n,k_\ell}(x_\ell)\}_{1\leq \ell\leq d, \, 0\leq k_\ell\leq n}$ in the same layer. We also have
	\begin{equation*}
		b_{n,k_\ell}\leq p_{n,k_\ell}, \quad\text{and}\quad \|p_{n,k_\ell}-b_{n,k_\ell}\|_\infty\leq \gamma.
	\end{equation*}
	For each $0\leq \bfk\leq n$, we use a $d$-term approximate product ReLU neural network $P_{\delta,d}$ as in \cref{prop:multseveral} and define
	$$
	b_{n,\bfk}(\bfx)
	:=P_{\delta,d}\big(b_{n,k_1}(x_1),\dots,b_{n,k_d}(x_d)\big).
	$$
	This is well-defined since $b_{n,k_\ell}\leq p_{n,k_\ell}\leq 1$. 
	
	We need $(n+1)^d$ many such $P_{\delta,d}$ networks placed in parallel, and each one has size $O(\log(1/\delta))$. The entire implementation of $\{b_{n,\bfk}(\bfx)\}_{0\leq \bfk\leq n}$ can be done by a network with
	\begin{equation*}
		\begin{split}
			O\big(n\log(n/\gamma)+ \log(1/\delta)\big) \quad &\text{layers}, \\
			O\big(n^2\log(n/\gamma) +  n^d\log(1/\delta) \big) \quad &\text{nodes and parameters}.
		\end{split}
	\end{equation*}
	
	Next we compute the errors between $p_{n,\bfk}$ and $b_{n,\bfk}$, and then optimize over the parameters. For each $0\leq \bfk\leq n$, we first apply  \cref{prop:multseveral} to get
	\begin{align*}
		|p_{n,\bfk}(\bfx)-b_{n,\bfk}(\bfx)|
		&\leq \Big| P_{\delta,d}\(b_{n,k_1}(x_1),\dots, b_{n,k_d}(x_d)\)-b_{n,k_1}(x_1)\cdots b_{n,k_d}(x_d)\Big| \\ 
		&\quad +| p_{n,\bfk}(\bfx)-b_{n,k_1}(x_1)\cdots b_{n,k_d}(x_d)| \\
		&\leq \delta  + | p_{n,\bfk}(\bfx)-b_{n,k_1}(x_1)\cdots b_{n,k_d}(x_d)|. 
	\end{align*}
	To control the right hand side, we use that $p_{n,\bfk}(\bfx)$ is a tensor product and peel off one term at a time. Then
	\begin{align*}
		|p_{n,\bfk}(\bfx)-&b_{n,k_1}(x_1)\cdots b_{n,k_d}(x_d)| \\
		&\leq |p_{n,k_1}(x_1)-b_{n,k_1}(x_1)| |p_{n,k_2}(x_2)\cdots p_{n,k_d}(x_d)| \\
		&\quad \quad + |b_{n,k_1}(x_1)| | p_{n,k_2}(x_2) \cdots p_{n,k_d}(x_d)-b_{n,k_2}(x_2)\cdots b_{n,k_d}(x_d)| \\
		&\leq \gamma + | p_{n,k_2}(x_2) \cdots p_{n,k_d}(x_d)-b_{n,k_2}(x_2)\cdots b_{n,k_d}(x_d)|. 
	\end{align*}
	Continuing in this manner, we obtain the inequality 
	$$
	|p_{n,\bfk}(\bfx)-b_{n,\bfk}(\bfx)|
	\leq \delta  + | p_{n,\bfk}(\bfx)-b_{n,k_1}(x_1)\cdots b_{n,k_d}(x_d)|
	\leq \delta+d\gamma.
	$$
	We select $\gamma=\epsilon/(2d)$ and $\delta=\epsilon/2$ to complete the proof.
\end{proof}

\section*{Acknowledgments}

Weilin Li was supported by the AMS--Simons Travel Grant and a startup fund provided by the CUNY Research Foundation.

% % % % % % % % % % % % % % % % % % % %

\bibliography{ApproxQuanNN}
\bibliographystyle{plain}

% % % % % % % % % % % % % % % % % % % %
\end{document}